\documentclass[twoside,11pt]{article}

\RequirePackage[widermath,partialup,mathcalasscript,fullsumlimits]{kpfonts}
\usepackage{amsmath,amsthm,amssymb,amsfonts}
\usepackage{mathtools}
\usepackage{enumerate}
\usepackage{algorithm, algorithmicx, algpseudocode}
\usepackage{xcolor}
\usepackage{paralist}
\usepackage{setspace}
\usepackage{longtable}
\usepackage{booktabs, multicol, multirow, fancyhdr}
\usepackage{physics}
\usepackage{mathrsfs}
\usepackage{dsfont}

\allowdisplaybreaks

\usepackage{subfig}
\usepackage[preprint]{jmlr2e}

\renewcommand{\d}{\mathrm{d}}
\newcommand{\R}{\mathbb{R}}
\newcommand{\E}{\mathsf{E}}
\newcommand{\M}{\mathcal{M}}
\renewcommand{\S}{\mathcal{S}}
\newcommand{\A}{\mathcal{A}}
\newcommand{\T}{\mathcal{T}}
\newcommand{\N}{\mathbb{N}}
\renewcommand{\P}{\mathsf{P}}
\newcommand{\F}{\mathcal{F}}
\newcommand{\G}{\mathcal{G}}

\newcommand{\X}{\mathcal{X}}

\newcommand{\cE}{\mathcal{E}}
\newcommand{\cN}{\mathcal{N}}

\newcommand{\argmin}{\mathop{\mathrm{argmin}}\limits}
\newcommand{\argmax}{\mathop{\mathrm{argmax}}\limits}

\makeatletter
\newcommand*{\transpose}{%
  {\mathpalette\@transpose{}}%
}
\newcommand*{\@transpose}[2]{%
  \raisebox{\depth}{$\m@th#1\intercal$}%
}
\makeatother

\usepackage{lastpage}
\jmlrheading{***}{2026}{1-\pageref{LastPage}}{9/24}{***}{[25-2333]}{Ayyagari, Eega and Dukkipati}

\ShortHeadings{Reinforcement Learning under External Influence}{Ayyagari, Eega and Dukkipati}
\firstpageno{1}

\begin{document}

\title{Reinforcement Learning under External Influence:\\Guarantees, Algorithms, and Sample Complexity}

\author{\name Ranga Shaarad Ayyagari \email rangaa@iisc.ac.in 
        \AND
        \name Revanth Raj Eega \email revanthraje@iisc.ac.in
        \AND
        \name Ambedkar Dukkipati \email ambedkar@iisc.ac.in \\
       \addr Department of Computer Science and Automation\\
       Indian Institute of Science\\
       Bengaluru, India
       }

\editor{My editor}

\maketitle
\begin{abstract}

\noindent
In this paper, we study the problem of reinforcement learning under the influence of external events. For this, we consider Markov decision processes with continuous state and action spaces whose transition dynamics are perturbed by an external process in a non-Markovian manner.
First, we establish the conditions under which the problem becomes tractable, allowing it to be addressed by considering only a finite history of events, based on the properties of the perturbations introduced by the exogenous process. To solve this problem, we propose and theoretically analyze a policy iteration algorithm that learns policies contingent on the current state of the environment and a finite history of prior exogenous events. Since this algorithm is not guaranteed to converge, we provide a guarantee for policy improvement in regions of the state space determined by the approximation error induced by considering tractable policies and value functions. We also establish the sample complexity of least-squares policy evaluation and policy improvement algorithms that account for approximations arising from considering only a finite history of temporal events. While our results are applicable to general discrete-time processes satisfying certain conditions on the rate of decay of the influence of their events, we further analyze the case of discrete-time Hawkes processes with Gaussian marks. We performed experiments to demonstrate our findings for policy evaluation and deployment in traditional control environments.
\end{abstract}


\section{Introduction}
Markov Decision Processes (MDPs) provide a mathematical framework for Reinforcement Learning (RL). In this framework, an agent must select an action at each time step by interacting with an environment and subsequently receives a reward.  The state of the environment changes stochastically according to the Markov property, depending solely on the current state and the agent's action. The objective of an RL algorithm is to learn a sequence of decisions that maximize the expected cumulative reward. Most RL algorithms assume that although the state of an MDP may be constantly changing, the rules governing state transitions remain unchanged. However, in practical applications, external events may influence the agent's environment and change its dynamics.

Consider portfolio management in finance, where an agent has access to a state comprising various fundamental and technical indicators for a select set of financial instruments and must make sequential decisions about whether to hold, sell, or short each instrument based on the state.
However, the price movements of these instruments may be affected by external forces, such as government decisions (e.g., rate hikes) or unexpected market shifts triggered by events in unrelated economic sectors that the agent may not consider. The impact of such external events can vary across different time scales, depending on the nature of the event. For instance, a rate cut might influence stock prices for several months, whereas the effects of a flash crash might dissipate within hours. Figure~\ref{fig:nonstationary_illustration} provides an example illustrating the persistent nonstationary effect of perturbations resulting from exogenous events.  

Consider, for example, a navigation problem in which an agent controls a robot tasked with traveling from one location to another. Although an agent can learn to execute this task under ideal conditions, its optimal actions may be influenced by external factors, such as rain, sudden traffic changes, and human movement patterns. Therefore, the agent cannot be oblivious to external changes that influence environmental dynamics.

\begin{figure}
  \includegraphics[width=\textwidth]{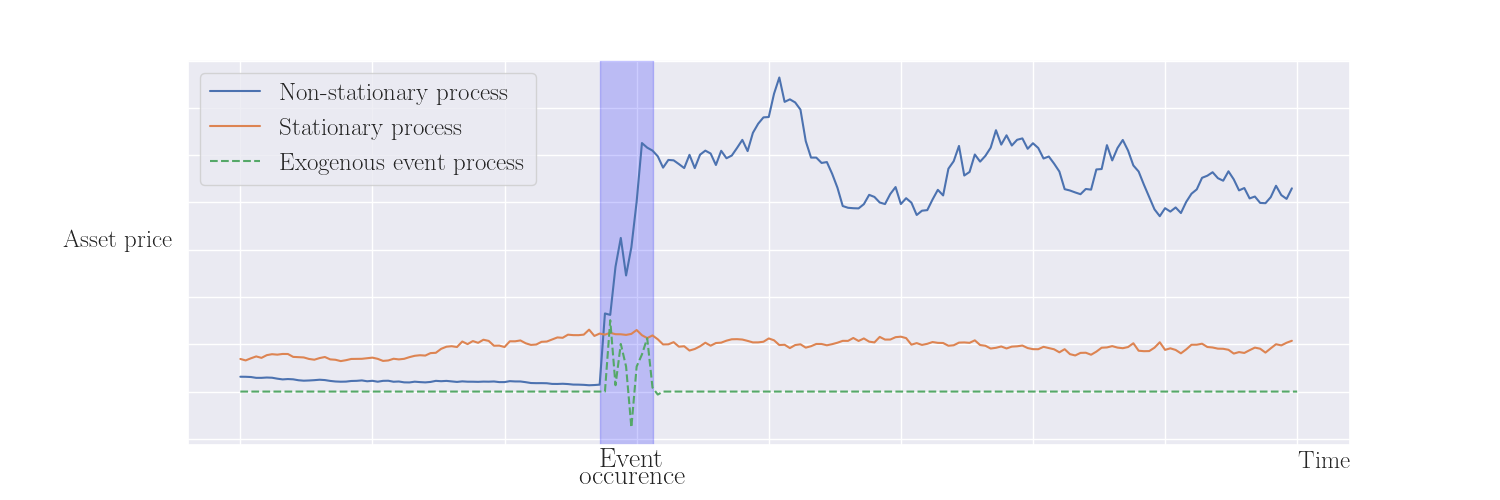}
  \caption{
  The plot illustrates the prices of the two assets following a Geometric Brownian motion. A few exogenous Gaussian shocks to one asset cause it to diverge significantly from the other over the long term, with a markedly different evolutionary pattern long after the exogenous event decays.}
  \label{fig:nonstationary_illustration}
\end{figure}

In the existing literature, nonstationarity in RL settings has been studied by considering several special cases, mostly providing practical solutions to the problem. 
Some studies model nonstationarity as piecewise stationarity~\citep{Alegre:2021:MinimumDelayAdaptationInNonStationaryRL,Hadoux:2014:SolvingHiddenSemiMarkovModeMDPs,Li:2025:TestingStationarityAndChangePointDetectionInRL}, whereas a few consider drifting environments~\citep{Lecarpentier:2019:NonStationaryMDPsAWorstCaseApproachUsingModelBasedRL,Cheung:2020:RLForNonStationaryMDPsTheBlessingOfMoreOptimism}. \citet{Hallak:2015:ContextualMDPs} study Contextual MDPs, in which the environment dynamics change based on externally specified ``contexts'' that are episodic and possibly unknown, chosen from among a finite set of known values. \citet{Tennenholtz:2023:RLwithHistoryDependentDynamicContexts} extend this setting to handle dynamic contexts known to the agent that change at each time step and influence the state transition distribution. However, their contexts are not exogenous and originate from a finite set of possible vectors.

In this paper, we study MDPs with continuous state and action spaces, where transition dynamics are influenced by an external process in a non-Markovian fashion. 
This setting is very general, as it assumes nonstationarity without specific details regarding its impact on the rewards and transition probability matrix. 
This compels one to consider the history of events at each time step to achieve optimal decision-making. This leads to a drastic increase in the computational complexity of the RL algorithms. 



\noindent
\textbf{Contributions.} Our contributions are as follows. \\
\textbf{(1)} We outline the necessary conditions that ensure the existence of a well-defined solution for this problem. Subsequently,  we establish the criteria under which an approximate solution can be obtained using only the current state and a finite history of past events. We show that this is possible when (i) the perturbations caused by events older than $t$ time steps on the MDP transition dynamics 
and the event process itself are bounded in total variation by $M_t$ and $N_t$ respectively, and (ii) $\sum_t M_t$ and $ \sum_t N_t$ are convergent series. We analyze the trade-off between the performance of the agent and the length of history considered.\\
\textbf{(2)} We propose a policy iteration algorithm and provide theoretical guarantees. We show that policy improvement occurs in regions of the state space where the Bellman error is ``not too low'' as determined by $M_t$ and $N_t$. Our analysis further reveals that an increase in approximation error due to nonstationarity results in a reduction of the region where policy improvement is assured. \\
\textbf{(3)} We subsequently analyze the sample complexity of pathwise least-squares temporal difference policy evaluation and approximate least-squares policy iteration in this nonstationary setting. Our study focuses on the influence of nonstationarity on both the expected error of the learned value function~\citep{Lazaric:2012:FiniteSampleAnalysisOfLSPI} and the expected suboptimality of the resulting policy after $K$ iterations. For policy evaluation, we show that, with high probability, the expected error of the value function estimate decomposes into three components: an approximation error arising from the exclusion of older events, an inherent error attributable to linear function approximation, and standard error terms due to stochasticity and mixing, as observed in the stationary case. For the overall policy iteration algorithm, we quantify the effect of the additional error arising from considering approximately greedy policies that depend only on the current state and a finite history of events.\\
\textbf{(4)} Our theoretical results are further explained through the application of the discrete-time Hawkes process. We conduct experiments to evaluate and deploy policies in nonstationary pendulum and point maze environments, thereby validating our findings.

\section{Preliminaries and Notation}
\label{section:prelims}

Consider a Markov Decision Process (MDP) defined by a tuple $\M = (\S, \A, Q, r, \gamma )$, where $\S$ is the state space, $\A$ is the action space, $Q$ is the transition probability kernel, $r$ is the reward function, and $\gamma$ is the discount factor. At each time step $t$, the agent in state $s \in \S$ takes an action $a \in \A$ and obtains a reward $r(s,a)$, and the environment transitions to state $s' \sim Q\left( \; . \; | s, a \right)$.

In general, $\S$ and $\A$ are Borel spaces, $r: \S \times \A \rightarrow \R$ is a measurable function, and $\gamma \in (0, 1)$. $Q$ is a stochastic kernel on $\S$ given $\S \times \A$. That is, $Q(. | y)$ is a probability measure on $\S$ for each $y \in \S \times \A$, and $Q(B | .)$ is a measurable function on $\S \times \A$ for each $B \in \mathcal{B}(\S)$, the Borel $\sigma$-algebra on $\S$.

Given a class of possible policies $\Pi$, the goal of the agent is to determine a policy $\pi \in \Pi$ that achieves the maximum possible value function $V^\pi$,
\[
V^\pi(s) = \E_{s_t, a_t}^\pi \left[ \sum_{t=0}^\infty \gamma^t r(s_t, a_t) \right], \quad \pi \in \Pi, s \in \S,
\]
the expected sum of discounted rewards obtained by an agent when it starts from a state $s$ and follows a policy $\pi$ resulting in a trajectory of states $s_0 = s, s_1, s_2, \dots$ by taking actions $a_0, a_1, a_2, \dots$. The expectation is over the states and actions in the agent's trajectory owing to the stochasticity of the environment and possibly the policy itself. 

Let $\pi^*$ be an optimal policy that achieves the corresponding value function $V^{\pi^*}$, denoted $V^*$. A measurable function $v: \S \rightarrow \R$ is said to be a solution to the Bellman optimality equation if it satisfies $v = \T v$,
where $\T$ is the optimal Bellman operator defined as
\begin{equation}
[\T v] (s) = \max_{a \in \A} \left[ r(s, a) + \gamma \int_{\S} v(s') Q(ds' | s,a) \right], \quad \text{for all } s \in \S.
\label{eqn:nonstationary_bellman}
\end{equation}
A function that satisfies the optimal Bellman equation is a fixed point of the operator $\T$. Under suitable regularity conditions on the rewards and transition kernel, the optimal value function $V^*$ satisfies~\eqref{eqn:nonstationary_bellman}.

In this paper, we consider the setting in which the agent receives rewards from the environment, defined by the reward function $r: \S \times \A \rightarrow \R$, and its goal is to maximize the expected return. Equivalently, many studies consider an agent incurring ``costs'', defined by an analogous cost function $c: \S \times \A \rightarrow \R$. In this case, the objective of the agent is to minimize the expected discounted sum of the costs incurred; thus, the Bellman operator is defined as the minimization of the right-hand side in ~\eqref{eqn:nonstationary_bellman}.

For this cost setting, under the assumptions of (i) the one-stage cost $c$ is lower semi-continuous, non-negative, and inf-compact on $\S \times \A$, (ii) $Q$ is strongly continuous, and (iii) there exists a policy $\pi$ such that $V^\pi(s) < \infty$ for all $s \in \S$, the existence and uniqueness of the optimal policy and value function can be established as follows:  
\begin{theorem}[Theorem 4.2.3 of \cite{Hernandez:2012:DiscreteTimeMCPs}]
If the two assumptions above hold, we have the following. 
\begin{enumerate}
	\item The optimal value function $V^*$ is the pointwise minimal solution to the above Bellman equation. If $u(.)$ is any other solution to the above equation, then $u(.) \geq V^*(.)$.
	\item There exists a function $\pi^* : \S \rightarrow \A$ such that $\pi^*(s) \in \A$ attains the minimum in the Bellman equation, i.e.,
$$
V^*(s) = c(x, \pi^*(s)) + \gamma \int V^*(s') Q(ds' | s, \pi^*(s)), \quad \text{ for all } s \in \S,
$$
The deterministic stationary policy $\pi^*$ is optimal. Conversely, any optimal deterministic stationary policy satisfies the above equation.
	\item If an optimal policy exists, then there exists one that is deterministic and stationary.
\end{enumerate}
\label{thm:nonstationary:optim_textbook}
\end{theorem}

\section{MDP influenced by a Temporal Process}
\label{section:formulation}
Let $\M = (\S, \A, Q, r, \gamma)$ be an MDP with state and action spaces $\S$ and $\A$ respectively, transition kernel $Q$, discount factor $\gamma$, and reward function $r: \S \times \A \times \S \rightarrow \R$, which defines the reward $r(s, a, s')$ as a function of the state-action pair $(s, a)$ and the next state $s'$. Consider
an external discrete-time temporal process $X = \left(t_i, X_i \right)_{i \in \N}$ that influences the transition probabilities of $\M$. Here, $X$ is exogenous; hence, it is not affected by  $\M$. The probability of an event occurring in the external process at time $t$ is a function of all past (external) events until that time. Each event is associated with a probabilistic mark $x \in \X$, where $\X$ is a closed subset of $\R^d$.

Let $H_t = \{(t_1, x_1), (t_2, x_2), \ldots, (t_j, x_j) \}$ be the history of external events until time $t$,  where the ordered pair $(t_i, x_i)$ denotes an event with mark $x_i$ that occurred at time $t_i$. The history $H_{t}$ perturbs the stochastic kernel $Q$ of the MDP to give rise to a new transition kernel $Q_{H_t}$ on $\S$ given $\S \times \A$. At the same time, this history determines the next event in the temporal process. We denote the distribution of the event mark at time $t$ by $Q^X_{H_t}$, which is a probability distribution on $\X$. Because this describes the event distribution outside the MDP, it does not depend on the MDP's state or the agent's action.

In this paper, we consider a discrete-time event process as an external influence and make the following assumptions. 

\noindent
\textbf{(A1)} There exists an event with a mark, say $X = 0$ (zero), that is equivalent to a non-event.\\
\textbf{(A2)} Events that occurred in the distant past have a reduced influence on the current probability transition kernel of the MDP. Specifically, there exists a convergent series $\sum_T M_T$ of real numbers, such that at any time $t$, for any event histories $H_t$ and $H'_t$, if the MDP is in state $s$ and action $a$ is taken, and if $H_t$ and $H'_t$ differ by only one event at time $t'$, then
	$$
	\mathsf{TV} \left( Q_{H_t}(. | s,a), Q_{H'_t}(. | s,a) \right) < M_T, \quad \text{for all } s \in \S, a \in \A,
	$$
	provided $t - t' \geq T$.
	That is, if an event is older than $T$, its influence on the current probability distribution has a Total Variation distance upper bounded by $M_T$. Furthermore, $Q_{H_t}$ is independent of the previous states conditioned on the current state. That is, $Q_{H_t}$ is Markovian with respect to the state. Furthermore, $s_{t+1}$ and $x_{t+1}$ are conditionally independent given the history ($H_t$), current state ($s_t$), and action ($a_t$).

\noindent    
\textbf{(A3)} Similarly, events that occurred in the distant past also have a reduced influence on the probability distribution of the current event mark. Specifically, there exists a convergent series $\sum_T N_T$ of real numbers such that for event histories $H_t$ and $H'_t$ at time $t$ that differ at only one event at time $t'$, if $t - t' \geq T$, then
	$$
	\mathsf{TV} \left( Q^X_{H_t}, Q^X_{H'_t} \right) < N_T.
	$$

Along with the above assumptions \textbf{(A1-A3)}, the MDP $\M$ under the influence of the exogenous event process must satisfy a set of regularity conditions to guarantee the existence and uniqueness of optimal solutions.

Define a new expected reward function $r_P: \S \times \A \rightarrow \R$ induced from $r$ by $P$ as
\[
r_P(s, a) = \E_{s' \sim P(. | s, a)} \left[ r(s, a, s') \right],
\]
where $P$ is any stochastic kernel on $\S$ given $\S \times \A$. We assume the following regularity conditions on $r$ and $Q$.

\noindent
\textbf{(B1)} $r_P$ is a bounded measurable function. Without any loss of generality, we assume that its codomain is $[0, 1] \subset \R$.

\noindent    
\textbf{(B2)} The function $r_P(s, .)$ is upper semi-continuous for each state $s \in \S$ and any probability distribution $P = Q_H$ induced as a transition kernel on the MDP by any possible external event history $H$ of the temporal process.

\noindent
\textbf{(B3)} $r_P$ is sup-compact on $\S \times \A$, i.e., for every $s \in \S, r \in \R$ and any probability distribution $P = Q_H$ induced as a transition kernel on the MDP by any possible external event history $H$ of the temporal process, the set $\{ a \in \A : r_P(s,a) \geq r \} (\subseteq \A)$ is compact.

\noindent
\textbf{(B4)} $Q$ is strongly continuous.

\noindent
\textbf{(B5)} $Q_H$ and $Q^X_H$ are strongly continuous for any feasible event history $H$.

Note that a transition kernel $P$ is said to be strongly continuous if
	$
		v(s, a) = \E_{s' \sim P (. | s,a)} [v(s')]
	$
	is continuous and bounded on $\S \times \A$ for every measurable bounded function $v$ on $\S$.

\subsection{Example}
\label{section:examples}
Non-Markov self-exciting event processes are studied extensively in finance, economics, epidemiology, and related fields, in the form of variants of the Hawkes and jump processes. Although they have been studied predominantly in continuous time, we consider discrete-time versions to fit into the traditional discrete-time reinforcement learning framework.

Consider a discrete-time marked Hawkes process $\left(B_{t}, X_t \right)_{t \in \N}$ defined as follows:
$\left( B_t \right)_{t \in \N}$ is a sequence of random variables taking values in $\left\{ 0, 1 \right\}$, with $B_{t}$ sampled from $\textsf{Bernoulli}\left( p_{t} \right)$, where the intensity $p_{t}$ at time $t$ depends on the realizations of $B_t$ at the previous time steps as
\begin{align*}
    p_{1}  = \alpha_0, \text{ and} \:\:\:
    p_{t} = \alpha_0 + \sum_{t' = 1}^{t - 1} \alpha_{t-t'} B_{t'} 
\end{align*}
and the marks $X_t$ are real-valued Gaussian random variables clipped to $[-b,b]$, satisfying
\begin{align*}
    X_{1} & \begin{dcases} \sim \cN_{[-b,b]} \left( 0, \; 1 \right) & \text{ if } B_1 = 1, \\
                    = 0 & \text{ otherwise,} 
                    \end{dcases} \text{ for } t > 1, \\
    X_{t} & \begin{dcases} \sim \cN_{[-b,b]} \left( \sum_{t' = 1}^{t-1} \beta_{t - t'} B_{t'} X_{t'}, \; 1 \right) & \text{ if } B_t = 1, \\
                    = 0 & \text{ otherwise,} 
                    \end{dcases} \text{ for } t > 1,
\end{align*}
where $\cN_{[-b,b]}$ is the normal distribution clipped to be within the interval $[-b,b]$, and $(\alpha_t)$ and $(\beta_t)$ are sequences that converge sufficiently quickly to zero.
The sum $S_t = \sum_{t' = 1}^t B_t$ is a self-exciting counting process that is the discrete-time counterpart of the continuous-time Hawkes process.

The above-defined sequence of random variables satisfies Assumptions \textbf{(A1)} and \textbf{(A2)}, which are required for the external process. $X_t$ is zero when no event occurs, and old events have a reduced and decaying influence on the current events. More precisely, given two historical sequences $(X_{t'})_{t' < t}$ and $(X'_{t'})_{t' < t}$ that only differ by one event at some time $t' \leq t - T$ with $B'_{t'} = X'_{t'} = 0$ and $B_{t'} = 1$, $X_{t'} = x$, the difference in intensities is $p_t - p'_t = \alpha_{t - t'}$, and hence,
\begin{align}
    \mathsf{TV} \left( X_t, X'_t \right) & \leq \frac{\alpha_{t-t'}}{2} + p_t \mathsf{TV} \left( \cN_{[-b,b]} \left(\sum_{t'' = 1}^{t-1} \beta_{t - t''} B_{t''} X_{t''}, \; 1 \right), \cN_{[-b,b]} \left(\sum_{t'' = 1}^{t-1} \beta_{t - t''} B_{t''} X'_{t''}, \; 1 \right) \right) \nonumber \\
    & \leq \frac{\alpha_{t-t'}}{2} + p_t \mathsf{TV} \left( \cN \left(\sum_{t'' = 1}^{t-1} \beta_{t - t''} B_{t''} X_{t''}, \; 1 \right), \cN\left(\sum_{t'' = 1}^{t-1} \beta_{t - t''} B_{t''} X'_{t''}, \; 1 \right) \right) \nonumber \\
        & = \frac{\alpha_{t-t'}}{2} + p_t \; \mathsf{erf}\left( \frac{\beta_{t-t'} x}{2\sqrt{2}} \right) 
         < \frac{\alpha_T}{2} + \mathsf{erf}\left( \frac{\beta_{T}b}{2\sqrt{2}} \right) = N_T.
    \label{eqn:nonstationary_tv_example}
\end{align}
One can find the details of this calculation in~Appendix~\ref{Appendix:TV-erf-calculation}.
That is, for sufficiently fast-converging $\alpha_T$ and $\beta_T$, the total variation perturbation $N_T$ due to events older than $T$ time steps approaches zero, satisfying Assumption \textbf{(A2)}. 

While $(\alpha_n)$ and $(\beta_n)$ are any general sequences that need to satisfy some regularity conditions \citep{Seol:2015:LimitThmsForDiscreteHawkes} and the above bound, they could also decay in a more structured manner as follows: For instance, for some $c, \lambda > 0$, letting
$
\alpha_n = c e^{-\lambda n}
$
gives a process with exponentially decaying intensity due to past events. In the terminology employed to describe the standard continuous-time Hawkes process, $\alpha_0 = c$ is like the base or background intensity, and $e^{-\lambda x}$ is the excitation function. Similarly, $\beta_n$, which determines the distribution of the event, can also be a parametric sequence that decays exponentially or polynomially.

\section{Guarantees for the Existence of Almost-Optimal Solutions}
\label{section:guarantees}

Consider the MDP $\M = (\S, \A, Q, r, \gamma)$, where $\S$ and $\A$ are closed subsets of $\R^m$ and $\R^n$ respectively, and $r: \S \times \A \times \S \rightarrow \R$ is the reward as a function of the state, action, and resultant next state.
Given an external discrete-time temporal process $X = \{(t_i, X_{t_i})\}_i$ that influences MDP $\M$, we construct a new augmented MDP $\M_{X}$, where the marks of all the events that occurred up to time $t$ are appended to the state $s \in \S$ to form a new augmented state $\bar{s} \in \bar{\S}$, where
$$
\bar{s} = (s, x_t, x_{t-1}, x_{t-2}, x_{t-3}, \dots), \text{ and } \bar{\S} = \S \times \prod_{t=0}^\infty \X \left( \subseteq \R^m \times \prod_{t=0}^\infty \R^d \right).
$$
Here, $x_t$ is the mark of an external event that occurs at time $t$. If no event occurs at time $t$, then $x_t = 0$. 

Due to the influence of external events, the decision process $(\S, \A, Q, r, \gamma)$ does not remain an MDP because the transition function is perturbed by external events. However, the corresponding decision process $\M_X = (\bar{\S}, \A, \bar{Q}, r, \gamma)$ is an MDP, where $\bar{Q}$ is the transition kernel on the augmented state $\bar{s}$, because all factors that affect the transitions are included in the augmented state $\bar{s} \in \bar{\S}$.

Furthermore, the state space of $\M_{X}$ is infinite-dimensional. However, Assumptions \textbf{(A2)} and \textbf{(A3)} in Section~\ref{section:formulation} allow one to study the possibility of finding a suitable policy that only depends on a finite horizon of past events. In this context, we present the following results. 
\begin{theorem}
	Suppose the MDP $\M$ and the external process satisfy the assumptions \textbf{(A1-A3)} and the regularity conditions \textbf{(B1-B5)}. Then, we have the following.
	\begin{enumerate}[(1)]
	\item There exists a deterministic optimal policy for the augmented MDP $\M_X$, with corresponding optimal value function $V^*$, that satisfies the optimal Bellman equation.
	\item For any $\epsilon> 0$, there exists a time horizon $T$ and a policy $\pi^{(T)}$ that depends only on the current state and past $T$ events such that
	$$
		V^{\pi^{(T)}} \geq V^* - \epsilon.
	$$
	That is, for every required maximum suboptimality $\epsilon$, there exists a corresponding ``finite-horizon" policy that achieves that $\epsilon$-suboptimal value. Further, the past event horizon $T$  required for $\epsilon$-suboptimality is $T$ that satisfies
    \begin{equation}
    \sum_{t=T+1}^\infty M_t < \frac{\epsilon (1 -   \gamma)^2}{4} \text{  and} \sum_{t=T+1}^\infty N_t <  \frac{\epsilon (1 - \gamma)^2}{4}.
    \label{eqn:nonstationary_series_convergence}
    \end{equation}
\end{enumerate}
\label{thm:nonstationary_optimality}
\end{theorem}
This result guarantees that the formulated problem is well-posed, with an optimal solution that can be approximated by a tractable policy that depends only on the current state and a finite history of events. This approximation can be made as accurate as necessary, depending on the properties of the exogenous event process and the size of the event history. To prove Theorem~\ref{thm:nonstationary_optimality}, we establish the following result, whose proof is given in Appendix~\ref{proof:nonstationary_lemma_approx}.

\begin{lemma}
    Consider a new auxiliary MDP $\M_X^{(T)}$ that has the same underlying stochastic process as the augmented MDP $\M_X$, with one difference: only events in the past $T$ time steps affect the current state transition and event distribution, with events before that having no effect and being effectively zero. Let $\pi$ be an arbitrary policy as a function of the augmented state $\bar{s} \in \bar{S}$. That is, $\pi$ can be either stochastic or deterministic and can depend on the current state $s \in \S$ and any number of past events in $\X$. Then,
	$$
		\left\Vert V^\pi_{\M^{(T)}_X} -  V^\pi_{\M_X} \right\Vert_{\infty} \leq
			\frac{1}{1 - \gamma} \left( \left\Vert r \right\Vert_\infty + \gamma \left\Vert V^\pi_{\M^{(T)}_X} \right\Vert_\infty \right) \sum_{t=T+1}^\infty \left( M_t + N_t \right),
	$$
    where $r$ is the reward function, and $V^\pi_\M$ denotes the value function of policy $\pi$ in MDP $\M$.
\label{lemma:nonstationary_approx}
\end{lemma}

\paragraph{Proof of Theorem~\ref{thm:nonstationary_optimality}}
The state space is Borel because it is a countable product of Borel sets. Therefore, part (1) follows from the regularity conditions \textbf{(B1-B5)} satisfying the assumptions of Theorem 4.2.3 in \citet{Hernandez:2012:DiscreteTimeMCPs}, coupled with the boundedness of the reward function. For part (2), let $\pi^*$ be the deterministic optimal stationary policy for $\M_X$ and $\pi^{*(T)}$ be the deterministic optimal stationary policy for $\M_X^{(T)}$. From Lemma~\ref{lemma:nonstationary_approx}, since $\Vert r \Vert_\infty \leq 1$ and $\Vert V \Vert_\infty \leq \frac{1}{1 - \gamma}$,
\begin{align*}
	V^{\pi^{*(T)}}_{\M_X} \geq V^{\pi^{*(T)}}_{\M^{(T)}_X} - \frac{1}{(1 - \gamma)^2} \left( \sum_{t=T+1}^\infty \left(M_t + N_t \right) \right) 
	 & \geq V^{\pi^{*}}_{\M^{(T)}_X} - \frac{1}{(1 - \gamma)^2} \left( \sum_{t=T+1}^\infty \left(M_t + N_t \right) \right) \\
	& \geq V^{\pi^{*}}_{\M_X} - \frac{2}{(1 - \gamma)^2} \left( \sum_{t=T+1}^\infty \left(M_t + N_t \right) \right).
\end{align*}
Since both series $\sum_t M_t$ and $\sum_t N_t$ converge, there exists $T \in \N$ satisfying~\eqref{eqn:nonstationary_series_convergence}. Choosing such a $T$ proves the result. \hfill $\blacksquare$

\section{A Policy Iteration Algorithm}
\label{section:policy_iteration}
Considering that we have established the existence of an $\epsilon$-optimal policy that is a function of only a finite event horizon, the aim is to find such a policy. To this end, we propose a policy iteration algorithm that alternates between approximate policy evaluation and approximate policy improvement. The procedure is listed in Algorithm~\ref{alg:nonstationary_policy_iteration}.

The policy evaluation step considers candidate value functions defined over a finite event horizon. We define the value function in $\M_X$ at an infinitely augmented state as a function of the state augmented by just a finite event horizon, by sampling the previous older events from the posterior distribution $\mu_1$ of the event process. In practice, $\mu_1$ is the actual event process. Therefore, the value function can be approximated using Monte Carlo methods. For the purpose of analysis, we consider the exact evaluation of the approximate value function.

In the policy improvement step, the policy is updated based on the reward and value functions derived from a single transition. This transition can be due to past events from an arbitrary distribution $\mu_2$. The deterministic policy at an augmented state depends only on the finite event horizon and the action that maximizes the right-hand side, which depends on the approximate value function.

Solving the optimization problem in the policy improvement step requires knowledge of the models $r(s, a, s')$, $Q_H$, and $Q_H^X$ for arbitrary histories $H$. In this work, we ignore any approximation errors arising from model estimation and analyze the algorithm by assuming the exact step is taken as defined.

\begin{algorithm}
\begin{algorithmic}
	\State Start with arbitrary deterministic policy $\pi_0: \S \times \X^{T+1} \rightarrow \A$\;
	\For{each $k \in \{0, 1, \dots \}$ till termination} 
		\State $//$ \texttt{Approximate Policy Evaluation}
            \State \begin{equation*}
				\hat{V}_k \left( \left( s, x_{0:\infty} \right) \right) = \E_{x'_{T+1:\infty} \sim \mu_1} V^{\pi_k}\left( \left( s, x_{0:T}, x'_{T+1:\infty} \right) \right).
		\end{equation*}
  
		\State $//$ \texttt{Approximate Policy Improvement}
            \State \begin{equation*}
			\pi_{k+1} \left( (s, x_{0:\infty}) \right) = \argmax_{a \in A} \; \E_{\begin{subarray}{l}{s' \sim Q_{x_{0:T}, x'_{T+1:\infty}}(\cdot | s,a)} \\ {x' \sim Q^X_{x_{0:T}, x'_{T+1:\infty}}(\cdot)} \\\\ {x'_{T+1:\infty} \sim \mu_2}\end{subarray}} \bigg[ r(s, a, s')  + \gamma \hat{V}_k \left( \left( s', x', x_{0:T-1} \right) \right) \bigg].
		\end{equation*}
	\EndFor
\end{algorithmic}
\caption{Policy Iteration}
\label{alg:nonstationary_policy_iteration}
\end{algorithm}

\subsection{Analysis}
\label{section:policy_iteration:subsection:analysis}
For the analysis of this algorithm, we establish the 
following results, which bound the difference in the value function at two different augmented states when they differ in events that are older than $T$ time steps.

\begin{lemma}
Let $\pi$ be a policy that depends only on the current state and past $T$ events. Then,
\begin{align*}
	\sup_{\bar{s} = (s, x_{0:\infty})} \left\vert V^{\pi} \left( \bar{s} \right) - V^{\pi} \left( \left( s, x_{0:T}, (0)_{T+1: \infty} \right) \right) \right\vert < \frac{1}{(1 - \gamma)^2} \sum_{t = T+1}^\infty \left( M_t + N_t \right).
\end{align*}
\label{lemma:nonstationary_state_cropping}
\end{lemma}
\noindent
Here, $(s, x_{0:T}, (0)_{T+1: \infty}) = (s, x_0, x_1, \dots, x_T, 0, 0, \dots) \in \bar{S}$ is the extended state that has been essentially ``truncated'' and made finite by replacing events older than $T$ time steps with zero.
 The proof of Lemma~\ref{lemma:nonstationary_state_cropping} is provided in Section~\ref{proof:nonstationary_lemma_state_cropping}.
 
\begin{corollary}
For a given policy $\pi$ that is a function of events only in the past $T$ time steps, and for any two augmented states $\bar{s}_1$ and $\bar{s}_2$ that have the same state and differ in any number of events that occurred before the previous $T$ time steps,
\begin{equation*}
\left\vert V^{\pi} \left( \bar{s}_1 \right) - V^{\pi} \left( \bar{s}_2 \right) \right\vert \leq \frac{2}{(1 - \gamma)^2} \sum_{t = T+1}^\infty \left( M_t + N_t \right).
\end{equation*}
\label{corollary:nonstationary_corollary_state_cropping}
\end{corollary}
Essentially, this means that if a policy considers only events in the past $T$ time steps, its value at states that differ at older events varies by an amount proportional to the extent of nonstationarity induced by events older than $T$ time steps.
This helps characterize the behavior of the policy iteration procedure described in Algorithm~\ref{alg:nonstationary_policy_iteration}. While exact policy iteration results in a sequence of policies that converge to the optimal policy, our approximate policy iteration produces policies whose value functions satisfy the following result, the proof of which is provided later in this section.

\begin{lemma}
\begin{equation*}
	V^{\pi_{k+1}} \geq V^{\pi_k} -  \frac{4}{(1-\gamma)^3}\sum_{t=T+1}^{\infty}(M_t+N_t).
\end{equation*}
\label{lemma:nonstationary_policy_improvement}
\end{lemma}

That is, while there is no guarantee that the sequence of value functions $\left\{ V^{\pi_k} \right\}_{k}$ is non-decreasing, Lemma~\ref{lemma:nonstationary_policy_improvement} guarantees that they do not decrease by more than a specific amount. This is due to the approximation error $\epsilon  = \displaystyle\frac{4}{(1 - \gamma)^{2}} \sum_{t = T+1}^\infty (M_t + N_t)$ at each time step, which is induced by the nonstationarity due to exogenous events.

In general, approximate policy iteration algorithms converge under the assumption that the approximation errors gradually vanish over the course of the algorithm. However, our approximation error of $\epsilon$ remains constant throughout, leading to the aforementioned situation. Intuitively, such a small approximation error $\epsilon$ should only disturb the convergence of the algorithm when we are $\epsilon$-close to the optimal solution. During the initial iterations, when far from the solution, such small approximation errors are insignificant.

Next, we establish that the value function can degrade only in regions of the state space where the Bellman error is approximately zero. A Bellman error of zero generally, although not always, corresponds to an optimal policy. This interplay between the Bellman error and approximation error in our algorithm is formalized by the following result.

\begin{theorem}
At iteration $k$ of the policy iteration procedure given in Algorithm~\ref{alg:nonstationary_policy_iteration}, for any augmented state $\bar{s} \in \bar{\S} = \S \times \displaystyle \prod_{t=0}^\infty \X$, at least one of the following holds:

\begin{align}
 V^{\pi_{k+1}}(\bar{s}) &\geq V^{\pi_k}(\bar{s}), \text{ or} \label{eqn:nonstationary_improvement} \\
	 \left\vert \T V^{\pi_k}(\bar{s}) - V^{\pi_k}(\bar{s}) \right\vert & < { \frac{4(2-\gamma)}{(1 - \gamma)^3}}\sum_{t=T+1}^{\infty}(M_t+N_t), \label{eqn:nonstationary_small_bellman_error}
\end{align}

where $V^\pi$ denotes the value function of policy $\pi$ in the MDP $\M_X$.
\label{thm:nonstationary_policy_iteration}
\end{theorem}

That is, the policy’s performance improves everywhere except in the region of the state space, where the Bellman error is very small. The faster the decay of the exogenous event influence, the smaller the approximation error and the smaller the region of the state space in which the policy may not improve.

\begin{corollary}
    Consider an MDP with the exogenous process being a discrete-time Hawkes process as described in Section~\ref{section:examples}, with an exponentially decaying excitation function $c_\alpha e^{-\lambda_\alpha t}$, along with an exponentially decaying $\beta_t = c_\beta e^{- \lambda_\beta t}$ and state transition perturbations that decay as $M_t = \bar{c} e^{-\bar{\lambda}t}$. Then, each step $k$ of the policy iteration algorithm is guaranteed to keep the value function of the policy non-decreasing everywhere in the state space except regions $\bar{s} \in \bar{\S}$ satisfying
   
    $$
        \left\vert \T V^{\pi_k}(\bar{s}) - V^{\pi_k}(\bar{s}) \right \vert < {\frac{4(2-\gamma)}{(1 - \gamma)^3}} \left( \frac{\bar{c}}{\lambda} e^{-\bar{\lambda}T} + \frac{c_\alpha}{\lambda_\alpha} e^{-\lambda_\alpha T} + \frac{c_\beta b}{\sqrt{2 \pi} \lambda_\beta} e^{-\lambda_\beta T} \right).
    $$
    \label{corollary:nonstationary_policy_iteration_hawkes}
\end{corollary}
  
As the time window $T$ considered increases, the error reduces as $e^{-\left\{ \bar{\lambda}, \lambda_\alpha, \lambda_\beta \right\} T}$. Hence, depending on the rate of influence decay, increasing the time window beyond a certain point results in diminishing returns. The error also has a proportional dependence on $c_\alpha$, the base intensity of the Hawkes process, and an inverse dependence on $\lambda_\alpha$, which is the decay of the excitation function of the Hawkes process.

\paragraph{Proof of Theorem~\ref{thm:nonstationary_policy_iteration}}
Whether \eqref{eqn:nonstationary_improvement} or \eqref{eqn:nonstationary_small_bellman_error} holds
depends on the state $\bar{s}$ falls in the ``sub-optimality" set $B$ or not, which is the set of all states $\bar{s} = (s, x_{0:\infty}) \in \bar{\S}$ that satisfy
\begin{align*}
& \E_{\begin{subarray}{l} s' \sim Q_{x_{0:T}, x'_{T+1:\infty}}(. | s,a) \\\\ x' \sim Q^X_{x_{0:T}, x'_{T+1:\infty}}(.) \\\\ x'_{T+1:\infty} \sim \mu_2 \\\\ a \sim  \pi_k(\bar{s}) \end{subarray}} \left[ r(s,a,s') + \gamma \hat{V}_k((s', x', x_{0:T-1})) \right] \\
& \qquad \qquad \qquad < \max_{a \in \A} \E_{\begin{subarray}{l} s' \sim Q_{x_{0:T}, x'_{T+1:\infty}}(. | s,a) \\\\ x' \sim Q^X_{x_{0:T}, x'_{T+1:\infty}}(.) \\\\ x'_{T+1:\infty} \sim \mu_2 \end{subarray}} \left[ r(s,a,s') + \gamma \hat{V}_k((s', x', x_{0:T-1})) \right] - \delta
\end{align*}
for some suitable $\delta > 0$ that is yet to be chosen. It is to be noted that whether or not this condition holds, the above inequality or the corresponding equality holds for $\delta = 0$ by the definition of $\pi_{k+1}$.

By definition of $\hat{V}_k$ and using corollary \ref{corollary:nonstationary_corollary_state_cropping}, we have
\begin{align}
\hat{V}_k((s,x_{0:T})) & = \E_{\substack{x'_{T+1:\infty} \sim \mu_1}} \left[ V^{\pi_k}((s,x_{0:T},x'_{T+1:\infty})) \right] \label{ineq:policy_iteration_1:proof:01} \\
\vert V^{\pi_k}((s,x_{0:\infty})) - \hat{V}_k((s,x_{0:T})) \vert & = \vert \E_{\substack{x'_{T+1:\infty} \sim \mu_1}} \left[ V^{\pi_k}((s,x_{0:\infty})) - V^{\pi_k}((s,x_{0:T},x'_{T+1:\infty})) \right] \vert \nonumber \\
& \leq \E_{\substack{x'_{T+1:\infty} \sim \mu_1}} \left[ \frac{2}{(1-\gamma)^2}\sum_{t=T+1}^{\infty} (M_t + N_t) \right] \label{ineq:policy_iteration_1:proof:02}
\end{align}

Using the Bellman equation, we have
\begin{align}
V^{\pi_k}((s,x_{0:\infty})) & = \E_{\begin{subarray}{l} s' \sim Q_{x_{0:\infty}}(.|s,a) \\ x' \sim Q^X_{x_{0:\infty}}(.) \\ a \sim \pi_k(\bar{s}) \end{subarray}} \left[ r(s,a,s') + \gamma V^{\pi_k}((s',x',x_{0:\infty})) \right] \nonumber \\
V^{\pi_k}((s,x_{0:\infty})) & \leq \E_{\begin{subarray}{l} s' \sim Q_{x_{0:\infty}}(.|s,a) \\\\ x' \sim Q^X_{x_{0:\infty}}(.) \\\\ a \sim \pi_k(\bar{s}) \end{subarray}} \left[ r(s,a,s') + \gamma \hat{V}_k((s',x',x_{0:T-1})) \right] \nonumber \\
& \qquad \qquad + \frac{2\gamma}{(1-\gamma)^2}\sum_{t=T+1}^{\infty} (M_t + N_t) \label{ineq:policy_iteration_1:proof:03} \\
& \leq \E_{\begin{subarray}{l} s' \sim Q_{x_{0:T},x'_{T+1:\infty}}(.|s,a) \\\\ x' \sim Q^X_{x_{0:T},x'_{T+1:\infty}}(.) \\\\ x'_{T+1:\infty} \sim \mu_2 \\\\ a \sim \pi_k(\bar{s}) \end{subarray}} \left[ r(s,a,s') + \gamma \hat{V}_k((s',x',x_{0:T-1})) \right] \nonumber \\
& \qquad \qquad + \frac{2}{1-\gamma}\sum_{t=T+1}^{\infty}(M_t + N_t) + \frac{2\gamma}{(1-\gamma)^2}\sum_{t=T+1}^{\infty} (M_t + N_t) \label{ineq:policy_iteration_1:proof:04}
\end{align}

The inequality in \eqref{ineq:policy_iteration_1:proof:03} follows from \eqref{ineq:policy_iteration_1:proof:02}, and the inequality in \eqref{ineq:policy_iteration_1:proof:04} follows from \eqref{eqn:function_different_histories:upper_bound}. 

Suppose that $\bar{s} \in B$. Then, 

\begin{align}
V^{\pi_k}((s,x_{0:\infty})) & \leq \E_{\begin{subarray}{l} s' \sim Q_{x_{0:T},x'_{T+1:\infty}}(.|s,a) \\\\ x' \sim Q^X_{x_{0:T},x'_{T+1:\infty}}(.) \\\\ x'_{T+1:\infty} \sim \mu_2 \\\\ a \sim \pi_{k+1}(\bar{s}) \end{subarray}} \left[ r(s,a,s') + \gamma \hat{V}_k((s',x',x_{0:T-1})) \right] \nonumber \\
& \qquad \qquad + \frac{2}{(1-\gamma)^2} \sum_{t=T+1}^{\infty}(M_t+N_t) - \delta \label{ineq:policy_iteration_1:proof:05}\\
& \leq \E_{\begin{subarray}{l} s' \sim Q_{x_{0:T},x'_{T+1:\infty}}(.|s,a) \\\\ x' \sim Q^X_{x_{0:T},x'_{T+1:\infty}}(.) \\\\ x'_{T+1:\infty} \sim \mu_2 \\\\ a \sim \pi_{k+1}(\bar{s}) \end{subarray}} \left[ r(s,a,s') + \gamma V^{\pi_k}((s',x',x_{0:\infty})) \right] \nonumber \\
& \qquad + \frac{2\gamma}{(1-\gamma)^2} \sum_{t=T+1}^{\infty}(M_t+N_t) + \frac{2}{(1-\gamma)^2} \sum_{t=T+1}^{\infty}(M_t+N_t) - \delta \label{ineq:policy_iteration_1:proof:06}\\
& \leq \E_{\begin{subarray}{l} s' \sim Q_{x_{0:\infty}}(.|s,a) \\\\ x' \sim Q^X_{x_{0:\infty}}(.) \\\\ a \sim \pi_{k+1}(\bar{s}) \end{subarray}} \left[ r(s,a,s') + \gamma V^{\pi_k}((s',x',x_{0:\infty})) \right] \nonumber \\
& \qquad + \frac{2}{1-\gamma} \sum_{t=T+1}^{\infty}(M_t+N_t) + \frac{2+2\gamma}{(1-\gamma)^2} \sum_{t=T+1}^{\infty}(M_t+N_t) - \delta \label{ineq:policy_iteration_1:proof:07}\\
& = \E_{\pi_{k+1}} \left[ r(s,a,s') + \gamma V^{\pi_k}(\bar{s}') \right] + \frac{4}{(1-\gamma)^2} \sum_{t=T+1}^{\infty}(M_t+N_t) - \delta \label{ineq:policy_iteration_1:proof:08}\\
& \leq \E_{\pi_{k+1}}\left[ \sum_{t=0}^{\infty} \gamma^t r(s,a,s') \right] + \frac{1}{1-\gamma}\left[ \frac{4}{(1-\gamma)^2}\sum_{t=T+1}^{\infty}(M_t+N_t) \right] - \delta \label{ineq:policy_iteration_1:proof:09}\\
& = V^{\pi_{k+1}}(\bar{s}) + \frac{4}{(1-\gamma)^3}\sum_{t=T+1}^{\infty}(M_t+N_t) - \delta \label{ineq:policy_iteration_1:proof:10}\\
& = V^{\pi_{k+1}}(\bar{s}), \nonumber
\end{align}

for $\delta = \frac{4}{(1-\gamma)^3}\sum_{t=T+1}^{\infty}(M_t+N_t)$. 

The inequality in \eqref{ineq:policy_iteration_1:proof:05} follows from the fact that $\bar{s} \in B$. The inequality in \eqref{ineq:policy_iteration_1:proof:06} follows from \eqref{ineq:policy_iteration_1:proof:02}, and the inequality in \eqref{ineq:policy_iteration_1:proof:07} follows from \eqref{eqn:function_different_histories:upper_bound}. For the initial augmented state $\bar{s}\in B$, inequality \eqref{ineq:policy_iteration_1:proof:08} holds with the term $-\delta$. For all augmented states $\bar{s}\in\bar{\mathcal S}$, the same inequality holds with $\delta=0$, as the approximate greedy improvement does not decrease the approximate one-step value. Therefore, when recursively expanding \eqref{ineq:policy_iteration_1:proof:08}, we retain the term $-\delta$ at the initial state and use the inequality with $\delta=0$ at all subsequent states. This gives \eqref{ineq:policy_iteration_1:proof:09}.

Now, suppose $\bar{s} \notin B$. We need to show that this implies (\ref{eqn:nonstationary_small_bellman_error}), that is, the Bellman error should be extremely small. It is always true that $\T V \geq V$. In the other direction, for any $\bar{s} = (s, x_{0:\infty}) \in \bar{\S}$, 

Since $\bar{s} \notin B$,

\begin{align}
& \max_{a \in \A} \E_{\begin{subarray}{l} s' \sim Q_{x_{0:T}, x'_{T+1:\infty}}(. | s,a) \\\\ x' \sim Q^X_{x_{0:T}, x'_{T+1:\infty}}(.) \\\\ x'_{T+1:\infty} \sim \mu_2 \end{subarray} } \left[ r(s,a,s') + \gamma \hat{V}_k((s', x', x_{0:T-1})) \right] \nonumber \\
& \qquad \qquad \qquad \leq \E_{\begin{subarray}{l} s' \sim Q_{x_{0:T}, x'_{T+1:\infty}}(. | s,a) \\\\ x' \sim Q^X_{x_{0:T}, x'_{T+1:\infty}}(.) \\\\ x'_{T+1:\infty} \sim \mu_2 \\\\ a = \pi_k(\bar{s}) \end{subarray}} \left[ r(s,a,s') + \gamma \hat{V}_k((s', x', x_{0:T-1})) \right] + \delta \label{ineq:policy_iteration_1:proof:11}
\end{align}

We can lower bound bound the left hand side of \eqref{ineq:policy_iteration_1:proof:11} using \eqref{ineq:policy_iteration_1:proof:02} and \eqref{eqn:function_different_histories:upper_bound}
\begin{align}
& \max_{a \in \A} \E_{\begin{subarray}{l} s' \sim Q_{x_{0:T}, x'_{T+1:\infty}}(. | s,a) \\\\ x' \sim Q^X_{x_{0:T}, x'_{T+1:\infty}}(.) \\\\ x'_{T+1:\infty} \sim \mu_2  \end{subarray}} \left[ r(s,a,s') + \gamma \hat{V}_k((s', x', x_{0:T-1})) \right] \nonumber \\
& \qquad \qquad \qquad \geq \max_{a \in \A} \E_{\begin{subarray}{l} s' \sim Q_{x_{0:T}, x'_{T+1:\infty}}(. | s,a) \\\\ x' \sim Q^X_{x_{0:T}, x'_{T+1:\infty}}(.) \\\\ x'_{T+1:\infty} \sim \mu_2  \end{subarray}} \left[ r(s,a,s') + \gamma V^{\pi_k}((s', x', x_{0:\infty})) \right] \nonumber \\
& \qquad \qquad \qquad \qquad \qquad - \frac{2\gamma}{(1-\gamma)^2}\sum_{t=T+1}^{\infty}(M_t+N_t) \nonumber \\
& \qquad \qquad \qquad \geq \max_{a \in \A} \E_{\substack{s' \sim Q_{x_{0:\infty}}(. | s,a) \\ x' \sim Q^X_{x_{0:\infty}}(.) }} \left[ r(s,a,s') + \gamma V^{\pi_k}((s', x', x_{0:\infty})) \right] \nonumber \\
& \qquad \qquad \qquad \qquad \qquad  - \frac{2}{(1-\gamma)}\sum_{t=T+1}^{\infty}(M_t+N_t) - \frac{2\gamma}{(1-\gamma)^2}\sum_{t=T+1}^{\infty}(M_t+N_t) \label{ineq:policy_iteration_1:proof:12}
\end{align}

We can upper bound the right hand side of \eqref{ineq:policy_iteration_1:proof:11} using \eqref{ineq:policy_iteration_1:proof:02} and \eqref{eqn:function_different_histories:upper_bound}

\begin{align}
& \E_{\begin{subarray}{l}  s' \sim Q_{x_{0:T}, x'_{T+1:\infty}}(. | s,a) \\\\ x' \sim Q^X_{x_{0:T}, x'_{T+1:\infty}}(.) \\\\ x'_{T+1:\infty} \sim \mu_2 \\ a = \pi_k(\bar{s}) \end{subarray}} \left[ r(s,a,s') + \gamma \hat{V}_k((s', x', x_{0:T-1})) \right] \nonumber \\
& \qquad  \leq \E_{\begin{subarray}{l}  s' \sim Q_{x_{0:T}, x'_{T+1:\infty}}(. | s,a) \\\\ x' \sim Q^X_{x_{0:T}, x'_{T+1:\infty}}(.) \\\\ x'_{T+1:\infty} \sim \mu_2 \\\\ a = \pi_k(\bar{s}) \end{subarray}} \left[ r(s,a,s') + \gamma V^{\pi_k}((s', x', x_{0:\infty})) \right] 
+ \frac{2\gamma}{(1-\gamma)^2}\sum_{t=T+1}^{\infty}(M_t+N_t) \nonumber \\
& \qquad  \leq \E_{\begin{subarray}{l} s' \sim Q_{x_{0:\infty}}(. | s,a) \\\\ x' \sim Q^X_{x_{0:\infty}}(.) \\\\ a = \pi_k(\bar{s}) \end{subarray}} \left[ r(s,a,s') + \gamma V^{\pi_k}((s', x', x_{0:\infty})) \right] \nonumber \\
& \qquad \qquad \qquad \qquad \qquad + \frac{2}{(1-\gamma)}\sum_{t=T+1}^{\infty}(M_t+N_t) + \frac{2\gamma}{(1-\gamma)^2}\sum_{t=T+1}^{\infty}(M_t+N_t) \label{ineq:policy_iteration_1:proof:13}
\end{align}

Therefore, from \eqref{ineq:policy_iteration_1:proof:12} and \eqref{ineq:policy_iteration_1:proof:13}, and taking $\delta = \frac{4}{(1-\gamma)^3}\sum_{t=T+1}^{\infty}(M_t+N_t)$ we get

\begin{align*}
& \frac{-2\gamma}{(1-\gamma)^2}\sum_{t=T+1}^{\infty} (M_t+N_t) - \frac{2}{1-\gamma} \sum_{t=T+1}^{\infty}(M_t+N_t) \\
& \qquad \qquad \qquad \qquad \qquad \qquad + \max_{a \in \A} \E_{\begin{subarray}{l} s' \sim Q_{x_{0:\infty}}(. | s,a) \\\\ x' \sim Q^X_{x_{0:\infty}}(.) \end{subarray}} \left[ r(s,a,s') + \gamma V^{\pi_k}((s', x', x_{0:\infty})) \right] \\
& \qquad \leq \E_{\begin{subarray} s' \sim Q_{x_{0:\infty}}(. | s,a) \\\\ x' \sim Q^X_{x_{0:\infty}}(.) \\\\ a = \pi_k(\bar{s}) \end{subarray}} \left[ r(s,a,s') + \gamma V^{\pi_k}((s', x', x_{0:\infty})) \right] \\
& \qquad \qquad \qquad \qquad \qquad + \frac{2\gamma}{(1-\gamma)^2}\sum_{t=T+1}^{\infty} (M_t+N_t) + \frac{2}{1-\gamma} \sum_{t=T+1}^{\infty}(M_t+N_t) + \delta 
\end{align*}
Hence,
\begin{align*}
 (\T V^{\pi_k})(\bar{s}) - \frac{2}{(1-\gamma)^2}\sum_{t=T+1}^{\infty}(M_t+N_t) 
&\:\leq\: V^{\pi_k}(\bar{s}) + \frac{2}{(1-\gamma)^2}\sum_{t=T+1}^{\infty}(M_t+N_t) + \delta \\
 (\T V^{\pi_k})(\bar{s}) - V^{\pi_k}(\bar{s}) &\:\leq\: \frac{4}{(1-\gamma)^2}\sum_{t=T+1}^{\infty}(M_t+N_t) \\
 & \qquad \qquad + \frac{4}{(1-\gamma)^3}\sum_{t=T+1}^{\infty}(M_t+N_t)\\
 (\T V^{\pi_k})(\bar{s}) - V^{\pi_k}(\bar{s}) &\:\leq\: \frac{4(2-\gamma)}{(1-\gamma)^3}\sum_{t=T+1}^{\infty}(M_t+N_t) \\
\vert (\T V^{\pi_k})(\bar{s}) - V^{\pi_k}(\bar{s})\vert &\:\leq\: \frac{4(2-\gamma)}{(1-\gamma)^3}\sum_{t=T+1}^{\infty}(M_t+N_t)
\end{align*}


\paragraph{Proof of Lemma~\ref{lemma:nonstationary_policy_improvement}}
\label{proof:lemma:nonstationary_policy_improvement}
Setting $\delta = 0$ in the first part of the above proof gives us the Lemma, for which there is no assumption on the extent of policy improvement.
\begin{align*}
V^{\pi_k}(\bar{s}) & \leq V^{\pi_{k+1}}(\bar{s}) + \frac{1}{1-\gamma}\left[ \frac{4}{(1-\gamma)^2}\sum_{t=T+1}^{\infty}(M_t+N_t) \right] \\
V^{\pi_k}(\bar{s}) & \leq V^{\pi_{k+1}}(\bar{s}) +\frac{4}{(1-\gamma)^3}\sum_{t=T+1}^{\infty}(M_t+N_t)
\end{align*} \hfill $\blacksquare$

In Theorem~\ref{thm:nonstationary_policy_iteration}, in regions of the augmented state space without guaranteed policy improvement, the Bellman error depends on the approximation error due to the ignoring of old events. This is the cumulative effect of all events older than $T$ on the current events, as well as the state transition kernel. Specifically, it is proportional to
    $
     \sum_{t=T+1}^\infty \left( M_t + N_t \right),
    $
where $T$ is the time window beyond which events are discarded, $M_t$ is an upper bound on the total variation disturbance on the state transition kernel caused by an event older than $t$ time steps, and $N_t$ bounds the total variation effect due to old events on the event process itself. The faster the decay of the exogenous event influence, the smaller is the truncation error due to nonstationarity.

\paragraph{Proof of Corollary~\ref{corollary:nonstationary_policy_iteration_hawkes}}
From~\eqref{eqn:nonstationary_tv_example}, the extra error introduced is proportional to
    \begin{align*}
        \sum_{t=T+1}^\infty \left( M_t + N_t \right) & \leq \sum_{t=T+1}^\infty \left( \bar{c} e^{-\bar{\lambda} t} + c_\alpha e^{-\lambda_\alpha t} + \mathsf{erf} \left( \frac{c_\beta b e^{-\lambda_\beta t}}{2 \sqrt{2}} \right) \right) \\
        & \leq \sum_{t=T+1}^\infty \left( \bar{c} e^{-\bar{\lambda} t} + c_\alpha e^{-\lambda_\alpha t} + \sqrt{1 - \exp\left( \frac{-c_\beta^2 b^2}{2 \pi} e^{-2 \lambda_\beta t} \right)} \right) \\
        & \leq \sum_{t=T+1}^\infty \left( \bar{c} e^{-\bar{\lambda} t} + c_\alpha e^{-\lambda_\alpha t} + \frac{c_\beta b}{\sqrt{2 \pi}} e^{-\lambda_\beta t} \right) \\
        & \leq \int_{t=T}^\infty \left( \bar{c} e^{-\bar{\lambda} t} + c_\alpha e^{-\lambda_\alpha t} + \frac{c_\beta b}{\sqrt{2 \pi}} e^{-\lambda_\beta t} \right) \d t \\
        & = \frac{\bar{c}}{\lambda} e^{-\bar{\lambda}T} + \frac{c_\alpha}{\lambda_\alpha} e^{-\lambda_\alpha T} + \frac{c_\beta b}{\sqrt{2 \pi} \lambda_\beta} e^{-\lambda_\beta T}. 
    \end{align*} \hfill $\blacksquare$

\section{Least-Squares Policy Iteration}
\label{section:lspi}
Although the algorithm presented in Section~\ref{section:policy_iteration} specifies the value function and its update, in practice this function must be learned from samples acquired from the environment and is typically approximated using a selected class of functions. 

In this section, we analyze the sample complexity associated with policy iteration, in which the policy evaluation step is achieved by employing pathwise Least-Squares Temporal Difference (LSTD) learning ~\citep{Lazaric:2012:FiniteSampleAnalysisOfLSPI}. This method utilizes samples derived from a single sample path generated by the policy.  We begin by establishing bounds on the evaluation error in the linear function-space case, along with certain regularity conditions imposed on the MDP. Subsequently, we employ these results to bound the suboptimality of the least-squares policy iteration.

In contrast to \citet{Lazaric:2012:FiniteSampleAnalysisOfLSPI}, our study deals with an MDP characterized by an infinite-dimensional augmented state space, a stochastic reward function, and, more importantly, an additional error due to the use of tractable policies and value functions in a finite-dimensional domain. This results in additional sample complexity, even when evaluating a given policy, depending on the extent to which exogenous events affect state evolution. In this section, we quantify the precise nature of this error.

Consider a function class $\F$ consisting of functions that can approximate the value function on the augmented MDP, 
$\F \subset \left\{ f: \bar{\S} \rightarrow \R  \right\}$.
Since the domain of the functions in this class is infinite-dimensional, this can be further broken down into two approximations for dealing with these functions practically: a truncation operation to make the domain finite-dimensional, and then another class of functions that are then used to cover possible functions on this domain, i.e., 
\begin{displaymath}
	\F = \left\{ f \; \circ \; f^{(T)}_{\mathrm{trunc}} : f \in \F^{(T)} \right\}, \text{ and } f^{(T)}_{\mathrm{trunc}} : \bar{\S} \rightarrow \S \times \prod_{t=0}^T \X,
\end{displaymath}

where
\begin{displaymath}
	 \F^{(T)} = \left\{ f : \S \times \prod_{t=0}^T \X \rightarrow \R \right\}
\end{displaymath}

Here, $f^{(T)}_{\mathrm{trunc}}$ is a fixed truncation function that removes extra events that are older than $T$ time steps while preserving the current state and new events, and $\F^{(T)}$ is a function class defined on this new truncated domain. For a fixed $T$, finding the best function in $\F$ is the same as finding the best function in $\F^{(T)}$. 


The standard approximating function space considered is an arbitrary, finite-dimensional function space. Specifically, assume that there are $d$ basis functions $\varphi_i: \S \times \X^{T+1} \rightarrow \R$ that linearly span $\F^{(T)}$ and are bounded by $L$. That is, for each $f \in \F^{(T)}$, there exists $\alpha \in \R^d$ such that $f = \sum_{i \in [d]} \alpha_i \varphi_i$. These basis functions together form a feature representation function $\phi: \S \times \X^{T+1} \rightarrow \R^d$, defined by $\phi(x) = ( \varphi_1(x), \dots, \varphi_d(x) )$. Furthermore, we define $\widetilde{\F}$ as the class of functions obtained by truncating the functions in $\F$ to be bounded by $L$.

We now consider adapting the standard Least-Squares Policy Iteration to our setting. The policy evaluation step is described in Section~\ref{section:lstd}, where a value function is learned as a function of the current state and a finite history of events, with the values of previous events set to zero for simplicity. In the policy improvement step, a greedy policy is defined as a function of the state and a finite history of events. This induces an additional approximation error in the model. The overall procedure is listed in  Algorithm~\ref{alg:nonstationary_least_squares_policy_iteration}.
\begin{algorithm}[h]
\begin{algorithmic}
        \State \textbf{Given:} Basis function set $\phi = \left( \varphi_1, \dots, \varphi_d \right): \S \times \X^{T+1} \rightarrow \R^d$
	\State Start with arbitrary deterministic policy $\pi_0: \S \times \X^{T+1} \rightarrow \A$\;
	\For{each $k \in \{0, 1, \dots \}$ till termination} 
		\State $//$ \texttt{Least-Squares Temporal Difference Learning for policy evaluation}
            \State Obtain samples $\bar{s}_1, \dots, \bar{s_n}$ and rewards $r_1, \dots, r_n$ using policy $\pi_k$
            \State Construct the feature vectors as
            \begin{equation}
                \Phi = \left( \phi\left( \bar{s}_1^{(T)} \right), \dots, \Phi \left( \bar{s}_n^{(T)} \right) \right).
            \end{equation}
            \State Obtain value function $\widetilde{V} = \textsf{Trunc} \left( \sum_{i \in [d]} \hat{\alpha}_i \varphi_i  \right)$, where
            \begin{equation*}
                \hat{\alpha} = \left[ \Phi^\transpose \left( I - \gamma \hat{P} \right) \right]^+ \Phi^\transpose r.
            \end{equation*}
  
		\State $//$ \texttt{Approximate Greedy Policy Improvement}
            \State Define new policy as the best approximating greedy policy that is a function of the state and just a finite history of events, as
            \begin{equation*}
			\pi_{k+1} \left( (s, x_{0:\infty}) \right) = \argmax_{a \in A} \; \E_{\begin{subarray}{l} s' \sim Q_{x_{0:T}, (0)}(. | s,a) \\\\ x' \sim Q^X_{x_{0:T}, (0)}(.|s) \end{subarray}} \bigg[ r(s, a, s')  + \gamma \tilde{V}_k \left( \left( s', x', x_{0:T-1} \right) \right) \bigg].
		\end{equation*}
	\EndFor
\end{algorithmic}
\caption{Approximate Least-Squares Policy Iteration}
\label{alg:nonstationary_least_squares_policy_iteration}
\end{algorithm}

\subsection{Approximate Policy Evaluation}
\label{section:lstd}
The error between the true value function and the learned truncated value function can be broken down in terms of the errors introduced by truncating the infinite-dimensional domain and employing linear function approximation, as well as the stochasticity of the samples used to learn the approximate function. The aim is to provide a bound to 
$
	\Vert V - \widetilde{V} \Vert_\rho,
$
where $V$ is the true value function on $\bar{\S}$, $\widetilde{V}$ is the learned value function in $\F$, $\rho$ is some distribution over $\bar{\S}$, and the norm $\left\Vert \cdot \right\Vert_\rho$ is the $l^2(\rho)$-norm, which is the expected $l^2$ norm of the value of the function w.r.t the measure $\rho$, i.e,
$$
	\left\Vert f \right\Vert^2_\rho = \int_{\bar{\S}} f(x)^2 \; \rho( \d x ).
$$
While the $f^{(T)}_{\mathrm{trunc}}$ operator ignores old events, we also define a projection operator $\Pi^{\mathrm{trunc}}$ onto the function space defined on the truncated domain, as
$$
    \Pi^{\mathrm{trunc}}V = \argmin_{f: \S \times \prod_{t=0}^T \X \rightarrow \R} \left\Vert f - V \right\Vert_{\rho}.
$$
which finds the best approximation that does not depend on the old events, where `best' is defined in expectation with respect to $\rho$.
The best approximate value function $\widetilde{V}$ is that function in $\F$ (or equivalently $\F^{(T)}$) that minimizes the above-expected norm difference.

Suppose we are given $N$ samples $\bar{s}_1, \dots, \bar{s}_N$ from a Markov chain induced by policy $\pi$ in the augmented MDP $\M_X$ and the corresponding rewards $r_1, \dots, r_N$. We minimize the empirical norm difference at these points, defined as
$$
	\left\Vert f \right\Vert_N = \left( \frac{1}{N} \sum_{t=1}^N f(x_t)^2 \right)^{\frac{1}{2}}.
$$
The above function defines a norm, along with a corresponding inner product, in a new inner product space $\F_N$, which is the space of all $N$ values of the functions at the given sample points. This can be considered a subset of $\R^N$ with an inner product. Here, $\F_N   = \{ ( f(\bar{s}_1), \dots, f(\bar{s}_N) ) : f \in \F \} 
		 = \{ ( f ( s^{(T)}_1 ), \dots, f( s^{(T)}_N ) ) : f \in \F^{(T)} \} 
         = \{ \Phi \alpha: \alpha \in \R^d \}$,  where 
         $\Phi = ( \phi ( \bar{s}_1^{(T)} ), \dots, \phi ( \bar{s}_N^{(T)} ) )_{N \times d}$  and  $\bar{s}_i^{(T)} = f^{(T)}_{\mathrm{trunc}} \left( \bar{s}_i \right)$.
In pathwise LSTD, given a sequence of states $\bar{s_1}, \dots, \bar{s_n}$ and corresponding rewards $r_1, \dots, r_n$ obtained by following policy $\pi$, the value function is approximated as
$$
    \hat{V} = \sum_{i \in [d]} \hat{\alpha}_i \varphi_i, \text{ for } \hat{\alpha} = \left[ \Phi^\transpose \left( I - \gamma \hat{P} \right) \right]^+ \Phi^\transpose r,
$$
where $\Phi = ( \phi ( \bar{s}_1^{(T)} ), \dots, \phi ( \bar{s}_N^{(T)} ) )_{N \times d}$ is the feature matrix,  $\bar{s}_i^{(T)} = f^{(T)}_{\text{trunc}}(\bar{s}_i)$, and $\hat{P}_{n \times n} = \left( \mathbb{I} \left\{ j = i+1 \right\} \right)_{i,j}$ is the pathwise empirical transition matrix.
Since the reward function takes values in $[0,1]$, the true value function is bounded by $1 / (1 - \gamma)$. Therefore, to obtain the best approximation, the above obtained $\widehat{V}$ can be truncated to take values in $[0, 1/(1-\gamma)]$ to obtain the final estimate $\widetilde{V}$.

\subsection{Sample complexity of Policy Evaluation}
The aim is to provide a bound to the expected error
$
	\left\Vert V - \widetilde{V} \right\Vert_\rho,
$
where $V$ is the true value function on $\bar{\S}$, $\widetilde{V}$ is the learned value function in $\F$, $\rho$ is a distribution over $\bar{\S}$, and the norm $\left\Vert \cdot \right\Vert_\rho$ is the $l^2(\rho)$-norm.

To determine the expected error of the value estimates in the specified norm, the intermediate task is to bound the empirical error, which is the difference between the true solution and the estimate based on the data points used to determine the estimate. We derive a bound on $\Vert V - \widehat{V} \Vert_N$ as a function of the number of data points and the complexity of the function class under consideration.


\begin{lemma}
Let $v = \left( V(\bar{s}_t) \right)_{t \in [N]}$ and $\hat{v} = \left( \hat{V}(\bar{s}_t) \right)_{t \in [N]}$. Then, with probability at least $1 - \delta$,
\begin{equation}\label{Eq:empError_indequality}
    \left\Vert v - \hat{v} \right\Vert_N \leq \frac{1}{\sqrt{1 - \gamma^2}} \left\Vert v - \hat{\Pi}v \right\Vert_N + \frac{L}{(1 - \gamma)^2} \sqrt{\frac{d}{\nu_N}} \left( \sqrt{\frac{2 \log(2d / \delta)}{N}} + \frac{1}{N} \right),
\end{equation}
where $N$ is the number of samples used, $\hat{\Pi}$ is the projection onto $\F_N$, $d$ is the dimensionality of the linear function space considered, $L$ is the upper bound of the basis functions, and $\nu_N$ is the smallest positive eigenvalue of the Gram matrix $\Phi^\transpose \Phi$.
\label{lemma:nonstationary_empirical_error_bound}
\end{lemma}

\begin{proof}
See Appendix~\ref{section:proofs:subsection:sample_complexity:subsubsection:empirical_error}.
\end{proof}

The inequality \eqref{Eq:empError_indequality} is almost identical to that in \citet{Lazaric:2012:FiniteSampleAnalysisOfLSPI} because we are still operating in the range space of the basis functions restricted to a finite set of states; therefore, the structure of the MDP in our setting does not affect the analysis. Because the proof follows a similar path, we have included in the Appendix the steps that differ from the original proof, which are mainly due to the noise terms taken in the martingale difference sequence. In our setting, these noise terms also include stochasticity due to the reward and the transition function, leading to a difference in the final expression.

While this final inequality is quite similar, we shall see that when considering the expected error in the augmented state space, we can highlight the properties of our setting by splitting the expected error into the error due to function truncation and the inherent error due to the linear function approximation.


The above inequality bounds the average error of the estimated value function, which is evaluated only at the points sampled from the environment. It is desirable to study the error in the approximate value function learned as an expectation over the states with respect to distribution $\rho$. Generally, $\rho$ is considered to be the stationary distribution of the Markov chain induced in the MDP by the policy $\pi$. To obtain such a result, additional assumptions are imposed on the Markov chain, such as the speed of convergence to its stationary distribution.

For linear function spaces, assuming that the policy induces a $\beta$-mixing Markov chain gives the following generalization bound for policy evaluation.

\begin{theorem}
    Assume that the policy $\pi$ induces on the MDP $\M_X$ a $\beta$-mixing Markov chain with parameters $\bar{\beta}$, $b$, $\kappa$ and stationary distribution $\rho$. Let $\bar{s}_1, \dots, \bar{s}_{N + \widetilde{N}}$ be a sample path obtained using the policy $\pi$. Suppose the first $\widetilde{N}$ samples are discarded for Markov chain burn-in, and the remaining $N$ samples are used to compute the truncated least-squares path-wise estimate $\widetilde{V}$ of the true value function $V$. Let $\nu$ be a lower bound on the eigenvalues of the sample-based Gram matrix that holds with a probability $1 - \delta/4$.
    Then, provided the number of discarded samples is $\tilde{N} = \left( \frac{1}{b} \log \frac{2e\bar{\beta}n}{\delta} \right)^{\frac{1}{\kappa}}$, with probability at least $1 - \delta$,

\begin{align*}
\left\Vert \widetilde{V} - V \right\Vert_\rho & \leq \frac{4 \sqrt{2}}{\sqrt{1 - \gamma^2}} \left( \frac{3}{(1 - \gamma^2)} \sum_{t = T+1}^\infty \left( M_t + N_t \right) + \left\Vert \Pi^{\mathrm{trunc}} V - \Pi V \right\Vert_\rho \right) \\*
				& \quad + \frac{2L}{(1 - \gamma)^2} \sqrt{\frac{d}{\nu}} \left( \sqrt{\frac{2 \log(8d / \delta)}{N}} + \frac{1}{N} \right) + \epsilon_1(N, \bar{\beta}, b, \kappa, d, \delta) + 2 \sqrt{2} \epsilon_2(N, \bar{\beta}, b, \kappa, \delta),
\end{align*}
\begin{align*}
    \text{where } \epsilon_1(N, d, \delta) & = \frac{24}{1 - \gamma} \sqrt{\frac{2 \Lambda_1(N, \bar{\beta}, d, \delta/4)}{N} \max \left\{ \frac{\Lambda_1(N, \bar{\beta}, d, \delta/4)}{b} , 1 \right\}^{\frac{1}{\kappa}}}, \\
    \epsilon_2(N, \delta) & = 12 \left( \frac{1}{1 - \gamma} + L \left\Vert \alpha^* \right\Vert \right) \sqrt{\frac{2 \Lambda_2(N, \bar{\beta}, \delta/4)}{N} \max \left\{ \frac{\Lambda_2(N, \bar{\beta}, \delta/4)}{b} , 1 \right\}^{\frac{1}{\kappa}}},
\end{align*}
    $V$ is the true value function, $\widetilde{V}$ is the learned value function clipped at $\frac{1}{1 - \gamma}$, $\Pi^{\mathrm{trunc}}V$ is the best possible value function on the truncated state space, $\Pi V$ is the best approximating value function in the linear function space considered, $M_t$ and $N_t$ are the upper bounds on the total variation due to exogenous events older than $t$ time steps induced in the transition dynamics and event mark distribution, respectively, $\Lambda_1(N, d, \delta)$ and $\Lambda_2(N, \delta)$ are as defined in Lemma~\ref{lemma:nonstationary_markov_chain_generalization} in Section~\ref{section:proofs:subsection:sample_complexity}, and $\alpha^*$ is such that $\Pi V = \sum_i \alpha^*_i \varphi_i$.
\label{thm:nonstationary_sample_complexity_evaluation}
\end{theorem}

This result breaks down the overall error into individual components, namely, the approximation error due to nonstationarity, the inherent error due to linear function approximation, the error due to the stochasticity of the samples that reduces with the number of samples, and other errors due to mixing of the Markov chain, the dimensionality of the function space, etc.

The effect of nonstationarity is reflected in the first term proportional to $\sum_{t} (M_t + N_t)$ for a general exogenous process whose events have an effect that decays as $M_t$ and $N_t$ on the state transition and next event distribution, respectively. For a specific temporal process, such as the discrete Hawkes process described in Corollary~\ref{corollary:nonstationary_policy_iteration_hawkes}, this term is replaced by its corresponding upper bound $\left( \frac{\bar{c}}{\lambda} e^{-\bar{\lambda}T} + \frac{c_\alpha}{\lambda_\alpha} e^{-\lambda_\alpha T} + \frac{c_\beta}{\sqrt{2 \pi} \lambda_\beta} e^{-\lambda_\beta T} \right)$, which depends on the parameters of the process.

The detailed background for the conditions on the MDP under which the results hold, the exact expressions for $\Lambda_1$ and $\Lambda_2$, and the supporting generalization Lemmas needed for the proof are provided in Section~\ref{section:proofs:subsection:sample_complexity}.

\paragraph{Proof of Theorem~\ref{thm:nonstationary_sample_complexity_evaluation}}
Following proof of Theorem $5$ in \citet{Lazaric:2012:FiniteSampleAnalysisOfLSPI}, 
	\begin{align}
		2 \left\Vert \hat{V} - V \right\Vert_N \geq 2 \left\Vert \widetilde{V} - V \right\Vert_N \geq \left\Vert \widetilde{V} - V \right\Vert_\rho - \epsilon_1, \label{ineq:sample_complexity_proof:000}
	\end{align}
	where the first inequality results from truncation, and the second inequality is a generalization Lemma for Markov chains stated in Section~\ref{section:proofs:subsection:sample_complexity}, and so
	\begin{align}
		\left\Vert \widetilde{V} - V \right\Vert_\rho & \leq 2 \left\Vert \hat{V} - V \right\Vert_N + \epsilon_1 \label{ineq:sample_complexity:proof:001}\\*
			& \leq \frac{2}{\sqrt{1 - \gamma^2}} \left\Vert V - \hat{\Pi} V \right\Vert_N + \frac{2L}{(1 - \gamma)^2} \sqrt{\frac{d}{\nu}} \left( \sqrt{\frac{2 \log(2d / \delta')}{N}} + \frac{1}{N} \right) + \epsilon_1 \label{ineq:sample_complexity:proof:002} \\
			& \leq \frac{2}{\sqrt{1 - \gamma^2}} \left\Vert V - \Pi V \right\Vert_N + \frac{2L}{(1 - \gamma)^2} \sqrt{\frac{d}{\nu}} \left( \sqrt{\frac{2 \log(2d / \delta')}{N}} + \frac{1}{N} \right) + \epsilon_1 \label{ineq:sample_complexity:proof:003} \\
			& \leq \frac{4 \sqrt{2}}{\sqrt{1 - \gamma^2}} \left\Vert V - \Pi V \right\Vert_\rho + \frac{2L}{(1 - \gamma)^2} \sqrt{\frac{d}{\nu}} \left( \sqrt{\frac{2 \log(2d / \delta')}{N}} + \frac{1}{N} \right) + \epsilon_1 + 2 \sqrt{2} \epsilon_2 \label{ineq:sample_complexity:proof:004} \\
			& \leq \frac{4 \sqrt{2}}{\sqrt{1 - \gamma^2}} \left( \left\Vert V - \Pi^{\mathrm{trunc}} V \right\Vert_\rho + \left\Vert \Pi^{\mathrm{trunc}} V - \Pi V \right\Vert_\rho \right) \nonumber \\ 
				& \qquad \qquad + \frac{2L}{(1 - \gamma)^2} \sqrt{\frac{d}{\nu}} \left( \sqrt{\frac{2 \log(2d / \delta')}{N}} + \frac{1}{N} \right) + \epsilon_1 + 2 \sqrt{2} \epsilon_2 \label{ineq:sample_complexity:proof:005} \\
			& \leq \frac{4 \sqrt{2}}{\sqrt{1 - \gamma^2}} \left( \left\Vert V -  V^\epsilon(s, x_{0:T}, \bar{0}) \right\Vert_\rho + \left\Vert \Pi^{\mathrm{trunc}} V - \Pi V \right\Vert_\rho \right) \nonumber \\
				& \qquad \qquad + \frac{2L}{(1 - \gamma)^2} \sqrt{\frac{d}{\nu}} \left( \sqrt{\frac{2 \log(2d / \delta')}{N}} + \frac{1}{N} \right) + \epsilon_1 + 2 \sqrt{2} \epsilon_2 \label{ineq:sample_complexity:proof:006} \\
			& \leq \frac{4 \sqrt{2}}{\sqrt{1 - \gamma^2}} \left( \left\Vert V -  V^\epsilon \right\Vert_\rho + \left\Vert V^\epsilon -  V^\epsilon(s, x_{0:T}, \bar{0}) \right\Vert_\rho + \left\Vert \Pi^{\mathrm{trunc}} V - \Pi V \right\Vert_\rho \right) \nonumber \\
				& \qquad \qquad + \frac{2L}{(1 - \gamma)^2} \sqrt{\frac{d}{\nu}} \left( \sqrt{\frac{2 \log(2d / \delta')}{N}} + \frac{1}{N} \right) + \epsilon_1 + 2 \sqrt{2} \epsilon_2 \label{ineq:sample_complexity:proof:007} \\
			& \leq \frac{4 \sqrt{2}}{\sqrt{1 - \gamma^2}} \left( \frac{3}{(1 - \gamma^2)} \sum_{t = T+1}^\infty \left( M_t + N_t \right) + \left\Vert \Pi^{\mathrm{trunc}} V - \Pi V \right\Vert_\rho \right) \nonumber \\
				& \qquad \qquad + \frac{2L}{(1 - \gamma)^2} \sqrt{\frac{d}{\nu}} \left( \sqrt{\frac{2 \log(2d / \delta')}{N}} + \frac{1}{N} \right) + \epsilon_1 + 2 \sqrt{2} \epsilon_2. \label{ineq:sample_complexity:proof:008}
	\end{align}

    Inequality~\eqref{ineq:sample_complexity:proof:002} follows from Lemma~\ref{lemma:nonstationary_empirical_error_bound}, \eqref{ineq:sample_complexity:proof:003} follows from the definition of $\hat{\Pi}$; and \eqref{ineq:sample_complexity:proof:004} is a generalization bound similar to \eqref{ineq:sample_complexity_proof:000} in the opposite direction for upper bounding the empirical error in terms of the expected error. \eqref{ineq:sample_complexity:proof:005} is just the triangle inequality.
 
    \eqref{ineq:sample_complexity:proof:006} is by definition of the $\Pi^{\mathrm{trunc}}$ operator, where $V^\epsilon$ is the `truncated' value function of an $\epsilon$-suboptimal policy that depends only on the past $T$ events. This is guaranteed by the proof of Theorem~\ref{thm:nonstationary_optimality} for
    $\epsilon =  \frac{2}{(1 - \gamma)^2} \sum_{t = T+1}^\infty (M_t + N_t)$.
    
    \eqref{ineq:sample_complexity:proof:007} is due to the triangle inequality by adding and subtracting the actual value function $V^\epsilon$ of the $\epsilon$-suboptimal policy. The final inequality~\eqref{ineq:sample_complexity:proof:008} is due to Lemma~\ref{lemma:nonstationary_state_cropping} and the value of $\epsilon$.
	
    Both these generalization bounds occur with probability at least $1 - \delta'$. Furthermore, with the lower bound $\nu$ on $\nu_N$ holding with probability at least $1 - \delta'$, the final bound holds with probability at least $1 - \delta = 1 - 4 \delta'$.
    
    \hfill $\blacksquare$

The generalization terms due to linear function approximation using Markov samples remain essentially the same as in \citet{Lazaric:2012:FiniteSampleAnalysisOfLSPI}. The main difference is the first extra term in the error, which is induced by truncating the augmented state space by discarding the features of old events. The amount of error introduced by this approximation depends on the extent to which past events influence current events and the state transition kernel.
    
    The results required for the proof can be taken from \citet{Lazaric:2012:FiniteSampleAnalysisOfLSPI} without any modifications because, even if some of the results, specifically the results on the bound on the covering numbers for linear function spaces taken from \citet{Gyorfi:2002:DistributionFreeTheoryOfNonparametricRegression}, require the domain of the basis functions to be finite-dimensional, the analysis takes place in the codomain and can be directly applied to our setting. We state the complete theorem along with the Lemmas in Section~\ref{section:proofs:subsection:sample_complexity}.

\subsection{Sample Complexity of Least Squares Policy Iteration}

To analyze Algorithm~\ref{alg:nonstationary_policy_iteration}, we need to consider the following standard regularity conditions on the stationary distribution of greedy policies, future-state distribution of arbitrary sequences of policies, and linear independence of the features.

\noindent
\textbf{(C1)} (Lower-bounding distribution) There exists a distribution $\nu$ and corresponding constant $C_\nu < \infty$ such that for any policy $\pi$ that is greedy with respect to a function in the truncated space  $\widetilde{\F}$, $\nu \leq C_\nu \rho^\pi$, where $\rho^\pi$ is the stationary distribution of policy $\pi$.

\noindent
\textbf{(C2)} (Discounted-average concentrability of future-state distribution~\citep{Antos:2008:LearningNearOptimalPoliciesBellmanResidualMinimization}) Given the target distribution $\rho$ and an arbitrary sequence of policies $\{\pi_m\}_{m\geq1}$, the concentrability coefficient
    \[
    c_{\rho, \nu}(m) = \sup_{\pi_1 \cdots \pi_m}\Bigg\|\frac{\d\nu P^{\pi_1}P^{\pi_2}...P^{\pi_m}}{\d\rho}\Bigg\|_{\infty},
    \]
   and the second-order discounted-average concentrability of future-state distributions, as defined below, is finite.
    \[
    C_{\rho,\nu} = (1-\gamma)^2\sum_{m\geq1}m\gamma^{m-1}c_{\rho,\nu}(m).
    \]

\noindent    
\textbf{(C3)} (Linearly independent features) Let $\nu$ be the lower-bounding distribution from Assumption \textbf{(C1)}. The features $\phi(\cdot)$ of the function space $\F$ are linearly independent w.r.t. $\nu$, and the smallest eigenvalue $\omega_\nu$ of the Gram matrix $G_\nu \in \R^{dxd}$ w.r.t. $\nu$ is strictly positive.

\noindent
\textbf{(C4)} (Slower $\beta$-mixing process) There exists a stationary $\beta$-mixing process with parameters $\bar{\beta}, b, \kappa,$ such that for any policy $\pi$ that is greedy w.r.t. a function in the truncated space $\tilde{\F}^{(T)}$, it is slower than the stationary $\beta$-mixing process with stationary distribution $\rho^\pi$ (with parameters $\bar{\beta_\pi},b_\pi,\kappa_\pi$). This means that $\bar{\beta}$ is larger, and $b$ and $\kappa$ are smaller than their counterparts $\bar{\beta_\pi},b_\pi$ and $\kappa_\pi$.

\begin{theorem}
Under Assumptions \textbf{(C1-C4)}, the value function of policy $\pi_K$ obtained using Algorithm~\ref{alg:nonstationary_least_squares_policy_iteration} for $K$ iterations satisfies, with probability at least $1 - \delta$, 
\begin{align*}
    \left\Vert V^* - V^{\pi_K} \right\Vert_{\rho} &\leq \frac{\sqrt{3\gamma}}{(1-\gamma)^2}\sqrt{C_{\rho, \nu}}\Bigg[(1+\gamma)\sqrt{2C_\nu} \Bigg( \frac{4 \sqrt{2}}{\sqrt{1 - \gamma^2}} \left( \frac{3}{(1 - \gamma^2)} \sum_{t = T+1}^\infty \left( M_t + N_t \right) + \cE_0(\F)\right) \\*
				& \quad + \frac{2L}{(1 - \gamma)^2} \sqrt{\frac{d}{\nu}} \left( \sqrt{\frac{2 \log(8d / \delta)}{N}} + \frac{1}{N} \right) + \cE_1 + 2 \sqrt{2} \cE_2 \Bigg) + \frac{2}{1-\gamma}\sum_{t=T+1}^\infty (M_t + N_t)\Bigg] \\*
    & \qquad \qquad \quad + \frac{\sqrt{6}\gamma^{K/2}}{(1-\gamma)^2},
\end{align*} 
where,
\begin{compactitem}
    \item[-] $\cE_0(\F) = \displaystyle \sup_{\pi \in \G\left(\tilde{\F}\right)}\|\Pi^{\mathrm{trunc}}V^\pi - \Pi V^\pi\|_{\rho^\pi}$
    \item[-] $\cE_1$ is $\epsilon_1$ from Theorem~\ref{thm:nonstationary_sample_complexity_evaluation} written for the slower $\beta$-mixing process defined in Assumption 4,
    \item[-] $\cE_2$ is $\epsilon_2$ from Theorem~\ref{thm:nonstationary_sample_complexity_evaluation} written for the slower $\beta$-mixing process defined in Assumption 4, and $\|\alpha^*\|$ replaced by $\sqrt{\frac{C}{\omega_\mu}}\frac{1}{1-\gamma}$, and 
    \item[-] $\nu_\mu$ is $\nu$ from Lemma~\ref{lemma:smallest_eigenvalue_bound_sample_gram_matrix} in which $\omega$ is replaced by $\omega_\nu$ defined in Assumption 3, and the second
term is written for the slower $\beta$-mixing process defined in Assumption 4.
\end{compactitem}

\label{thm:nonstationary_sample_complexity_iteration}
\end{theorem}

The error due to the finite history approximation of nonstationarity induced by external influence is captured by the first and second-to-last terms in the bound, that is, $\sum_{t=T+1}^\infty(M_t+N_t)$. For instance, when we consider the external process as a discrete-time Hawkes process (see Section~\ref{section:examples}), this additional error due to external influence can be written as
$\left( \frac{\bar{c}}{\lambda} e^{-\bar{\lambda}T} + \frac{c_\alpha}{\lambda_\alpha} e^{-\lambda_\alpha T} + \frac{c_\beta}{\sqrt{2 \pi} \lambda_\beta} e^{-\lambda_\beta T} \right)$,
 One can notice that this term decays exponentially with $T$. 

\begin{proof}{\textbf{of Theorem~\ref{thm:nonstationary_sample_complexity_iteration}}}
For the sequence $\tilde{V_k}$ of approximate value functions of the policies $\pi_k$ obtained by Algorithm~\ref{alg:nonstationary_least_squares_policy_iteration}, Lemma~\ref{lemma:munos_recursion} provides the following recursion for the suboptimality of the policies.
\begin{align*}
    V^* - V^{\pi_{k+1}} & \leq \gamma P^{\pi^*} (V^* - V^{\pi_k}) + \gamma E_k b_k + E_k' h_k, \\
    \text{where } E_k & = P^{\pi_{k+1}} \left( I - \gamma P^{\pi_{k+1}} \right)^{-1} - P^{\pi^*} \left( I - \gamma P^{\pi_k} \right)^{-1}, \\
    E_k' & = \gamma P^{\pi_{k+1}} \left( I - \gamma P^{\pi_{k+1}} \right)^{-1} + I, \\
    b_k &= \tilde{V}_k - T^{\pi_k} \tilde{V}_k, \\
    \text{and }h_k &= T^{\bar{\pi}_{k+1}} \tilde{V}_k - T^{\pi_{k+1}} \tilde{V}_k.
\end{align*}
Here, $P^\pi$ is an operator that provides the expected next-step value function when following policy $\pi$, and is defined as $P^\pi V (\bar{s}) = \E_{s' \sim Q(\bar{s}, \pi(s))} V(\bar{s}')$.

This bound on the difference between the optimal and current approximate value function has an extra term $E_k' h_k$ because our algorithm uses an approximate greedy policy that only depends on the current state and past $T$ events instead of the true greedy policy, which can be a function of an infinite history of events. This is reflected in $h_k$, which is the difference due to applying $T^{\bar{\pi}_{k+1}}$ on the previous value function $\tilde{V}_k$ using the true greedy policy $\bar{\pi}$, instead of $T^{\pi_{k+1}}$ with the approximate greedy policy $\pi_{k+1}$.

By induction and then taking the absolute value, we obtain
\begin{align*}
    \left\vert V^* - V^{\hat{\pi}_K} \right\vert & \leq \sum_{k=0}^{K-1}{(\gamma P^{\pi^*})^{K-k-1} (\gamma F_k |b_k| + F_k' |h_k|)} + (\gamma P^{\pi^*})^K \left\vert V^*-V^{\pi_0} \right\vert, \\
    \text{where } F_k & = P^{\pi_{k+1}} \left( I - \gamma P^{\pi_{k+1}} \right)^{-1} + P^{\pi^*} \left( I - \gamma P^{\pi_k} \right)^{-1} \\
    \text{and } F_k' & = \gamma P^{\pi_{k+1}} \left( I - \gamma P^{\pi_{k+1}} \right)^{-1} + I.
\end{align*}
Using the fact that $ V^* - V^{\pi_0} \leq \frac{2}{1-\gamma} R_{\text{max}}\mathbf{1}$, where $R_\text{max}$ is the maximum possible reward in any transition, one can rewrite the above bound as
\begin{align*}
    \left\vert V^* - V^{\pi_K} \right\vert \leq \frac{1+2\gamma + \gamma^K - 4\gamma^{K+1}}{(1-\gamma)^2} \left[ \sum_{k=0}^{K-1}{(\alpha_k A_k |b_k| + \beta_k B_k |h_k|)} + \alpha_K A_K R_{\mathrm{max}} \mathbf{1} \right],
\end{align*}
where $A_{k}$ and $B_{k}$ represent the operators defined as 
\begin{align*}
    A_k & = \begin{cases}
        \frac{1-\gamma}{2}(P^{\pi^*})^{K-k-1}F_k, & 0 \leq k < K \\
        (P^{\pi^*})^K & k = K.
    \end{cases} \\
    B_k & = (1-\gamma) (P^{\pi^*})^{K-k-1}F_k',
\end{align*}
and corresponding coefficients
\begin{align*}
    \alpha_k & = \begin{cases}
        \frac{2(1-\gamma) \gamma^{K-k}}{1+2\gamma + \gamma^K - 4\gamma^{K+1}}, & 0 \leq k < K, \\
        \frac{2(1-\gamma)\gamma^K}{1+2\gamma + \gamma^K - 4\gamma^{K+1}}, & k = K,
    \end{cases} \\
    \beta_k & = \frac{(1-\gamma) \gamma^{K-k-1}}{1+2\gamma + \gamma^K - 4\gamma^{K+1}}.
\end{align*}

The operators $A_k$ and $B_k$ are positive, that is, $A_k V \geq 0$ and $B_k V \geq 0$ whenever $V \geq 0$, and leave unity invariance, that is, $A_k\mathbf{1} = \mathbf{1}$ and $B_k\mathbf{1}=\mathbf{1}$,  and the corresponding coefficients sum to 1. Therefore, after raising both sides of the inequality to the power $p$ and integrating with respect to an arbitrary distribution $\rho$, we can use the Jensen inequality to obtain
\begin{align*}
\left\Vert V^* - V^{\pi_K} \right\Vert_{p,\rho}^p \leq \lambda_K .\; \rho\left[ \sum_{k=0}^{K-1}{(\alpha_k A_k |b_k|^p + \beta_k B_k |h_k|^p)} + \alpha_K A_K R_{\mathrm{max}}^p \mathbf{1} \right],
\end{align*}
where the $\rho \cdot$ operator gives the expectation with respect to $\rho$, i.e., $\rho V = \E_{\bar{s} \sim \rho} V(\bar{s})$, and $\lambda_K = \left( \frac{1+2\gamma+\gamma^K-4\gamma^{K+1}}{(1-\gamma)^2} \right)^p$.
\end{proof}
From definitions of coefficients $c_{\rho,\nu}(m)$, 
\begin{align*}
\rho A_k & \leq (1-\gamma)\sum_{m\geq0}{\gamma^m c_{\rho,\nu}(m+K-k)\nu}, \text{ and}\\
\rho B_k & \leq (1-\gamma) \sum_{m\geq0}{\gamma^m c_{\rho,\nu}(m+K-k-1)}\nu. 
\end{align*}
This gives us
\begin{align*}
\left\Vert V^* - V^{\pi_K} \right\Vert_{p,\rho}^p 
&\leq \lambda_K \Bigg[ 
    \sum_{k=0}^{K-1} \Bigg( 
        \alpha_k (1-\gamma) \sum_{m \geq 0} \gamma^m c_{\rho,\nu}(m + K - k) \|b_k\|_{p,\nu}^p \\
&\quad \quad \quad \quad + \beta_k (1-\gamma) \sum_{m \geq 0} \gamma^m c_{\rho,\nu}(m + K - k - 1) \|h_k\|_{p,\nu}^p 
    \Bigg) 
    + \alpha_K R_{\max}^p 
\Bigg].
\end{align*}

\begin{align*}
\text{Let}\:\:\:\epsilon_1 & \stackrel{\text{def}}{=} \max_{ 0 \leq k < K}\|b_k\|_{p, \nu} = \max_{0 \leq k < K} \left\Vert \tilde{V}_k - T^{\pi_k} \tilde{V}_k \right\Vert_{p,\nu} ,
\text{ and } \\ \epsilon_2 & \stackrel{\text{def}}{=} \frac{2R_{\mathrm{max}}}{1-\gamma}\sum_{t=T+1}^\infty (M_t + N_t).
\end{align*}
Here, $\epsilon_1$ is the maximum Bellman error for all value function approximations over all iterations, and $\epsilon_2$ is the approximation error term due to the use of a truncated and finite representation of the augmented state space, which consists of just the current state and events in the past $T$ time steps. Using Lemma \ref{lemma:bound_h_k}, $|h_k|$ can be bounded and, one can write,
\begin{align*}
    \left\Vert V^*-V^{\pi_K} \right\Vert_{p, \rho}^{p} & \leq \lambda_K \Bigg[\frac{2 \gamma}{1+2\gamma+\gamma^K-4 \gamma^{K+1}} C_{\rho,\nu} \epsilon_1^p + \frac{\gamma}{1+2\gamma+\gamma^K-4 \gamma^{K+1}} C_{\rho,\nu} \epsilon_2^p \\*
    & \qquad \qquad \qquad \qquad \qquad \qquad \qquad \qquad \qquad + \frac{2 (1 - \gamma) \gamma^K}{1+2\gamma+\gamma^K-4 \gamma^{K+1}} R_{\mathrm{max}}^p \Bigg] \\
    & \leq \Bigg(\frac{1}{(1-\gamma)^2}\Bigg)^p \Bigg[\gamma (1+2\gamma + \gamma^K - 4\gamma^{K+1})^{p-1}C_{\rho,\nu}(2\epsilon_1^p + \epsilon_2^p) \\
    & \qquad \qquad \qquad \qquad \qquad \quad + 2(1-\gamma)\gamma^K(1+2\gamma + \gamma^K - 4\gamma^{K+1})^{p-1}R_{\mathrm{max}}^p\Bigg].
\end{align*}
Noting that $1+2\gamma + \gamma^K - 4\gamma^{K+1} < 3$, and considering the $l_2$ norm by setting $p = 2$, we obtain
\begin{align*}
    \left\Vert V^* - V^{\pi_K} \right\Vert_{\rho} & \leq \frac{\sqrt{3\gamma}}{(1-\gamma)^2}\sqrt{C_{\rho, \nu}}(\sqrt{2}\epsilon_1 + \epsilon_2) + \frac{\sqrt{6}\gamma^{K/2}}{(1-\gamma)^2}R_{\mathrm{max}} \\
    & = \frac{\sqrt{3\gamma}}{(1-\gamma)^2}\sqrt{C_{\rho, \nu}}\Bigg(\sqrt{2}\max_{0\leq k < K}\|V_k - T^{\hat{\pi}_k} V_k\|_{\nu} + \frac{2R_{\mathrm{max}}}{1-\gamma}\sum_{t=T+1}^\infty (M_t + N_t)\Bigg) \\*
    & \qquad \qquad \:\: + \frac{\sqrt{6}\gamma^{K/2}}{(1-\gamma)^2}R_{\mathrm{max}}.
\end{align*}
Based on Assumption~\textbf{(C1)} of $\nu$ being the lower bounding distribution of all $\rho_k$ with corresponding constant $C_\nu$, we have
\begin{align*}
    \left\Vert V^* - V^{\pi_K} \right\Vert_{\rho} &\leq \frac{\sqrt{3\gamma}}{(1-\gamma)^2}\sqrt{C_{\rho, \nu}}\Bigg(\sqrt{2C_\nu}\max_{0\leq k < K}\|V_k - T^{\pi_k} V_k\|_{\rho_k} + \frac{2R_{\text{max}}}{1-\gamma}\sum_{t=T+1}^\infty (M_t + N_t)\Bigg) \\*
    & \qquad \qquad \quad + \frac{\sqrt{6}\gamma^{K/2}}{(1-\gamma)^2}R_{\text{max}}.
\end{align*}
By applying Lemma~9 of \citet{Lazaric:2012:FiniteSampleAnalysisOfLSPI} to the approximate greedy policy $\pi_k$, one can bound the Bellman error using the evaluation error as
\begin{align*}
    \left\Vert V^* - V^{\pi_K} \right\Vert_{\rho} &\leq \frac{\sqrt{3\gamma}}{(1-\gamma)^2}\sqrt{C_{\rho, \nu}}\Bigg((1+\gamma)\sqrt{2C_\nu}\max_{0\leq k < K}\|V_k - V^{\pi_k}\|_{\rho_k} + \frac{2R_{\text{max}}}{1-\gamma}\sum_{t=T+1}^\infty (M_t + N_t)\Bigg) \\
    & \qquad \qquad \quad + \frac{\sqrt{6}\gamma^{K/2}}{(1-\gamma)^2}R_{\text{max}}.
\end{align*}
The first term, which is the maximum evaluation error over $K$ iterations, can be bounded using Theorem~\ref{thm:nonstationary_sample_complexity_evaluation} to obtain the final result.

\section{Remarks on Change of Rewards under the influence of external events}
Here, we consider the scenario in which external events influence rewards directly, along with the dynamics of the environment. To extend our results to this case, we need the following assumption along with \textbf{A1-A3}.

\textbf{A4} Events that occurred in the distant past have a reduced influence on the current rewards. Specifically, there exists a convergent series $\sum_T M'_T$ of nonnegative real numbers, such that at any time t, for any event histories $H_t$ and $H'_t$, if the MDP is in state s and action a is taken, and if $H_t$ and $H'_t$ differ by only one event at time $t'$, then 
\[
\left\vert r_{H_t}(s,a,s') - r_{H'_t}(s,a,s') \right\vert < \left\Vert r \right\Vert_\infty M'_T, \quad \text{for all } s\in \S, a\in \A, s' \in \S,
\]
provided $t - t' \geq T$. Here $\Vert r \Vert_\infty$ is the supremum norm over all states, actions, next states, and all feasible event history. 

Assumption A4 complements A1--A3 by restricting the influence of events in the distant past on the reward function. Together, these assumptions allow us to consider a setting in which the external process affects both the transition dynamics and the rewards. We extend the policy iteration procedure to this setting by approximating the infinite event history using a finite history window of length $T$. The resulting approximate policy evaluation and policy improvement steps are summarized in Algorithm~\ref{alg:nonstationary_policy_iteration_analogous}. The main performance guarantee for this procedure is stated in Theorem~\ref{thm:nonstationary_policy_iteration_analogous}. The auxiliary lemmas and corollaries required for its analysis are deferred to Appendix~\ref{rew_change:lemmas_and_corollaries}, where they are presented before their respective proofs.

\begin{algorithm}
\begin{algorithmic}
	\State Start with arbitrary deterministic policy $\pi_0: \S \times \X^{T+1} \rightarrow \A$\;
	\For{each $k \in \{0, 1, \dots \}$ till termination} 
		\State $//$ \texttt{Approximate Policy Evaluation}
            \State \begin{equation*}
				\hat{V}_k \left( \left( s, x_{0:\infty} \right) \right) = \E_{x'_{T+1:\infty} \sim \mu_1} V^{\pi_k}\left( \left( s, x_{0:T}, x'_{T+1:\infty} \right) \right).
		\end{equation*}
  
		\State $//$ \texttt{Approximate Policy Improvement}
            \State \begin{equation*}
			\pi_{k+1} \left( (s, x_{0:\infty}) \right) = \argmax_{a \in A} \; \E_{\begin{subarray}{l}{s' \sim Q_{x_{0:T}, x'_{T+1:\infty}}(\cdot | s,a)} \\ {x' \sim Q^X_{x_{0:T}, x'_{T+1:\infty}}(\cdot)} \\\\ {x'_{T+1:\infty} \sim \mu_2}\end{subarray}} \bigg[ r(s, a, s', x_{0:T},x'_{T+1:\infty})  + \gamma \hat{V}_k \left( \left( s', x', x_{0:T-1} \right) \right) \bigg].
		\end{equation*}
	\EndFor
\end{algorithmic}
\caption{Policy Iteration (case where both reward and dynamics change)}
\label{alg:nonstationary_policy_iteration_analogous}
\end{algorithm}


\begin{theorem}
At iteration $k$ of the policy iteration procedure given in Algorithm~\ref{alg:nonstationary_policy_iteration_analogous}, for any augmented state $\bar{s} \in \bar{\S} = \S \times \displaystyle \prod_{t=0}^\infty \X$, at least one of the following holds:

\begin{align}
 V^{\pi_{k+1}}(\bar{s}) &\geq V^{\pi_k}(\bar{s}), \text{ or} \label{eqn:nonstationary_improvement_analogous} \\
	 \left\vert \T V^{\pi_k}(\bar{s}) - V^{\pi_k}(\bar{s}) \right\vert & < {\frac{4(2-\gamma)}{(1-\gamma)^3}}\sum_{t=T+1}^{\infty}(M_t+N_t) + \frac{2(1+\gamma^2)(2-\gamma)}{(1-\gamma)^3} \sum_{t=T+1}^{\infty}M'_t, \label{eqn:nonstationary_small_bellman_error_analogous}
\end{align}

where $V^\pi$ denotes the value function of policy $\pi$ in the MDP $\M_X$.

Observe that a weaker bound can be obtained, which retains the same exact form as Theorem~\ref{thm:nonstationary_policy_iteration}, with the term $(M_t + N_t)$ simply replaced by $(M_t + M'_t + N_t)$. 

\label{thm:nonstationary_policy_iteration_analogous}
\end{theorem}


\section{Experiments}
\label{section:experiments}
\subsection{Policy Evaluation}
We performed experiments to assess policy evaluation in a nonstationary variant of the classic \textsf{Pendulum-v1} environment within the OpenAI Gym control suite. The task involves applying torque to a pendulum to swing it and keep it upright. We utilized a fixed neural network policy that was partially pretrained using DDPG~\citep{Lillicrap:2016:DDPG}. Subsequently, we employed pathwise LSTD with linear function approximation to evaluate this policy, as described in Section~\ref{section:lspi}. The features considered are the standard cosine and sine of the angle and the angular velocity, with additional features being the angle itself, along with the squares of all these features, resulting in an
$8$-dimensional state.

To address nonstationarity, we consider the discrete-time Gaussian-marked Hawkes process outlined in Section~\ref{section:examples}. The intensity resulting from 
due to an event decays as $\alpha = e^{- \lambda_\alpha t }$ with $\lambda_\alpha = 1.0$, and the effect on the event mark decays as $\beta_t = 1/(1 + t^2)$. The events were added to the torque applied to the pendulum. Figure~\ref{fig:nonstationary_exps:lspi_pendulum} shows the expected error of the learned value function as a function of the number of samples, event horizon $T$, and rate $\lambda_\alpha$ of the decay of the Hawkes process.

\begin{figure}
\centering
    \includegraphics[width=\textwidth]{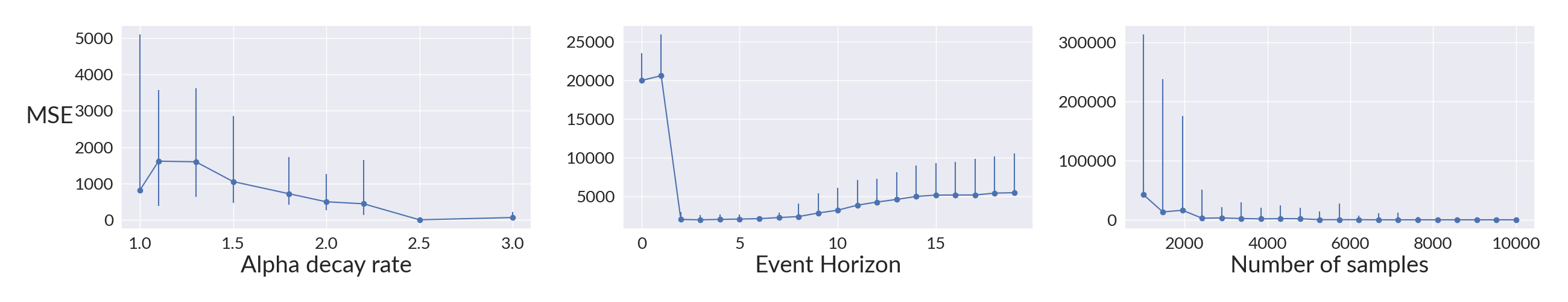}
    \caption{Performance of pathwise LSTD. The plots show the dependence of the Mean Squared Error of the value function learned using pathwise LSTD as a function of the rate of decay of the exogenous Hawkes process, the horizon of past events considered, and the number of samples used.}
    \label{fig:nonstationary_exps:lspi_pendulum}
\end{figure}

For a fixed event horizon $5$, the error decreases as the number of samples increases, which is intuitive. For a fixed number of samples $10,000$, the average error initially decreases as the event horizon increases, corresponding to the first term in the sample complexity of the total variation induced by events older than the event horizon. After a certain stage, increasing the event horizon results in a gradual increase in the error. This is because events older than the horizon now contribute a negligible influence on the current dynamics, while the dimensionality of the features increases owing to the inclusion of more past events in the features, resulting in a slightly greater expected error.

The first plot shows the dependence on $\lambda_\alpha$, the rate of decay of $\alpha_t$. Faster decay results in less nonstationarity and a lower approximation error owing to the first term in the upper bound of sample complexity.

We conducted $20$ trials and plotted the median of the mean squared errors for the experiments with respect to the decay rate and number of samples. The error bars indicate the range of $40$-$60$ percentile versus the decay rate and $20$-$80$ percentile for the other two.

These results empirically demonstrate the sample complexity bound for policy evaluation with linear function approximation in Theorem~\ref{thm:nonstationary_sample_complexity_evaluation}. The average error decreases with increasing $N$ and decreasing $N_t$.

Increasing the event horizon $T$ increases $d$ while reducing $\sum_{t=T+1}^\infty \{ M_t, N_t \}$, resulting in a trade-off between the two corresponding terms in the bound, with the minimal error being achieved for $T=2$.

\begin{figure}[t]
    \centering
    \raisebox{8.8mm}{\includegraphics[width=0.45\linewidth]{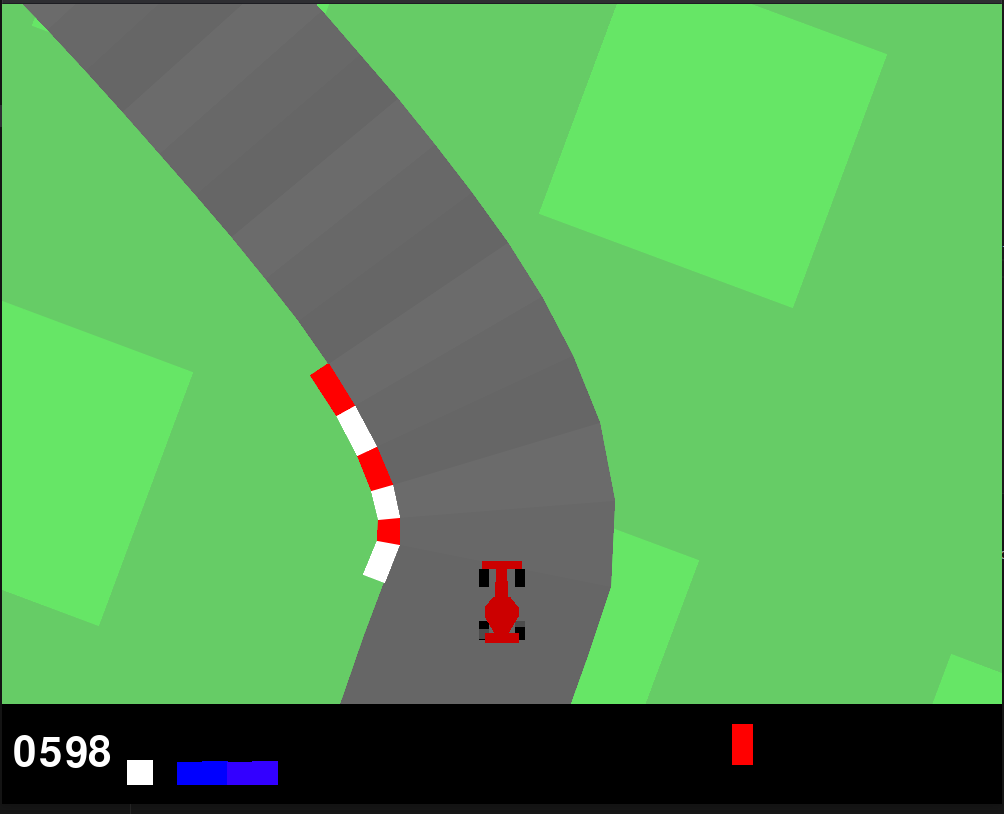}}
    \includegraphics[width=0.50\linewidth]{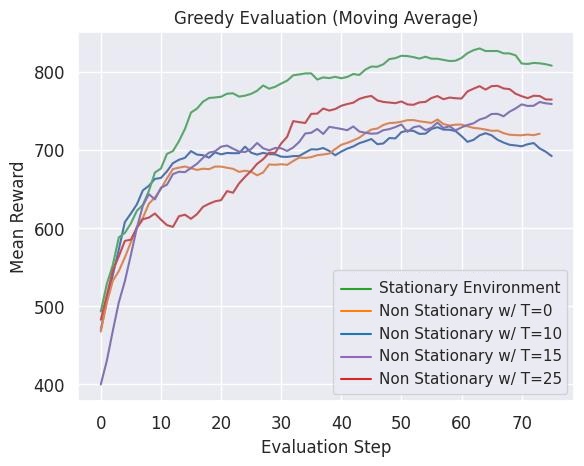}
    \caption{Left: Car racing environment in action. The friction between the road and car tires changes due to rain, which acts as an external event. The road friction will increase to the original value over time. Right: Policy improvement for nonstationary environment with different time horizons T.}
    \label{fig:nonstationary_exps:policy_iteration}
\end{figure}


We conducted experiments for policy improvement using the car racing environment of the Gymnasium library. We modified the environment to make it nonstationary. The external event in this case was rain, which reduced road friction. This reduction in road friction is temporary, and friction increases exponentially over time, returning to the original value as the road dries. The agent should learn to account for the reduced friction, especially in the corners, and reduce its speed. Otherwise, the agent would go off course. In a car racing scenario, where even a slight mistake can cost the race, the agent must drive at the right speed to lead while ensuring that the car does not skid. 

We trained the original environment with a maximum of $800$ steps. The agent was trained almost perfectly, completing the lap without going off course. This score serves as the upper bound for the nonstationary case. Next, we trained an agent for the nonstationary case, in which the agent had no information about external events. As expected, the agent's performance was not as good as that in the stationary case. Because the agent had no information about the external event, it learned to navigate the course at a lower speed to avoid going off course. We then experimented by providing the agent with information regarding the past $T$ events for $T=10$, $15$, and $25$. With the additional information of $T$, the agent can learn a better policy. As shown in Figure \ref{fig:nonstationary_exps:policy_iteration}, as we increase the horizon $T$, the agent can learn a better policy. However, the effect of passing information related to external events decreases as $T$ increases because the effect of older events diminishes exponentially.

\subsection{Policy Deployment}
Suppose a trained policy is available to the agent that is optimal in the absence of these external influences. We now wish to deploy this agent in a nonstationary environment where these external events affect the MDP dynamics. In this case, it would be more efficient to modify the existing policy to deal with nonstationarity directly than to train a new policy from the ground up.

\begin{figure}[t]
    \centering
    \raisebox{2.2mm}{\includegraphics[width=0.48\linewidth]{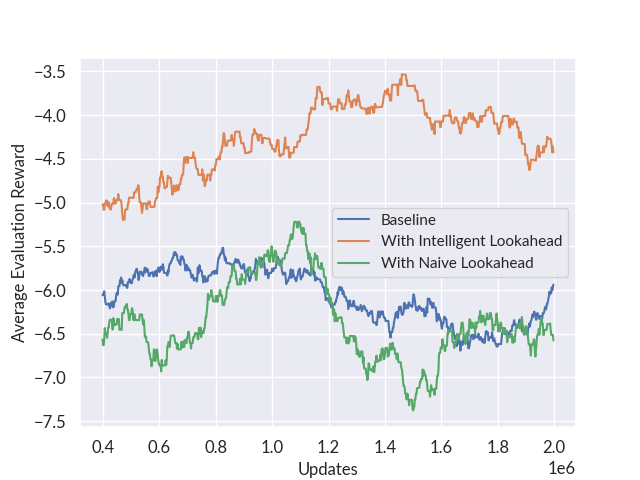}}
    \includegraphics[width=0.48\linewidth]{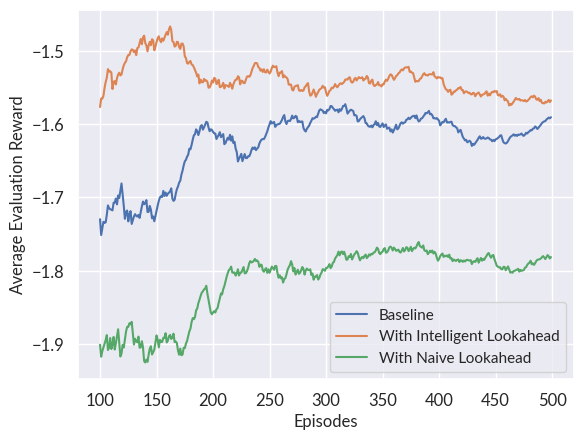}
    \caption{Results on nonstationary Pendulum and Point Maze environments. Learning state dynamics and exogenous events separately and planning intelligently outperforms a stationary policy or a model-based policy that tries to learn the dynamics model in the augmented state space directly.}
    \label{fig:nonstationary_exps:deployment}
\end{figure}

We consider a simple model-based planning strategy, wherein the agent learns the dynamics model and uses it to simulate a few trajectories based on the current policy. More precisely, given a pretrained policy $\pi$, when at an augmented state $\bar{s} \in \S$ during deployment, $d$ actions $a_1, \dots, a_d$ are sampled by perturbing the action $a = \pi(\bar{s})$. From each action $a_i$, a trajectory $(s, a_i, s_{i, 1}, a_{i,1}, \dots, s_{i,H}, a_{i,H})$ is obtained, and its corresponding sum of rewards $r_i = \sum_{t=1}^H r(s_{i,1}, a_{i,1})$ is calculated. The action $a_i$ that yields the highest reward $r_i$ is chosen and performed in the environment. This process was repeated at each time step during deployment.

This procedure can be performed in two different ways. A naive way to perform this is to learn a dynamics model $\hat{T}: \bar{\S} \times \A \rightarrow \Delta \bar{\S}$ directly in the augmented state space and use it to simulate the trajectories. However, a more intelligent method is to learn the event process separately and use it to simulate events $e_{i,t}$ that are fed into a dynamic model $\bar{T}: \bar{\S} \times \A \rightarrow \S$ that only predicts the states. Figure~\ref{fig:nonstationary_exps:deployment} depicts the results obtained using these two methods, along with those of a baseline that is the stationary policy learned in the stationary version of the environment. The first plot shows the results of the nonstationary pendulum task described previously. The second plot is for a nonstationary version of the Gymnasium Point Maze environment~\citep{Fu:2020:D4RL}, in which a 2DoF ball is force-actuated in the Cartesian $x$-$y$ directions to reach a randomly specified target goal in a closed, continuous 2D maze. Nonstationarity is induced by independent Hawkes processes, as described previously, that occur at $5$ randomly sampled and fixed points in the maze, each of which exerts a pulling force on the agent towards itself that is inversely proportional to the squared distance. From the plots, it is evident that an intelligent model-based look-ahead policy that mimics the structure of nonstationarity performs better than the naive baselines.

The algorithms theoretically analyzed in the previous sections are extensions of standard reinforcement learning algorithms that operate in the augmented state space. This is a valid strategy in many cases, as observed in the policy evaluation experiments described earlier. However, such a strategy ignores the structure of nonstationarity and sacrifices performance for generality. When additional knowledge about the nature of the exogenous event process is available, a valid question is whether using this knowledge helps solve the problem more effectively. This experiment showed that, for the case of deploying a stationary policy in a nonstationary environment using model-based planning, considering the nature of the event process and incorporating it intelligently with current algorithms will lead to better results than ignoring the structure of the nonstationarity and solely operating in the augmented state space.

This discrepancy is due to the nature of the neural network function approximation, which is primarily used in practical scenarios. The approximate policy improvement step in Algorithm~\ref{alg:nonstationary_policy_iteration} attempts to maximize the one-step returns obtained from each augmented state $\bar{s}$, which requires knowledge of the dynamics model. When this is unknown, an approximate model is used, which can be constructed in two ways, as described in the previous sections. Standard neural network dynamics models used in the literature are ill-equipped to model event processes, compelling one to model them separately from the original state $s \in \S$ when possible. Further research is needed to precisely model and incorporate event processes in reinforcement learning policies in a way that utilizes the structure of nonstationarity and the nature of the induced perturbations.

\subsection{Policy Improvement}

\begin{figure}[t]
    \centering
    \raisebox{2.2mm}{\includegraphics[width=1.00\linewidth]{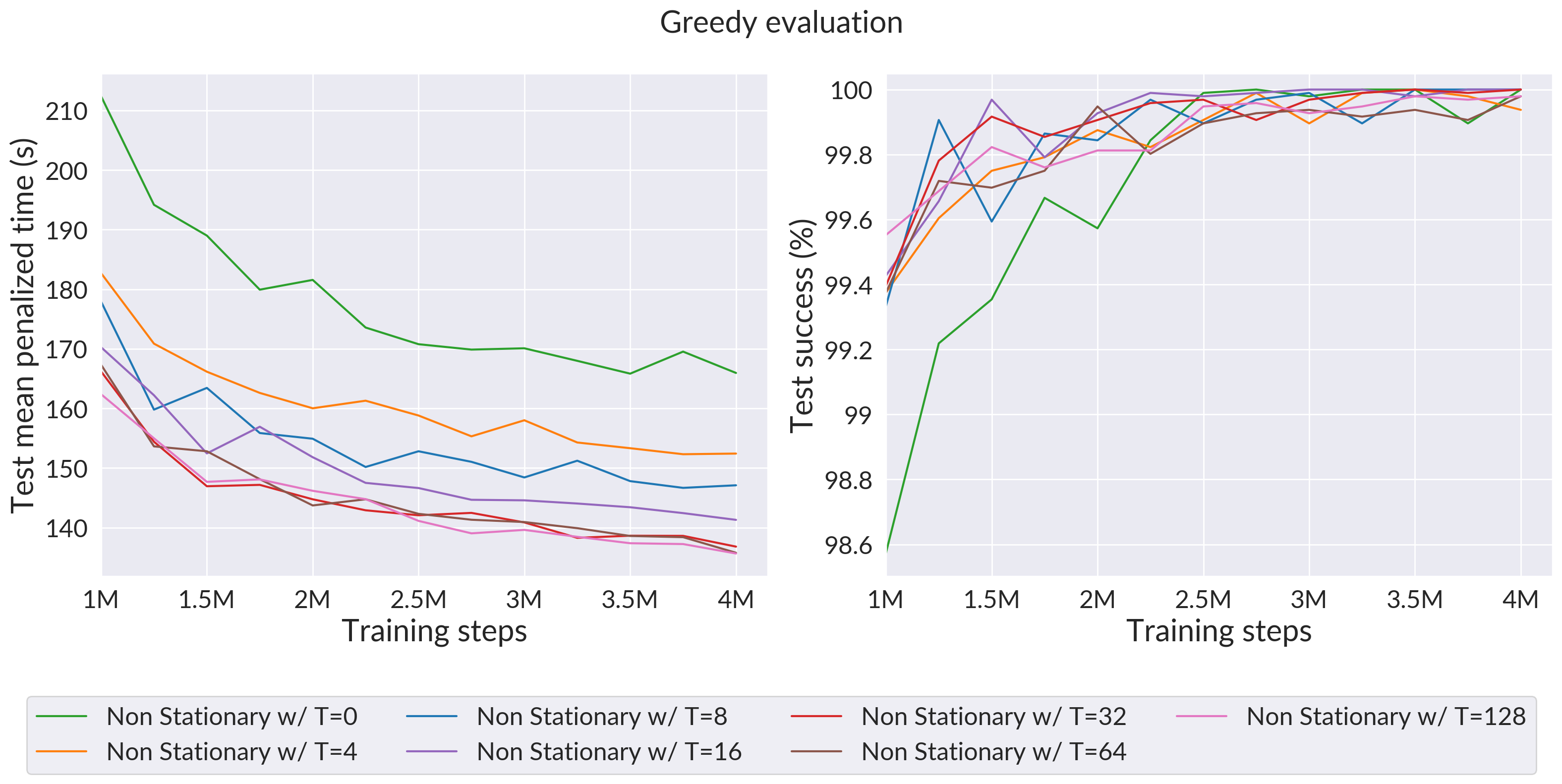}}
    \caption{Results on nonstationary Graph navigation environment, where exogenous process affects the transition dynamics. PPO policy is trained with states augmented with events from previous T time steps. The plots show mean penalized travel time and success rate for different history lengths. Increasing history length will reduce travel time, with diminishing gains.}
    \label{fig:nonstationary_exps:graph_dyn_change:improvement}
\end{figure}

\subsubsection{Graph Navigation Environment}
\label{policy_imp:graph_nav_env}
We designed a graph navigation environment based on Town03 of the CARLA simulator~\citep{Dosovitskiy17}. Each edge in the graph has a travel time associated with it, which is the cost. A start and goal node were randomly selected and assigned to the agent. The task was to reach the goal in the fastest route. An external process $\{X\}_{t \in \N}$ affects the transition dynamics of the environment by blocking a few edges. We used a discrete-time spatio-temporal Hawkes process as the external process. Whenever an agent attempts to enter a blocked edge, it must wait until that edge is unblocked. The per-step reward is defined as

\[
r_t=-\frac{\tau_t}{100}
+10\mathds{1}_{\{\text{goal reached at }t\}}
-10\mathds{1}_{\{\text{episode failed at }t\}},
\]
where \(\tau_t\) denotes the travel time incurred in the current step. Success and failure bonuses are applied only after the episode termination. This reward encourages the agent to reach its goal while minimizing travel time.

For evaluation, we report the mean penalized travel time. The penalized travel time of an episode is defined as the total travel time taken by the agent if it successfully reaches the goal, and as 2400 seconds, corresponding to the maximum allotted time budget, if it fails to reach the goal. This metric penalizes unsuccessful episodes while measuring the travel time of successful trajectories.

We compared the effects of a truncated history on policy improvement. We trained PPO\footnote{We used the stable-baselines3~\cite{stable-baselines3} implementation of PPO} with an extended state consisting of the current state and a sliding-window history of events for different window lengths $T$. Figure \ref{fig:nonstationary_exps:graph_dyn_change:improvement} shows that the mean-penalized travel time decreases with increasing history length (T). The improvement diminishes as T increases, and there is almost no gain from T = 64 to T = 128.

\subsubsection{Reward Change}

\begin{figure}[t]
    \centering
    \raisebox{2.2mm}{\includegraphics[width=1.00\linewidth]{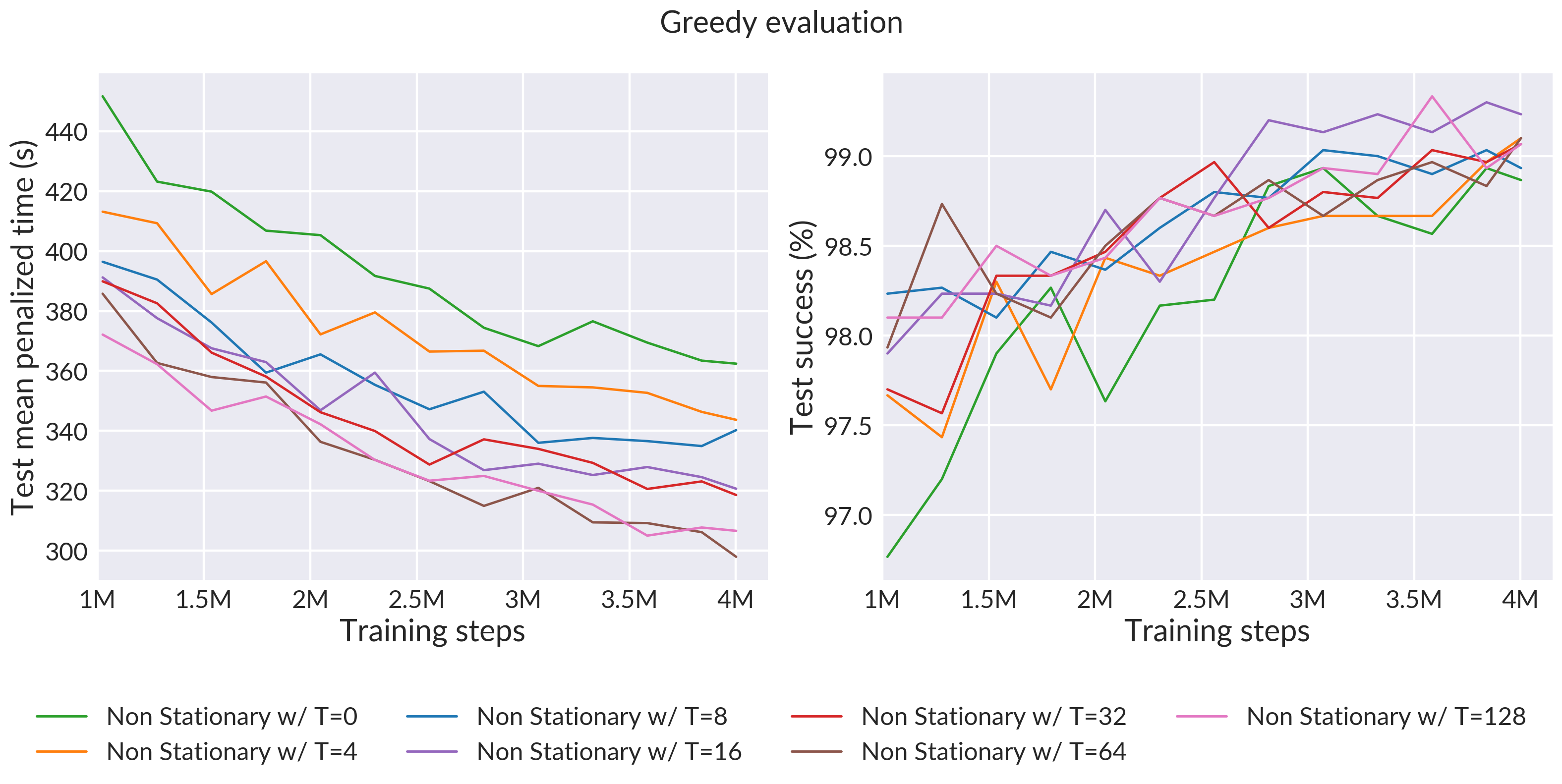}}
    \caption{ Results on nonstationary graph navigation environment, where exogenous process affects the rewards. PPO policy is trained with states augmented with events from previous T time steps. The plots show mean penalized travel time and success rate for different history lengths. Increasing history length will reduce travel time, with diminishing gains.}
    \label{fig:nonstationary_exps:graph_reward_change:improvement}
\end{figure}

In this section, we provide the empirical results for the case in which the exogenous process affects the rewards instead of the transition dynamics. We designed a graph navigation environment based on Town03 of the CARLA simulator. Each edge in the graph has a base travel time associated with it, which is the base cost. A start and goal node are randomly selected and given to the agent. The task is to reach the goal using the fastest route. An external process $\{X\}_{t \in \mathbb{N}}$ affects the rewards of the environment by increasing the travel times of the edges. We used a discrete-time spatio-temporal Hawkes process as the external process. We used the same per-step reward and penalized the travel time metric defined in Section~\ref{policy_imp:graph_nav_env}.

Figure \ref{fig:nonstationary_exps:graph_reward_change:improvement} shows that the mean penalized travel time decreases with increasing history length (T). The improvement diminishes as T increases, and there is almost no gain from T = 64 to T = 128. These trends are similar to those shown in Figure \ref{fig:nonstationary_exps:graph_dyn_change:improvement}, where the nonstationarity is due to the transition dynamics change. 

To illustrate how event history influences routing decisions, Figure~\ref{fig:nonstationary_exps:graph_reward_change:routes} compares the trajectories of PPO policies trained with $T=0$ and $T=64$ seconds in a selected episode. Both policies are evaluated with identical start and goal nodes, departure times, and external event schedules. Despite taking a longer route, the history-conditioned policy reaches the goal substantially earlier, illustrating how past event information can help the agent select routes with lower travel times under time-varying congestion.

\begin{figure}[t]
    \centering
    \includegraphics[width=\linewidth]{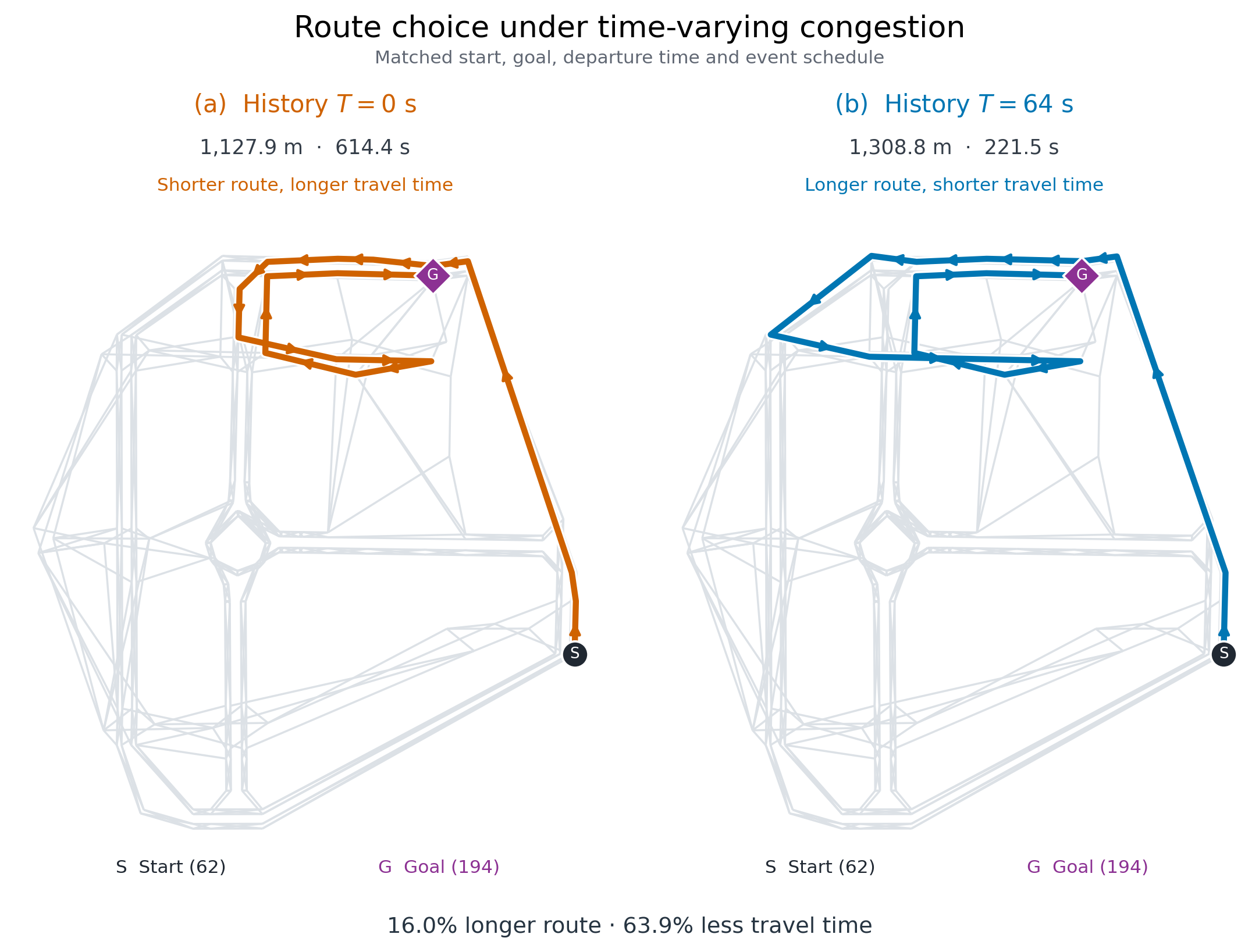}
    \caption{Illustrative route choices under time-varying congestion in the reward-change navigation environment. Trajectories executed by PPO policies with event-history windows of (a) $T=0$ s and (b) $T=64$ s, using identical start and goal nodes, departure time, and external event schedule. The history-conditioned policy travels 16.0\% farther (1,308.8 m versus 1,127.9 m) but reaches the goal in 63.9\% less time (221.5 s versus 614.4 s). This selected episode illustrates the benefit of incorporating historical external-event information when making routing decisions.}
    \label{fig:nonstationary_exps:graph_reward_change:routes}
\end{figure}

\subsubsection{Comparison with LILAC and DynaMITE}

\begin{figure}[ht!]
    \centering
    \subfloat[LILAC performance on the graph navigation environment where exogenous events affect the transition dynamics. PPO is trained with states augmented with events from previous T time steps.]{%
        \includegraphics[width=1.0\textwidth]{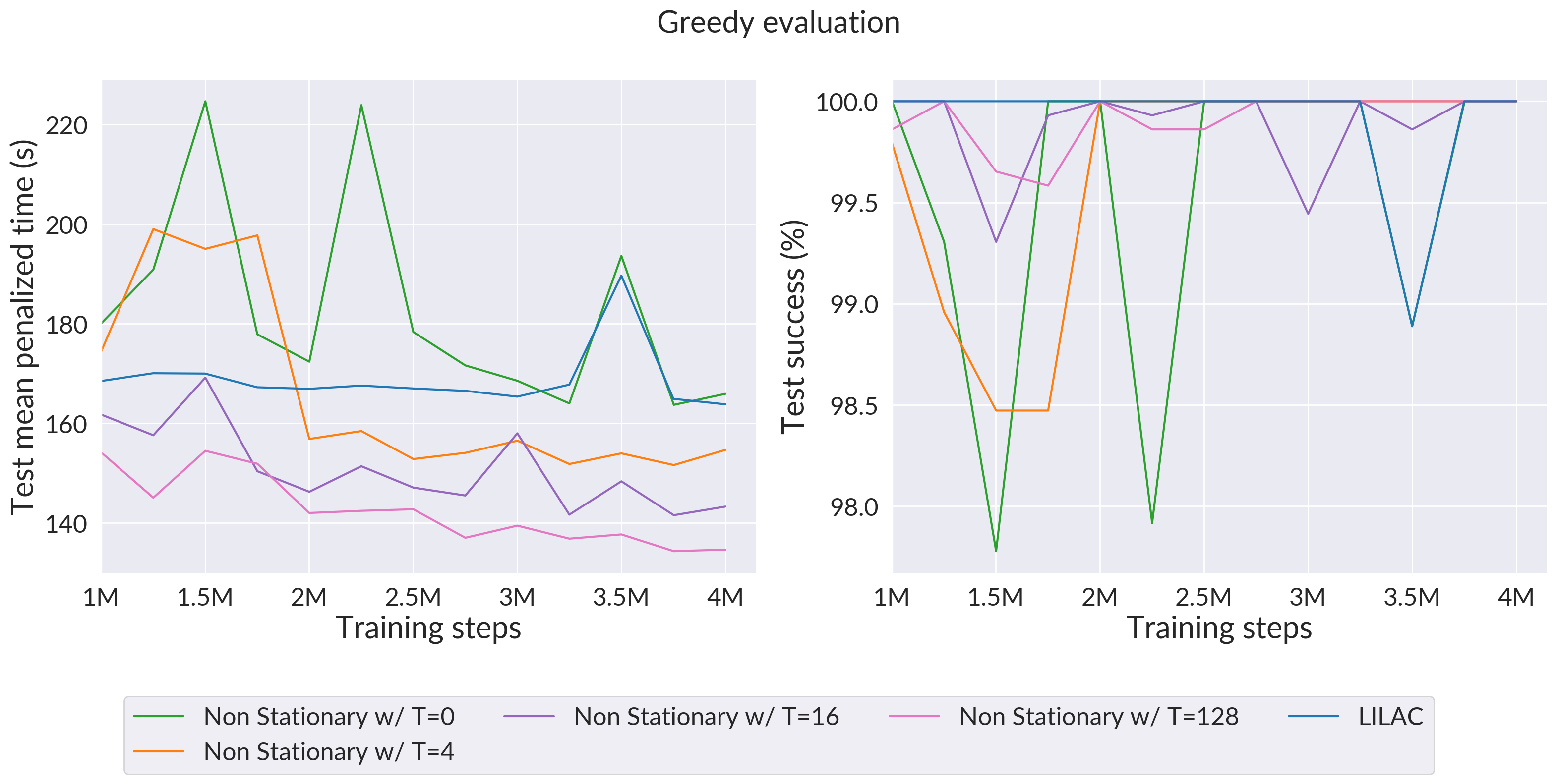}%
    }
    
    \subfloat[DynaMITE performance on graph navigation environment where exogenous events affect the transition dynamics. PPO is trained with states augmented with events from previous T time steps.]{%
        \includegraphics[width=1.0\textwidth]{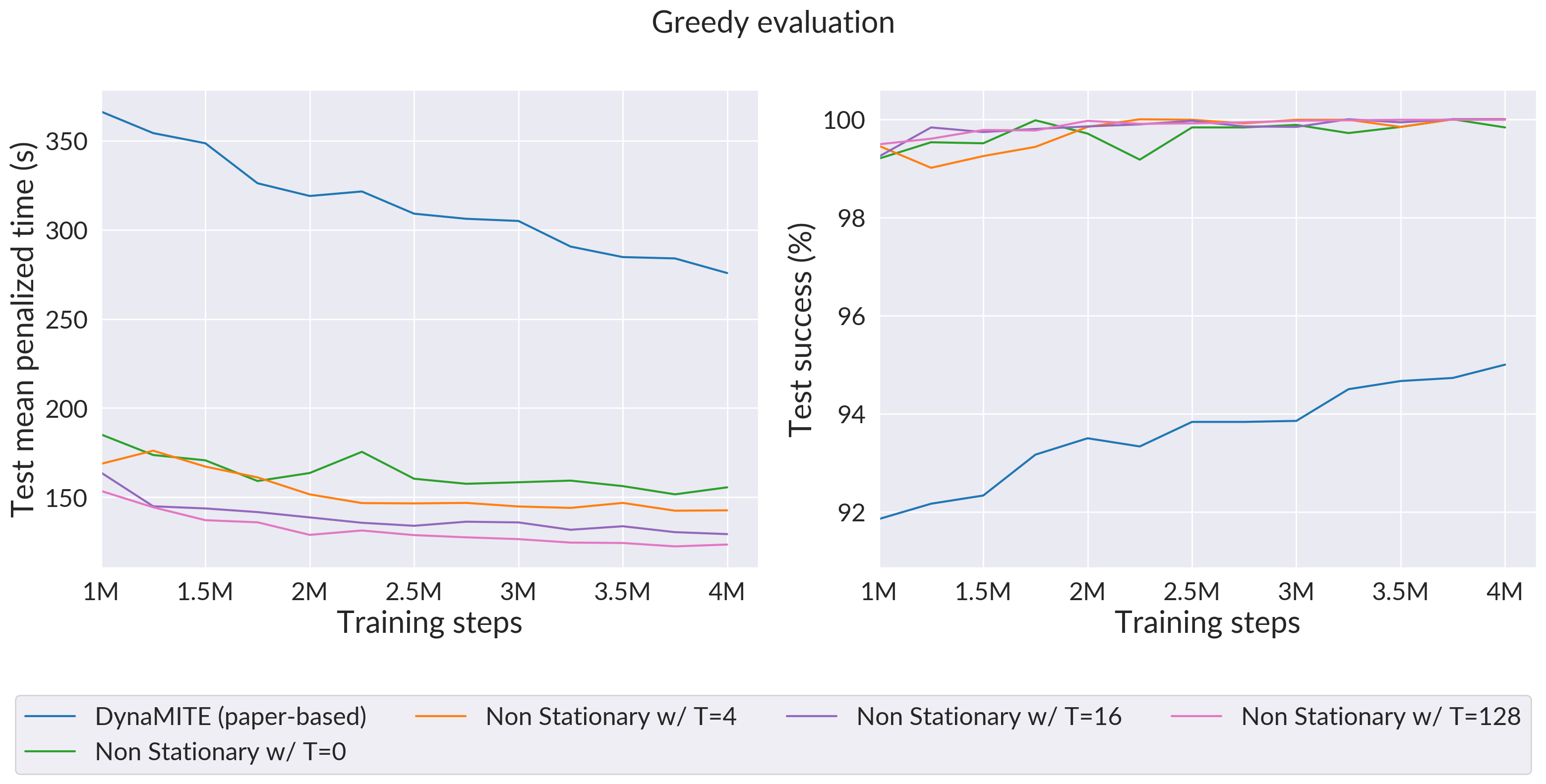}%
    }
    
    \caption{LILAC and DynaMITE performance in our framework}
    \label{fig:lilac_and_dyn_under_transition_dyn_change}
\end{figure}

We tested the LILAC~\citep{xie2020deepreinforcementlearningamidst} and DynaMITE~\citep{liang2024dynamiterldynamicmodelimproved} algorithms directly on a graph navigation environment described in the graph navigation environment subsection, where exogenous events affect the transition dynamics by temporarily blocking edges. All methods were evaluated using the penalized travel time metric defined in Section~\ref{policy_imp:graph_nav_env}.

We compared these methods against PPO (sb3) with an extended state consisting of the current state and a sliding-window history of events for different window lengths $T$. Figure \ref{fig:lilac_and_dyn_under_transition_dyn_change} shows the average penalized travel time over different training steps for each method. With some hyperparameter tuning using grid search, the performance of LILAC matched that of PPO with no history, indicating that LILAC's framework is unable to gain an advantage over vanilla PPO. DynaMITE's mean penalized travel time was worse than that of PPO without any additional event history. Both methods performed either similarly to vanilla PPO or worse, indicating that their methodology does not provide an advantage over vanilla PPO in our setting, where the transition dynamics change at every step owing to the exogenous process.

\section{Related Work}
\label{section:related_work}
Numerous prior studies addressing nonstationary RL problems predominantly offer practical solutions that lack theoretical backing. In this context, we refrain from discussing such works and instead aim to position our research within the framework of existing theoretical findings. 

\citet{Lecarpentier:2019:NonStationaryMDPsAWorstCaseApproachUsingModelBasedRL} considers a finite horizon MDP where the transition dynamics and reward function change in a manner that is Lipschitz continuous over time. \citet{Cheung:2020:RLForNonStationaryMDPsTheBlessingOfMoreOptimism} analyzes a similar scenario within a finite MDP, where the total variation in the environment, termed the variation budget, is known to the agent. This work proposes an optimistic value iteration-based algorithm that guarantees an upper bound on its dynamic regret.
\citet{Wei:2021:NonStationaryRLWithoutPriorKnowledgeOptimalBlackBoxApproach} introduces a black-box reduction that transforms any appropriate algorithm with optimal performance in a near-stationary environment into an algorithm that achieves optimal dynamic regret in a nonstationary environment.
\citet{Guo:2024:OnAverageOptimalityofMDPsinBorelSpaces} investigates average reward nonstationary MDPs in Borel spaces and proposes a rolling-horizon algorithm to obtain arbitrarily almost-optimal solutions based on a suitably changing horizon. At each stage, a finite horizon problem starting from that stage is solved to obtain the decision rule. In contrast, our algorithm considers a horizon of past external events instead of choosing the current action. Furthermore, along with the standard regularity conditions of compactness and continuity, Condition 1 in \citet{Guo:2024:OnAverageOptimalityofMDPsinBorelSpaces} contains global conditions on the sequence of transition kernels at each stage of the MDP to guarantee the existence of optimal policies and solutions to the Average Optimality Equation. In contrast, our assumptions (A2) and (A3) are more local in nature, with bounds on the perturbations of the transition kernel.

\citet{Hadoux:2014:SolvingHiddenSemiMarkovModeMDPs} introduced a class of piecewise stationary MDPs known as Hidden-Semi-Markov-Model MDPs, characterized by semi-Markov transitions between hidden stationary MDP models. To solve this problem, this work adapts the partially observable Monte Carlo planning algorithm used for solving partially observable MDPs.

\citet{Feng:2022:FactoredAdaptationForNonStationaryRL} propose factored nonstationary MDPs, where the transition dynamics are defined in terms of a dynamic Bayesian network over the states, actions, rewards, and latent change factors that induce nonstationarity. Here, both the states and the latent change factors have transition dynamics obeying this causal graph, and the latent factors evolve in a Markovian fashion. This work suggests employing variational autoencoders to learn the transition dynamics.  

A factorization setting called an exogenous MDP is presented in ~\citep{Efroni:2022:SampleEfficientRLExogenousInformation},  wherein the state is divided into two parts: an endogenous state, which is controllable by the agent's actions, and an exogenous state, which evolves on its own and does not affect the agent's reward. The precise manner in which this division exists is unknown to the agent. Although this may appear similar to our interpretation of states and external events, in our scenario, external events affect the dynamics of the state and, consequently, the rewards received by the agent. It is important to note that this work only considered the tabular setting and proposed algorithms that can only deal with finite horizon episodes.
A similar decomposition of the state space into endogenous and exogenous subspaces was studied by \citet{Dietterich:2018:DiscoveringAndRemovingExoStatesAndRewardsForRL}. In this study, it was posited that the exogenous states evolve in a Markov fashion, and that the reward function can be additively decomposed into exogenous and endogenous rewards. This study proposed algorithms to determine the exogenous component of the state as projections that maximize the partial correlation coefficient as a proxy for the mutual information between the exogenous next state and the endogenous current state and action conditioned on the exogenous current state.

While previous studies, such as \citet{Efroni:2022:SampleEfficientRLExogenousInformation,Dietterich:2018:DiscoveringAndRemovingExoStatesAndRewardsForRL}, have focused on the decomposition of states into endogenous and exogenous components, our work addresses a fundamentally different problem: the impact of integrating exogenous event information into reinforcement learning. 

Changing dynamics of environments based on externally specified ``contexts'' have also been studied under Contextual MDPs~\citep{Hallak:2015:ContextualMDPs}. However, these studies considered contexts that were constant for each episode and had a finite number of possible values. The aforementioned study proposed the CECE algorithm, which learns a set of policies corresponding to all contexts and determines the latent context for an episode (if unknown) to select the policy corresponding to the current context.

Specifically, the CECE clusters an initial set of trajectories into groups corresponding to a fixed set of latent contexts. For each new episode, a partial realization is used to classify it into one of these clusters, and a model learned from that cluster is used to exploit the rest of the trajectory. This algorithm cannot handle the setting considered in this study because each new time step results in a change in the context, and knowledge of a latent context for the current time step does not guarantee the same context for the rest of the episode, except in a trivial, special case.

\citet{Modi:2018:MDPsWithContinuousSideInformation} consider a relatively closer setting to ours, where the context is known and possibly continuous, with their algorithm being more amenable to being adapted to handle dynamic contexts. However, it works by keeping count of the number of occurrences of $(s,a)$ in each of a set of representative contexts to learn a set of explicit models, which is only possible in the setting of finite states and action spaces.

\citet{Tennenholtz:2023:RLwithHistoryDependentDynamicContexts} extend Contextual MDPs to handle dynamic contexts that are known to the agent, change at each time step, and influence the state transition distribution. However, their work differs from ours in several significant aspects. This study considers the contexts at each time step to depend on the history of states, actions, and contexts, whereas the next state depends only on the current state, action, and context. This means that causality exists in both directions, from states to contexts and vice versa. In contrast, in our work, the contexts are exogenous events that cannot be controlled by the agent, and the causality flows only in one direction, from contexts to states. Furthermore, their work considers finite state spaces, with the analysis explicitly geared towards such spaces and bounds that depend on the size of the state and action sets.

In \citet{Tennenholtz:2023:RLwithHistoryDependentDynamicContexts}, the context space is a fixed finite set of $M$ vectors, and a latent map is learned from the (state, action, context) space to $\mathbb{R}^M$, with the range of the map still being a finite, albeit exponentially increasing, set. This study also considers a finite-horizon MDP. This relatively simple setting allows one to tractably find the latent map and learn an optimal policy. In this sense, our setting is much more general and challenging.
In addition, \citet{Tennenholtz:2023:RLwithHistoryDependentDynamicContexts} presented a regret analysis, whereas our study focuses on the convergence results of algorithms.

The problem settings studied in ~\citet{xie2020deepreinforcementlearningamidst}  (LILAC) and ~\citet{liang2024dynamiterldynamicmodelimproved}  (DynaMITE) differ substantially from ours. Specifically, LILAC assumes an episodic framework in which a fixed, \textit{unknown} context governs each episode. Similarly, DynaMITE models tasks with a fixed, unknown context per ``session,'' where sessions transition randomly and intermittently according to a Markov process. Consequently, both approaches rely on variational autoencoder (VAE) architectures to infer latent context representations from all prior trajectories for policy training and inference.
In contrast, our work assumes that the context changes at every time step and is also known to the agent. Further, we assume that events in the past have a decaying effect on current dynamics, unlike in LILAC and DynaMITE, where only the latest/current context determines the current dynamics.

Our analysis addresses the nonstationarity induced in the MDP by the influence of non-Markovian events, necessitating the learning of policies that are contingent upon both the current state and a history of prior events. Although this approach bears resemblance to the strategies employed in Partially Observable MDPs (POMDPs)~\citep{Sunberg:2018:OnlineAlgosforPOMDPsWithContinuousStateActionObservationSpaces}, the underlying rationale diverges significantly.
In POMDPs, histories of observations are utilized despite the Markovian nature of the state transition and observation functions, primarily to mitigate an agent's inability to perceive the true state. Conversely, in our framework, both the state and external events are fully observable, and the necessity for an augmented space that incorporates event histories arises from the intrinsic non-Markovian characteristics of a nonstationary environment.
Moreover, the causal framework of our scenario is distinctly different. The events exhibit non-Markovian properties, and the state at any given time directly affects the distribution of the subsequent state, regardless of the knowledge of an event. In POMDPs, while beliefs and sufficient statistics are contingent upon a history of observations, the states themselves maintain a Markovian nature, leading to observations that are causally independent.

Our work makes significant inroads into understanding how the characteristics of exogenous event processes affect the tractability of the problem and the convergence and sample complexity of RL algorithms. To the best of our knowledge, no existing studies have reported these findings.

\section{Outlook}
This paper lays the foundation for understanding problems of sequential decision-making under the influence of external temporal events by providing several theoretical insights within the reinforcement learning framework. The results are established in a general setting that relies on learning functions in an augmented state space, which is the cross product of the original state space and event history. The challenge here is that the augmented state space can be exponentially larger due to the long-term dependencies of external events; hence, it is essential to develop provable approximate methods to mitigate this. 

The ``factorization" of the environment into Markovian states and non-Markovian events that impact these states, compressing the histories of events into latent features that are independent of the states, may provide a strategic advantage. This approach can facilitate the learning of representations of the expanded states by developing a distinct event distribution model that can be scaled up, as it does not rely on the agent, along with a separate state transition model that necessitates agent interaction with the environment. A compressed representation of events derived from extensive event data can be used to learn a feasible state transition kernel.
This can be achieved either implicitly through function approximators, such as recurrent neural networks (or LSTMs), or explicitly by learning compressed state representations using mutual information bottlenecks. Although these are practical methods, the theoretical analysis of such learned compressed state representations remains an open problem.

In this work, we used an augmented state representation obtained by considering the current state and a fixed time horizon $T$ of past events. The choice of $T$ is a hyperparameter that controls the extent to which information about the environment is preserved, with larger values of $T$ resulting in better approximation and, thereby, better performance of the corresponding algorithms, as described in Theorems~\ref{thm:nonstationary_policy_iteration}, \ref{thm:nonstationary_sample_complexity_evaluation} and \ref{thm:nonstationary_sample_complexity_iteration}. Although this is a fixed choice made by the algorithm designer, the next logical step would be to learn the optimal value of $T$ using the data collected from the environment and adapt it based on the performance of the agent. We leave this as a future research problem.

The analysis of sample complexity in
Section~\ref{section:lspi} hinges on generalizing the sample error of policy evaluation to an expected error over the entire state space. Such generalizations typically necessitate assumptions regarding the distribution of states. In our results, the assumption pertains to the rate of convergence of the Markov chain induced from the augmented MDP by the policy under evaluation, a notion frequently encountered in the reinforcement learning literature. However, in our case, the specific structure of the state space, specifically, the separation of the augmented state space into states and events with different characteristics, may necessitate alternate assumptions. For instance, while the recurrence and aperiodicity of the Markov chain may not hold, it may be reasonable to assume stationarity solely for the event sequences. 
Therefore, the generalizability of the value function in our setting can be analyzed in a more specific way that considers this exogenous structure.


The policy improvement step involves optimizing the action space in accordance with the expected one-step reward and value function for the subsequent time step, necessitating an understanding of the model’s structure. Therefore, when faced with an unknown transition and reward model, it is crucial to explore the sample complexity associated with learning these models from samples and the impact this has on the effectiveness of the resulting policy.

\textcolor{violet}{The code repository for all the experiments presented in this paper is available at \\\href{https://github.com/sml-iisc/RL_under_External_Influence}{\url{https://github.com/sml-iisc/RL_under_External_Influence}}}





\appendix

\section{List of notations and abbreviations}

Standard mathematical notation is used.

\begin{longtable}{p{2.5cm}p{12cm}}
    \toprule
    \textbf{Notation} & \textbf{Description} \\
    \midrule
    \midrule
    $\N$ & Set of natural numbers, $\{1, 2, 3, \dots\}$. \\
    $\R$ & Set of real numbers. \\
    $\abs{\mathcal{X}}$ & Cardinality of set $\mathcal{X}$. \\
    $\inf$ & Infimum of a given set \\
    $\sup$ & Supremum of a given set \\
    $\norm{\cdot}_2$ & $l_2$ norm of a vector. \\
    $[N]$ & The set $\{ 1, 2, \dots, N \}$ for some $N \in \N$. \\
    ~ & ~ \\
    $\E [ \cdot ]$ & Expectation with respect to a certain distribution. \\
    $\P( \cdot )$ & Probability of a given event with respect to some distribution. \\
    $\Delta ( \cdot )$ & Class of probability distributions over a given set \\
    $\cN(\mu, \sigma^2)$ & Normal distribution with mean $\mu$ and variance $\sigma^2$ \\
    \bottomrule
\end{longtable}

Notation related to reinforcement learning

\begin{longtable}{p{2.5cm}p{12cm}}
    \toprule
    \textbf{Notation} & \textbf{Description} \\
    \midrule
    \midrule
    $\M$ & A Markov Decision Process (MDP) \\
    $\S$ & State space of an MDP \\
    $\A$ & Action space of an MDP \\
    $r$ & Reward function \\
    $Q$ & Transition kernel of an MDP\\
    $\gamma$ & Discount factor \\
    $\pi$ & A policy that acts in an MDP \\
    $\pi^*$ & Optimal policy \\
    $V^\pi$ & Value function of policy $\pi$ \\
    $V^*$ & Optimal value function \\
    $Q^\pi$ &  Action-value function of policy $\pi$ \\
    $Q^*$ & Optimal Q function \\
    $\T$ & Bellman operator \\
    \bottomrule
\end{longtable}

Abbreviations.

\begin{longtable}{p{3.5cm}p{11cm}}
    \toprule
    \textbf{Notation} & \textbf{Description} \\
    \midrule
    \midrule
    MDP & Markov Decision Process \\
    PPO & Proximal Policy Optimization \\
    DOF & Degrees Of Freedom \\
    ~ & ~ \\
    \bottomrule
\end{longtable}


\section{Proof of inequality \ref{eqn:nonstationary_tv_example}}
\label{Appendix:TV-erf-calculation}
First we show that $\mathsf{TV} \left( \mathcal{N}(\mu_1, 1), \mathcal{N}(\mu_2, 1) \right) = \erf\left(\frac{\lvert\mu_2-\mu_1\rvert}{2\sqrt{2}}\right)$. 
Without loss of generality, we assume $\mu_2 \geq \mu_1$. Now,
\begin{align}
    \mathsf{TV} \left( \mathcal{N}(\mu_1, 1), \mathcal{N}(\mu_2, 1) \right) & = \frac{1}{2} \int_{-\infty}^{\infty}{\Bigg| \frac{1}{\sqrt{2\pi}} \exp\!\left( -\frac{(x-\mu_1)^2}{2} \right) - \frac{1}{\sqrt{2\pi}} \exp\!\left( -\frac{(x-\mu_2)^2}{2} \right) \Bigg|\d x} \nonumber \\
    & = \int_{\frac{\mu_1+\mu_2}{2}}^{\infty}{\frac{1}{\sqrt{2\pi}}\bigg( \exp\!\left( -\frac{(x-\mu_2)^2}{2} \right) - \frac{1}{\sqrt{2\pi}} \exp\!\left( -\frac{(x-\mu_1)^2}{2} \right) \bigg) \d x} \nonumber \\
    & = \frac{1}{\sqrt{2\pi}}\Bigg(\int_{\frac{\mu_1-\mu_2}{2\sqrt{2}}}^{\infty}{\sqrt{2}\exp(-y^2) \d y} - \int_{\frac{\mu_2-\mu_1}{2\sqrt{2}}}^{\infty}{\sqrt{2}\exp(-z^2) \d z}\Bigg) \nonumber \\
    & = \frac{2\sqrt{2}}{\sqrt{2\pi}} \int_{0}^{\frac{\mu_2-\mu_1}{2\sqrt{2}}}{\exp(-y^2) \d y} = \erf \bigg(\frac{|\mu_1-\mu_2|}{2\sqrt{2}}\bigg)
    \label{eqn:TV_dist_normal_dist}
\end{align}

Now, we show the steps involved in establishing \eqref{eqn:nonstationary_tv_example}. Given two historical sequences $(X_{t'})_{t' < t}$ and $(X'_{t'})_{t' < t}$ that only differ by one event at some time $t' \leq t - T$ with $B'_{t'} = X'_{t'} = 0$ and $B_{t'} = 1$, $X_{t'} = x$, the difference in intensities is $p_t - p'_t = \alpha_{t - t'}$

Let $q_1$ and $q_2$ be the density function of $X_t$ and $X'_t$
\begin{align}
    \mathsf{TV} \left( X_t, X'_t \right) & = \frac{1}{2} \int_{-\infty}^{\infty}{\lvert {q_1(x) - q_2(x)} \rvert \d x} \nonumber \\
    &= \frac{1}{2} \int_{-\infty}^{\infty} 
   \Big|\, q_1(x \mid E_t=0)\Pr(E_t=0) 
          + q_1(x \mid E_t=1)\Pr(E_t=1) \nonumber \\ 
    &\qquad\quad 
          - q_2(x \mid E'_t=0)\Pr(E'_t=0) 
          - q_2(x \mid E'_t=1)\Pr(E'_t=1) 
   \,\Big| \, dx \nonumber \\
   & \leq \frac{1}{2} \int_{-\infty}^{\infty}{\lvert q_1(x|E_t=0)(1-p_t) - q_2(x|E'_t=0)(1-p'_t) \rvert \d x} \nonumber \\
   & \quad \frac{1}{2} \int_{-\infty}^{\infty}{\lvert q_1(x|E_t=1)p_t - q_2(x|E'_t=1)p'_t \rvert} \d x \nonumber \\
   & \leq \frac{1}{2}\int_{-\infty}^{\infty}{\lvert \delta(x)(1-p_t) - \delta(x)(1-p't) \rvert \d x} \nonumber \\
   & \quad + \frac{p_t}{2} \int_{-\infty}^{\infty}{\lvert q_1(x|E_t=1) - q_2(x|E'_t=1) \rvert \d x} \nonumber \\
   & = \frac{\lvert p'_t - p_t \rvert}{2} \int_{-\infty}^{\infty}{\delta(x) \d x} + p_t \; \mathsf{TV} \left( \mathcal{N}(\mu_1, 1), \mathcal{N}(\mu_2, 1) \right) \nonumber \\
   & = \frac{\alpha_{t-t'}}{2} + p_t\;\mathsf{TV} \left( \mathcal{N}(\mu_1, 1), \mathcal{N}(\mu_2, 1) \right) \nonumber \\
   & = \frac{\alpha_{t-t'}}{2} + p_t \; \erf\bigg(\frac{|\mu_1-\mu_2|}{2\sqrt{2}}\bigg) = \frac{\alpha_{t-t'}}{2} + p_t \; \erf\bigg(\frac{|\beta_{t-t'}x|}{2\sqrt{2}}\bigg) \nonumber \\
   & \leq \frac{\alpha_T}{2} + \erf\bigg(\frac{\beta_Tb}{2\sqrt{2}}\bigg) = N_T. \nonumber
\end{align}

\section{Proof of Lemma~\ref{lemma:nonstationary_approx}}
\label{proof:nonstationary_lemma_approx}

The aim is to study the value function of the policy in two MDPs, $\M_X$ and $\M^{(T)}_X$, where $\M_X$ is the non-stationary MDP with augmented states consisting of the actual state and the entire history of events until then, and $\M^{(T)}_X$ is an approximate MDP in which only events in the past $T$ time steps affect the state transition and event distribution.
For notational convenience, one can extend the state space of $\M^{(T)}_X$ to that of $\M_X$ by including the information of every past event in the state, without changing the transition function. We define a class of auxiliary MDPs $\M^{(T)}_h$, $h \in \left\{ 0, 1, 2, \dots \right\}$ (similar to that of Theorem~1 from \citet{Xiao:2019:LearnToCombatCompoundingModelError}) such that when starting from the initial (augmented) state $\bar{s} \in \bar{S}$, the transitions follow the transition kernel of $\M_X$ for $h$ time steps and then the transition kernel of $\M^{(T)}_X$ for all transitions that follow. The value functions induced by $\pi$ in $\M_X$ and $\M^{(T)}_X$ can be related using these new intermediate MDPs, which interpolate between the two environments.
	
Since only a single policy is under discussion, the $\pi$ in $V^\pi$ is omitted for the rest of this proof for simplicity. $V$ denotes the value function of the policy in $\M_X$, $V^{(T)}$ denotes its value in $\M^{(T)}_X$, and $V_h^{(T)}$ denotes its value in $\M^{(T)}_h$.

From the definition of $\M^{(T)}_{h}$ and the boundedness of $r$, we have 
	$$
		V^{(T)} = V_0^{(T)}, \text{ and } V = \lim_{h \rightarrow \infty} V^{(T)}_h,
	$$
	and hence,
	\begin{align*}
		V - V^{(T)} = \lim_{H \rightarrow \infty} V_H^{(T)} - V_0 
			= \lim_{H \rightarrow \infty} \left( V_H^{(T)} - V_0 \right) 
			 & = \lim_{H \rightarrow \infty} \sum_{h=0}^{H-1} \left( V_{h+1}^{(T)} - V_h^{(T)} \right) \\*
			 & = \sum_{h=0}^{\infty} \left( V_{h+1}^{(T)} - V_h^{(T)} \right).
	\end{align*}
That is, the difference between $V$ and $V^{(T)}$ can be characterized in terms of the differences between $V_h^{(T)}$ and $V_{h+1}^{(T)}$ for each $h \in \{ 0, 1, 2, \dots \}$. 
 Now, for any augmented state $\bar{s} \in \bar{S}$,
	\begin{align*}
	V_h^{(T)}(\bar{s}) &= \sum_{t=0}^{h-1} \gamma^t \E_{\begin{subarray}{l} \bar{s}_t \sim p_t(. | \bar{s}) \\\\  a_t \sim \pi(. | \bar{s}_t) \\\\ \bar{s}_{t+1} \sim p(. | \bar{s}_t, a_t) \end{subarray}} \left[ r(\bar{s}_t, a_t, \bar{s}_{t+1}) \right] + \gamma^h \E_{\begin{subarray}{l} \bar{s}_h \sim p_h(. | \bar{s}) \\\\  a_h \sim \pi(. | \bar{s}_h) \\\\ \bar{s}_{h+1} \sim p^{(T)}(. | \bar{s}_h, a_h) \end{subarray}} \left[ r(\bar{s}_h, a_h, \bar{s}_{h+1}) \right] \\*
		& \qquad \qquad \qquad \qquad \qquad \qquad \qquad + \sum_{t=h+1}^{\infty} \gamma^t \E_{\begin{subarray}{l} \bar{s}_t \sim p_{t-h}^{(T)} \circ p_h(. | \bar{s}) \\\\  a_t \sim \pi(. | \bar{s}_t) \\\\ \bar{s}_{t+1} \sim p^{(T)}(. | \bar{s}_t, a_t) \end{subarray}} \left[ r(\bar{s}_t, a_t, \bar{s}_{t+1}) \right],
	\end{align*}
	where $p_t(. | \bar{s})$ denotes a transition of $t$ steps in the environment $\M_X$ starting from $\bar{s}$, a superscript $^{(T)}$ denotes the corresponding object in $\M^{(T)}_X$, and the policy $\pi$ is implicit whenever it is not explicitly mentioned in any expression. The reward function does not depend on the external events, so $r(\bar{s}_t, a_t, \bar{s}_{t+1})$ is just $r(s_t, a_t, s_{t+1})$, where $\bar{s} = (s, x)$ and $s$ is the actual state.

	This can also be written as
	$$
		V_h^{(T)}(\bar{s}) = \sum_{t=0}^{h-1} \gamma^t \E_{\begin{subarray}{l} \bar{s}_t \sim p_t(. | \bar{s}) \\\\  a_t \sim \pi(. | \bar{s}_t) \\\\ \bar{s}_{t+1} \sim p(. | \bar{s}_t, a_t) \end{subarray}} \left[ r(\bar{s}_t, a_t, \bar{s}_{t+1}) \right] + \gamma^h \E_{\substack{\bar{s}_h \sim p_h(. | \bar{s}) }} V^{(T)} (\bar{s}_h).
	$$
	Using these two representations for $V_h^{(T)}$ and $V_{h+1}^{(T)}$ respectively gives
\begin{align*}
	V_{h+1}^{(T)}(\bar{s}) &- V_h^{(T)}(\bar{s}) \\ 
    & = \gamma^h \E_{\begin{subarray}{l} \bar{s}_h \sim p_h(. | \bar{s}) \\\\  a_h \sim \pi(. | \bar{s}_h) \\\\ \bar{s}_{h+1} \sim p(. | \bar{s}_h, a_h) \end{subarray}} \left[ r(\bar{s}_h, a_h, \bar{s}_{h+1}) \right] 
	+ \gamma^{h+1} \E_{\substack{\bar{s}_{h+1} \sim p_{h+1}(. | \bar{s}) }} V^{(T)} (\bar{s}_{h+1}) \\
	& \qquad \qquad- \gamma^h \E_{\begin{subarray}{l} \bar{s}_h \sim p_h(. | \bar{s}) \\\\  a_h \sim \pi(. | \bar{s}_h) \\\\ \bar{s}_{t+1} \sim p^{(T)}(. | \bar{s}_h, a_h) \end{subarray}} \left[ r(\bar{s}_h, a_h, \bar{s}_{h+1}) \right] - \gamma^{h+1} \E_{\substack{\bar{s}_{h+1} \sim p^{(T)} \circ p_h(. | \bar{s}) }} V^{(T)}(\bar{s}_{h+1}) \\
	& = \gamma^h \E_{\begin{subarray}{l} \bar{s}_h \sim p_h(. | \bar{s}) \\\\  a_h \sim \pi(. | \bar{s}_h) \end{subarray}}  \bigg[ \E_{\bar{s}_{h+1} \sim p(. | \bar{s}_h, a_h)} r(\bar{s}_h, a_h, \bar{s}_{h+1}) - \E_{\bar{s}_{h+1} \sim p^{(T)}(. | \bar{s}_h, a_h)} r(\bar{s}_h, a_h, \bar{s}_{h+1}) \bigg] \\*
	& \qquad + \gamma^{h+1} \E_{\begin{subarray}{l} \bar{s}_h \sim p_h(. | \bar{s}) \\\\ a_h \sim \pi(. | \bar{s}_h, a_h) \end{subarray}} \bigg[ \E_{\bar{s}_{h+1} \sim p(. | \bar{s}_h, a_h)} V^{(T)}(\bar{s}_{h+1}) - \E_{\bar{s}_{h+1} \sim p^{(T)}(. | \bar{s}_h, a_h)} V^{(T)}(\bar{s}_{h+1}) \bigg].
	\end{align*}

Because the current state and current event marks are independent and depend only on the previous events, the transition kernel can be factored as two independent distributions over $S$ and $\X$.
	Hence, replacing the augmented state $\bar{s} \in \bar{S}$ with the explicit representation $(s, H)$, where $s \in S$ and $H \in \X^\infty$ gives, for any bounded measurable function $f$ on $S \times \X^\infty$,
	\begin{align}
	& \E_{(s_{t+1}, H_{t+1}) \sim p(. | s_{t}, H_{t}, a_t)} f(s_{t+1}, H_{t+1}) - \E_{(s_{t+1}, H_{t+1}) \sim p^{(T)}(. | s_{t}, H_{t}, a_t)} f(s_{t+1}, H_{t+1}) \nonumber \\* 
	& \qquad = \E_{\begin{subarray}{l} s_{t+1} \sim Q_{H_t}(. | s_{t}, a_t) \\\\ X_{t+1} \sim Q^X_{H_t}(.| s_{t}, a_t) \end{subarray}} f(s_{t+1}, X_{t+1}, H_t) - \E_{\begin{subarray}{l} s_{t+1} \sim Q^{(T)}_{H_t}(. | s_{t}, a_t) \\\\ X_{t+1} \sim Q^{X(T)}_{H_t}(.| s_{t}, a_t) \end{subarray}} f(s_{t+1}, X_{t+1}, H_t) \nonumber \\
	& \qquad = \E_{\begin{subarray}{l} s_{t+1} \sim Q_{H_t}(. | s_{t}, a_t) \\\\ X_{t+1} \sim Q^X_{H_t}(.) \end{subarray}} f(s_{t+1}, X_{t+1}, H_t) - \E_{\begin{subarray}{l} s_{t+1} \sim Q^{(T)}_{H_t}(. | s_{t}, a_t) \\\\ X_{t+1} \sim Q^X_{H_t}(.) \end{subarray}} f(s_{t+1}, X_{t+1}, H_t) \nonumber \\*
	& \qquad \qquad \qquad + \E_{\begin{subarray}{l} s_{t+1} \sim Q^{(T)}_{H_t}(. | s_{t}, a_t) \\\\ X_{t+1} \sim Q^X_{H_t}(.) \end{subarray}} f(s_{t+1}, X_{t+1}, H_t) - \E_{\begin{subarray}{l} s_{t+1} \sim Q^{(T)}_{H_t}(. | s_{t}, a_t) \\ X_{t+1} \sim Q^{X(T)}_{H_t}(.) \end{subarray}} f(s_{t+1}, X_{t+1}, H_t) \nonumber \\
	& \qquad \leq \E_{\substack{X_{t+1} \sim Q^X_{H_t}(.)}} \left[ \left\Vert f(., X_{t+1}, H_t) \right\Vert_{\infty} \mathsf{TV}\left( Q_{H_t}(. | s_t, a_t), Q^{(T)}_{H_t}(. | s_t, a_t) \right) \right] \nonumber \\*
	& \qquad \qquad \qquad + \E_{s_{t+1} \sim Q^{(T)}_{H_t}(. | s_t, a_t)} \left[ \left\Vert f(s_{t+1}, ., H_t) \right\Vert_{\infty} \mathsf{TV}\left( Q^X_{H_t}(.), Q^{X(T)}_{H_t}(.) \right) \right] \nonumber \\
	& \qquad \leq \left\Vert f \right\Vert_\infty \left( \sum_{t=T+1}^\infty M_t + \sum_{t=T+1}^\infty N_t \right).
    \label{eqn:function_different_histories:upper_bound}
	\end{align}
	The second-to-last inequality is a result of the Total Variation Distance being an integral probability metric~\citep{Sriperumbudur:2009:IPMsDivergencesBinaryClassification}. Using this expression in the previous equation gives
	\begin{align*}
	V_{h+1}^{(T)}(\bar{s}) - V_h^{(T)}(\bar{s}) & \leq \gamma^h \left\Vert r \right\Vert_\infty \left( \sum_{t=T+1}^\infty \left(M_t + N_t \right) \right) 
		 + \gamma^{h+1} \left\Vert V^{(T)} \right\Vert_\infty \left( \sum_{t=T+1}^\infty \left(M_t + N_t \right) \right) \\
		& \leq \gamma^h \left( \left\Vert r \right\Vert_\infty + \gamma \left\Vert V^{(T)} \right\Vert_\infty \right) \left( \sum_{t=T+1}^\infty \left(M_t + N_t \right) \right).
	\end{align*}
	This gives the final result,
	\begin{align*}
	\left\Vert V - V^{(T)} \right\Vert_\infty & \leq \sum_{h=0}^\infty \gamma^h \left( \left\Vert r \right\Vert_\infty + \gamma \left\Vert V^{(T)} \right\Vert_\infty \right) \left( \sum_{t=T+1}^\infty \left(M_t + N_t \right) \right) \\*
	& = \frac{1}{1 - \gamma} \left( \left\Vert r \right\Vert_\infty + \gamma \left\Vert V^{(T)} \right\Vert_\infty \right) \left( \sum_{t=T+1}^\infty \left(M_t + N_t \right) \right).
	\end{align*}

    \hfill$\blacksquare$

\subsection{Proof of Lemma~\ref{lemma:nonstationary_state_cropping}}
\label{proof:nonstationary_lemma_state_cropping}
Let $\pi$ be a policy that depends only on the current state and the past $T$ events.
\begin{align*}
	V^{\pi}((s, x_{0:\infty})) & = \E_{\begin{subarray}{l} s' \sim Q_{x_{0:\infty}}(. | s, a) \\\\ x' \sim Q_{x_{0:\infty}}(.) \\\\ a = \pi((s, x_{0:T})) \end{subarray}} \left[ r(s, a, s') + \gamma V^\pi\left( (s', x', x_{0:\infty}) \right) \right], \\
	\text{and } V^{\pi}((s, x_{0:T}, 0)) & = \E_{\begin{subarray}{l} s' \sim Q_{x_{0:T}}(. | s, a) \\\\ x' \sim Q_{x_{0:T}}(.) \\\\ a = \pi((s, x_{0:T})) \end{subarray}} \left[ r(s, a, s') + \gamma V^\pi \left( (s', x', x_{0:T}, 0) \right) \right].
\end{align*}
Taking the difference between these two expressions and adding and subtracting the additional term
\begin{displaymath}
\E_{\begin{subarray}{l} s' \sim Q_{x_{0:T}}(. | s, a) \\\\ x' \sim Q_{x_{0:T}}(.) \\\\ a = \pi((s, x_{0:T})) \end{subarray}} V^\pi((s', x', x_{0:\infty}))
\end{displaymath}
\begin{align*}
\text{gives } & \left| V^{\pi}((s, x_{0:\infty})) - V^{\pi}((s, x_{0:T}, 0)) \right| \\
& \qquad \qquad \qquad \leq \left\Vert r \right\Vert_\infty \left( \sum_{t=T+1}^\infty M_t \right) + \gamma \left\Vert V^\pi \right\Vert_{\infty} \left( \sum_{t=T+1}^\infty M_t + N_t \right) \\
	& \qquad \qquad \qquad \qquad \qquad + \gamma \E_{\begin{subarray}{l} s' \sim Q_{x_{0:T}}(. | s, a) \\\\ x' \sim Q_{x_{0:T}}(.) \\\\ a = \pi((s, x_{0:T})) \end{subarray}} \left\vert V^\pi((s', x', x_{0:\infty})) - V^\pi((s', x', x_{0:T}, 0)) \right\vert.
\end{align*}

The first term on the right-hand side is a result of modification of inequality \eqref{eqn:function_different_histories:upper_bound}, where the function is not dependent on the external process.
In the above inequality, the left-hand side is the difference between two value functions whose first $T$ events are the same and differ everywhere else, but the right-hand side has a discounted term in which the first $T+1$ terms coincide. Therefore, repeating the same procedure results in value function differences in states that increasingly coincide with each other.

Since each repetition adds a $\gamma$ factor to the term, which itself is bounded by $\frac{1}{1 - \gamma}$ due to the bounded rewards, the resulting series converges, giving us the following bound.
\begin{align*}
\sup_{\bar{s} = (s, x_{0:\infty})} \big\vert V^{\pi} ((s, x_{0:\infty})) & - V^{\pi}((s, x_{0:T}, 0)) \big\vert \\
& \leq \sum_{t = T+1}^\infty \gamma^{t - T - 1} \left[ \left\Vert r \right\Vert_\infty \left( \sum_{t'=t}^\infty M_{t'} \right) + \gamma \left\Vert V \right\Vert_{\infty} \left( \sum_{t'=t}^\infty M_{t'} + N_{t'} \right) \right] \\
& \leq \left( 1 + \frac{\gamma}{1 - \gamma} \right) \sum_{t=T+1}^\infty \sum_{t'=t}^\infty \gamma^{t-T-1} (M_{t'} + N_{t'}) \\
& \leq \frac{1}{1 - \gamma} \sum_{t'=T+1}^\infty \sum_{t=0}^{t'-T-1} \gamma^t (M_{t'} + N_{t'}) \\
& = \frac{1}{1 - \gamma} \sum_{t'=T+1}^\infty \frac{1 - \gamma^{t' - T}}{1 - \gamma} (M_{t'} + N_{t'}) \\
& \leq \frac{1}{(1 - \gamma)^2} \sum_{t = T+1}^\infty \left( M_t + N_t \right).
\end{align*}

\hfill$\blacksquare$

\subsection{Auxiliary Results for the Reward-Change Setting}
\label{rew_change:lemmas_and_corollaries}
We next establish auxiliary results required for the analysis of the policy iteration procedure under reward changes. These results extend the bounds developed for the transition-change setting to the case where external events influence both the transition dynamics and the reward function. In particular, we first bound the approximation error introduced by restricting the influence of past events to a finite history window and then derive a corresponding bound for policies that depend only on the current state and the most recent \(T\) events.

One can establish a result analogous to Lemma~\ref{lemma:nonstationary_approx} as follows. 

\begin{lemma}
    Consider a new auxiliary MDP $\M_X^{(T)}$ that has the same underlying stochastic process as the augmented MDP $\M_X$, with one difference: only events in the past $T$ time steps affect the current state transition, reward, and event distribution, with events before that having no effect and being effectively zero. Let $\pi$ be an arbitrary policy as a function of the augmented state $\bar{s} \in \bar{S}$. That is, $\pi$ can be either stochastic or deterministic and can depend on the current state $s \in \S$ and any number of past events in $\X$. Then,
	$$
		\left\Vert V^\pi_{\M^{(T)}_X} -  V^\pi_{\M_X} \right\Vert_{\infty} \leq
			\frac{1}{1 - \gamma} \left( \left\Vert r \right\Vert_\infty \sum_{t=T+1}^\infty \left( M_t + M'_t \right) + \gamma \left\Vert V^\pi_{\M^{(T)}_X} \right\Vert_\infty  \sum_{t=T+1}^\infty \left( M_t + N_t \right) \right) ,
	$$
    where $r$ is the reward function, and $V^\pi_\M$ denotes the value function of policy $\pi$ in MDP $\M$.

    One can observe that a weaker bound can be obtained, which retains the same form as Lemma~\ref{lemma:nonstationary_approx}, with the term $(M_t + N_t)$ replaced by $(M_t + M'_t + N_t)$. 
\label{lemma:nonstationary_approx_analogous}
\end{lemma}

Similarly, we can state the following, analogous to Lemma~\ref{lemma:nonstationary_state_cropping}. 

\begin{lemma}
Let $\pi$ be a policy that depends only on the current state and past $T$ events. Then,
\begin{align*}
	\sup_{\bar{s} = (s, x_{0:\infty})} \left\vert V^{\pi} \left( \bar{s} \right) - V^{\pi} \left( \left( s, x_{0:T}, (0)_{T+1: \infty} \right) \right) \right\vert < \frac{1}{(1 - \gamma)^2} \sum_{t = T+1}^\infty \left( M_t + M'_t + N_t \right).
\end{align*}
\label{lemma:nonstationary_state_cropping_analogous}
\end{lemma}

The following corollary is analogous to Corollary~\ref{corollary:nonstationary_corollary_state_cropping}.  

\begin{corollary}
For a given policy $\pi$ that is a function of events only in the past $T$ time steps, and for any two augmented states $\bar{s}_1$ and $\bar{s}_2$ that have the same state and differ in any number of events that occurred before the previous $T$ time steps,
\begin{equation*}
\left\vert V^{\pi} \left( \bar{s}_1 \right) - V^{\pi} \left( \bar{s}_2 \right) \right\vert \leq \frac{2}{(1 - \gamma)^2} \sum_{t = T+1}^\infty \left( M_t + M'_t + N_t \right).
\end{equation*}
\label{corollary:nonstationary_corollary_state_cropping_analogous}
\end{corollary}

Similarly, Lemma~\ref{lemma:nonstationary_policy_improvement} can be extended as follows. 

\begin{lemma}
\begin{equation*}
	V^{\pi_{k+1}} \geq V^{\pi_k} -  \frac{4}{(1-\gamma)^3}\sum_{t=T+1}^{\infty}(M_t+N_t) - \frac{2(1+\gamma^2)}{(1-\gamma)^3}\sum_{t=T+1}^\infty M'_t.
\end{equation*}
Observe that a looser bound can be obtained which retains the same exact form of Lemma~\ref{lemma:nonstationary_policy_improvement}, with the term $(M_t + N_t)$ simply replaced by $(M_t + M'_t + N_t)$. 
\label{lemma:nonstationary_policy_improvement_analogous}
\end{lemma}

\subsubsection{Proof of Lemma \ref{lemma:nonstationary_approx_analogous}}
\label{proof:nonstationary_lemma_approx_analogous}
\begin{proof}
From the definition of $V^\pi_{\M^{(T)}_X}$ and $V^\pi_{\M_X}$
    \begin{align}
        V^\pi_{\M^{(T)}_X}((s,x_{0:\infty})) = \E_{\begin{subarray}{l} s' \sim Q_{x_{0:T}}(\cdot \mid s,a) \\ x' \sim Q^X_{x_{0:T}}(\cdot) \\ a \sim  \pi_k(\bar{s}) \end{subarray}} \left[ r(s,a,s',x_{0:T}) + \gamma V^\pi_{\M^{(T)}_X} ((s',x',x_{0:\infty})) \right] \label{lemma_nonstationary_approx_analogous:proof:01} \\
        V^\pi_{\M_X}((s,x_{0:\infty})) = \E_{\begin{subarray}{l} s' \sim Q_{x_{0:\infty}}(\cdot \mid s,a) \\ x' \sim Q^X_{x_{0:\infty}}(\cdot) \\ a \sim  \pi_k(\bar{s}) \end{subarray}} \left[ r(s,a,s',x_{0:\infty}) + \gamma V^\pi_{\M_X} ((s',x',x_{0:\infty})) \right] 
    \end{align}
Subtracting both expressions, we get
    \begin{align}
        & V^\pi_{\M^{(T)}_X}((s,x_{0:\infty})) - V^\pi_{\M_X}((s,x_{0:\infty})) \nonumber \\
        & \qquad \qquad = \E_{a \sim \pi(\bar{s})} \Bigg[
        \E_{\begin{subarray}{l} s' \sim Q_{x_{0:T}} (\cdot \mid s,a) \\ x' \sim Q^X_{x_{0:T}}(\cdot)\end{subarray}}
        \left[ r(s,a,s',x_{0:T}) + \gamma V^\pi_{\M^{(T)}_X} ((s',x',x_{0:\infty})) \right] \nonumber \\
        & \qquad \qquad \qquad \qquad -
        \E_{\begin{subarray}{l} s' \sim Q_{x_{0:\infty}} (\cdot \mid s,a) \\ x' \sim Q^X_{x_{0:\infty}}(\cdot)\end{subarray}}
        \left[ r(s,a,s',x_{0:\infty}) + \gamma V^\pi_{\M_X} ((s',x',x_{0:\infty})) \right]
        \Bigg] \\
        & \qquad \qquad = \E_{a \sim \pi(\bar{s})} \Bigg[
        \E_{\begin{subarray}{l} s' \sim Q_{x_{0:T}} (\cdot \mid s,a) \end{subarray}} \left[ r(s,a,s',x_{0:T}) \right] - \E_{\begin{subarray}{l} s' \sim Q_{x_{0:\infty}} (\cdot \mid s,a) \end{subarray}} \left[ r(s,a,s',x_{0:\infty}) \right] \nonumber \\
        & \qquad \qquad \qquad \qquad + \E_{\begin{subarray}{l} s' \sim Q_{x_{0:T}} (\cdot \mid s,a) \\ x' \sim Q^X_{x_{0:T}}(\cdot)\end{subarray}}
        \left[ \gamma V^\pi_{\M^{(T)}_X} ((s',x',x_{0:\infty})) \right] \nonumber \\
        & \qquad \qquad \qquad \qquad -
        \E_{\begin{subarray}{l} s' \sim Q_{x_{0:\infty}} (\cdot \mid s,a) \\ x' \sim Q^X_{x_{0:\infty}}(\cdot)\end{subarray}}
        \left[ \gamma V^\pi_{\M_X} ((s',x',x_{0:\infty})) \right]
        \Bigg] \label{lemma_nonstationary_approx_analogous:proof:01}
    \end{align}
Bounding first term of the expression inside the expectation \ref{lemma_nonstationary_approx_analogous:proof:01} using the result \ref{eqn:function_different_histories:upper_bound} and assumption (A4), we get
    \begin{align}
        & \E_{\begin{subarray}{l} s' \sim Q_{x_{0:T}} (\cdot \mid s,a) \end{subarray}} \left[ r(s,a,s',x_{0:T}) \right] - \E_{\begin{subarray}{l} s' \sim Q_{x_{0:\infty}} (\cdot \mid s,a) \end{subarray}} \left[ r(s,a,s',x_{0:\infty}) \right] \nonumber \\
        & \qquad \qquad =  \E_{\begin{subarray}{l} s' \sim Q_{x_{0:T}} (\cdot \mid s,a) \end{subarray}} \left[ r(s,a,s',x_{0:T}) \right] - \E_{\begin{subarray}{l} s' \sim Q_{x_{0:\infty}} (\cdot \mid s,a) \end{subarray}} \left[ r(s,a,s',x_{0:T}) \right] \nonumber \\
        & \qquad \qquad \qquad + \E_{\begin{subarray}{l} s' \sim Q_{x_{0:\infty}} (\cdot \mid s,a) \end{subarray}} \left[ r(s,a,s',x_{0:T}) \right] - \E_{\begin{subarray}{l} s' \sim Q_{x_{0:\infty}} (\cdot \mid s,a) \end{subarray}} \left[ r(s,a,s',x_{0:\infty}) \right] \\
        & \qquad \qquad \leq \left\Vert r \right\Vert_\infty \sum_{t=T+1}^\infty M_t + \left\Vert r \right\Vert_\infty \sum_{t=T+1}^\infty M'_t \label{lemma_nonstationary_approx_analogous:proof:02}
    \end{align}
Similarly, bounding second term of the expression inside the expectation \ref{lemma_nonstationary_approx_analogous:proof:01} using the result \ref{eqn:function_different_histories:upper_bound}, we get
    \begin{align}
        & \E_{\begin{subarray}{l} s' \sim Q_{x_{0:T}} (\cdot \mid s,a) \\ x' \sim Q^X_{x_{0:T}}(\cdot)\end{subarray}}
        \left[ V^\pi_{\M^{(T)}_X} ((s',x',x_{0:\infty})) \right] -
        \E_{\begin{subarray}{l} s' \sim Q_{x_{0:\infty}} (\cdot \mid s,a) \\ x' \sim Q^X_{x_{0:\infty}}(\cdot)\end{subarray}}
        \left[ V^\pi_{\M_X} ((s',x',x_{0:\infty})) \right] \nonumber \\
        & \qquad = \E_{\begin{subarray}{l} s' \sim Q_{x_{0:T}} (\cdot \mid s,a) \\ x' \sim Q^X_{x_{0:T}}(\cdot)\end{subarray}}
        \left[ V^\pi_{\M^{(T)}_X} ((s',x',x_{0:\infty})) \right] - \E_{\begin{subarray}{l} s' \sim Q_{x_{0:\infty}} (\cdot \mid s,a) \\ x' \sim Q^X_{x_{0:\infty}}(\cdot)\end{subarray}}
        \left[ V^\pi_{\M^{(T)}_X} ((s',x',x_{0:\infty})) \right] \nonumber \\
        & \qquad \qquad + \E_{\begin{subarray}{l} s' \sim Q_{x_{0:\infty}} (\cdot \mid s,a) \\ x' \sim Q^X_{x_{0:\infty}}(\cdot)\end{subarray}}
        \left[ V^\pi_{\M^{(T)}_X} ((s',x',x_{0:\infty})) \right] -
        \E_{\begin{subarray}{l} s' \sim Q_{x_{0:\infty}} (\cdot \mid s,a) \\ x' \sim Q^X_{x_{0:\infty}}(\cdot)\end{subarray}}
        \left[ V^\pi_{\M_X} ((s',x',x_{0:\infty})) \right] \\
        & \qquad \leq \left\Vert V^\pi_{\M^{(T)}_X} \right\Vert_\infty \sum_{t=T+1}^\infty (M_t+N_t) \nonumber \\
        & \qquad \qquad \qquad + \E_{\begin{subarray}{l} s' \sim Q_{x_{0:\infty}} (\cdot \mid s,a) \\ x' \sim Q^X_{x_{0:\infty}}(\cdot)\end{subarray}}
        \left[ V^\pi_{\M^{(T)}_X} ((s',x',x_{0:\infty})) -  V^\pi_{\M_X} ((s',x',x_{0:\infty}))\right] \label{lemma_nonstationary_approx_analogous:proof:03}
    \end{align}
Substituting the bounds \ref{lemma_nonstationary_approx_analogous:proof:02} and \ref{lemma_nonstationary_approx_analogous:proof:03} in \ref{lemma_nonstationary_approx_analogous:proof:01}
    \begin{align}
        & V^\pi_{\M^{(T)}_X}((s,x_{0:\infty})) - V^\pi_{\M_X}((s,x_{0:\infty})) \nonumber \\
        & \qquad \leq \E_{a \sim \pi(\bar{s})} \Bigg[ \left\Vert r \right\Vert_\infty \sum_{t=T+1}^\infty (M_t + M'_t) + \gamma \left\Vert V^\pi_{\M^{(T)}_X} \right\Vert_\infty \sum_{t=T+1}^\infty (M_t+N_t) \nonumber \\
        & \qquad \qquad + \gamma \E_{\begin{subarray}{l} s' \sim Q_{x_{0:\infty}} (\cdot \mid s,a) \\ x' \sim Q^X_{x_{0:\infty}}(\cdot)\end{subarray}}
        \left[ V^\pi_{\M^{(T)}_X} ((s',x',x_{0:\infty})) -  V^\pi_{\M_X} ((s',x',x_{0:\infty}))\right] \Bigg]
    \end{align}
Taking supremum over the augmented state space, we get
    \begin{align}
        \sup_{\bar{s}=(s,x_{0:\infty})} \left\vert V^\pi_{\M^{(T)}_X}(\bar{s}) - V^\pi_{\M_X}(\bar{s}) \right\vert & \leq \left\Vert r \right\Vert_\infty \sum_{t=T+1}^\infty (M_t+M'_t) + \gamma \left\Vert V^\pi_{\M^{(T)}_X} \right\Vert_\infty \sum_{t=T+1}^\infty (M_t+N_t) \nonumber \\
        & \qquad + \gamma \sup_{\bar{s}=(s,x_{0:\infty})}\left\vert V^\pi_{\M^{(T)}_X}(\bar{s}) - V^\pi_{\M_X}(\bar{s}) \right\vert \\
        \sup_{\bar{s}=(s,x_{0:\infty})} \left\vert V^\pi_{\M^{(T)}_X}(\bar{s}) - V^\pi_{\M_X}(\bar{s}) \right\vert & \leq \frac{1}{1-\gamma} \left\Vert r \right\Vert_\infty \sum_{t=T+1}^\infty (M_t+M'_t) \nonumber \\
        & \qquad + \frac{\gamma}{1-\gamma} \left\Vert V^\pi_{\M^{(T)}_X} \right\Vert_\infty \sum_{t=T+1}^\infty (M_t+N_t)
    \end{align}

    \begin{align}
        \left\Vert V^\pi_{\M^{(T)}_X} - V^\pi_{\M_X} \right\Vert_\infty \leq \frac{1}{1-\gamma} \left\Vert r \right\Vert_\infty \sum_{t=T+1}^\infty (M_t+M'_t) + \frac{\gamma}{1-\gamma} \left\Vert V^\pi_{\M^{(T)}_X} \right\Vert_\infty \sum_{t=T+1}^\infty (M_t+N_t)
    \end{align}
\end{proof}

\subsubsection{Proof of Lemma \ref{lemma:nonstationary_state_cropping_analogous}}
\label{proof:nonstationary_lemma_state_cropping_analogous}
\begin{proof}
    Let $\pi$ be a policy that depends only on the current state and the past $T$ events.
\begin{align*}
	V^{\pi}((s, x_{0:\infty})) & = \E_{\begin{subarray}{l} s' \sim Q_{x_{0:\infty}}(. | s, a) \\\\ x' \sim Q^X_{x_{0:\infty}}(.) \\\\ a = \pi((s, x_{0:T})) \end{subarray}} \left[ r(s, a, s', x_{0:\infty}) + \gamma V^\pi\left( (s', x', x_{0:\infty}) \right) \right], \\
	\text{and } V^{\pi}((s, x_{0:T}, 0)) & = \E_{\begin{subarray}{l} s' \sim Q_{x_{0:T}}(. | s, a) \\\\ x' \sim Q^X_{x_{0:T}}(.) \\\\ a = \pi((s, x_{0:T})) \end{subarray}} \left[ r(s, a, s', x_{0:T}, 0) + \gamma V^\pi \left( (s', x', x_{0:T}, 0) \right) \right].
\end{align*}
Taking the difference between these two expressions and adding and subtracting the additional term
\begin{displaymath}
\E_{\begin{subarray}{l} s' \sim Q_{x_{0:T}}(. | s, a) \\\\ x' \sim Q_{x_{0:T}}(.) \\\\ a = \pi((s, x_{0:T})) \end{subarray}} \left[ r(s, a, s', x_{0:\infty}) + \gamma V^\pi((s', x', x_{0:\infty})) \right]
\end{displaymath}
\begin{align*}
\text{gives } & \left| V^{\pi}((s, x_{0:\infty})) - V^{\pi}((s, x_{0:T}, 0)) \right| \\
& \qquad \qquad \qquad \leq \left\Vert r \right\Vert_\infty \left( \sum_{t=T+1}^\infty M_t + M'_t \right) + \gamma \left\Vert V^\pi \right\Vert_{\infty} \left( \sum_{t=T+1}^\infty M_t + N_t \right) \\
	& \qquad \qquad \qquad \qquad \qquad + \gamma \E_{\begin{subarray}{l} s' \sim Q_{x_{0:T}}(. | s, a) \\\\ x' \sim Q_{x_{0:T}}(.) \\\\ a = \pi((s, x_{0:T})) \end{subarray}} \left\vert V^\pi((s', x', x_{0:\infty})) - V^\pi((s', x', x_{0:T}, 0)) \right\vert.
\end{align*}

In the above inequality, the left-hand side is the difference between two value functions whose first $T$ events are the same and differ everywhere else, but the right-hand side has a discounted term in which the first $T+1$ terms coincide. Therefore, repeating the same procedure results in value function differences in states that increasingly coincide with each other.

Since each repetition adds a $\gamma$ factor to the term, which itself is bounded by $\frac{1}{1 - \gamma}$ due to the rewards being bounded, the resulting series converges, giving us the following bound.
\begin{align*}
\sup_{\bar{s} = (s, x_{0:\infty})} \big\vert V^{\pi} ((s, x_{0:\infty})) & - V^{\pi}((s, x_{0:T}, 0)) \big\vert \\
& \leq \sum_{t = T+1}^\infty \gamma^{t - T - 1} \left[ \left\Vert r \right\Vert_\infty \left( \sum_{t'=t}^\infty M_{t'}+M'_{t'} \right) + \gamma \left\Vert V \right\Vert_{\infty} \left( \sum_{t'=t}^\infty M_{t'} + N_{t'} \right) \right] \\
& \leq \left( 1 + \frac{\gamma}{1 - \gamma} \right) \sum_{t=T+1}^\infty \sum_{t'=t}^\infty \gamma^{t-T-1} (M_{t'} + M'_{t'} + N_{t'}) \\
& \leq \frac{1}{1 - \gamma} \sum_{t'=T+1}^\infty \sum_{t=0}^{t'-T-1} \gamma^t (M_{t'} + M'_{t'} + N_{t'}) \\
& = \frac{1}{1 - \gamma} \sum_{t'=T+1}^\infty \frac{1 - \gamma^{t' - T}}{1 - \gamma} (M_{t'} + M'_{t'} + N_{t'}) \\
& \leq \frac{1}{(1 - \gamma)^2} \sum_{t = T+1}^\infty \left( M_t + M'_t + N_t \right).
\end{align*}
\end{proof}

\subsubsection{Proof of Theorem~\ref{thm:nonstationary_policy_iteration_analogous}}
\label{proof:non_stationary_policy_iteration_analogous}
Whether \eqref{eqn:nonstationary_improvement_analogous} or \eqref{eqn:nonstationary_small_bellman_error_analogous} holds
depends on the state $\bar{s}$ falls in the ``sub-optimality" set $B$ or not, which is the set of all states $\bar{s} = (s, x_{0:\infty}) \in \bar{\S}$ that satisfy
\begin{align*}
& \E_{\begin{subarray}{l} s' \sim Q_{x_{0:T}, x'_{T+1:\infty}}(. | s,a) \\\\ x' \sim Q^X_{x_{0:T}, x'_{T+1:\infty}}(.) \\\\ x'_{T+1:\infty} \sim \mu_2 \\\\ a \sim  \pi_k(\bar{s}) \end{subarray}} \left[ r(s,a,s',x_{0:T},x'_{T+1:\infty}) + \gamma \hat{V}_k((s', x', x_{0:T-1})) \right] \\
& \qquad \qquad \qquad < \max_{a \in \A} \E_{\begin{subarray}{l} s' \sim Q_{x_{0:T}, x'_{T+1:\infty}}(. | s,a) \\\\ x' \sim Q^X_{x_{0:T}, x'_{T+1:\infty}}(.) \\\\ x'_{T+1:\infty} \sim \mu_2 \end{subarray}} \left[ r(s,a,s',x_{0:T},x'_{T+1:\infty}) + \gamma \hat{V}_k((s', x', x_{0:T-1})) \right] - \delta
\end{align*}
for some suitable $\delta > 0$ that is yet to be chosen. It is to be noted that whether or not this condition holds, the above inequality or the corresponding equality holds for $\delta = 0$ by the definition of $\pi_{k+1}$.

By definition of $\hat{V}_k$ and using corollary \ref{corollary:nonstationary_corollary_state_cropping_analogous}, we have
\begin{align}
\hat{V}_k((s,x_{0:T})) & = \E_{\substack{x'_{T+1:\infty} \sim \mu_1}} \left[ V^{\pi_k}((s,x_{0:T},x'_{T+1:\infty})) \right] \label{ineq:policy_iteration_1_analogous:proof:01} \\
\vert V^{\pi_k}((s,x_{0:\infty})) - \hat{V}_k((s,x_{0:T})) \vert & = \vert \E_{\substack{x'_{T+1:\infty} \sim \mu_1}} \left[ V^{\pi_k}((s,x_{0:\infty})) - V^{\pi_k}((s,x_{0:T},x'_{T+1:\infty})) \right] \vert \nonumber \\
& \leq \E_{\substack{x'_{T+1:\infty} \sim \mu_1}} \left[ \frac{2}{(1-\gamma)^2}\sum_{t=T+1}^{\infty} (M_t + M'_t + N_t) \right] \label{ineq:policy_iteration_1_analogous:proof:02}
\end{align}

Using the Bellman equation, we have
\begin{align}
V^{\pi_k}((s,x_{0:\infty})) & = \E_{\begin{subarray}{l} s' \sim Q_{x_{0:\infty}}(.|s,a) \\ x' \sim Q^X_{x_{0:\infty}}(.) \\ a \sim \pi_k(\bar{s}) \end{subarray}} \left[ r(s,a,s',x_{0:\infty}) + \gamma V^{\pi_k}((s',x',x_{0:\infty})) \right] \nonumber \\
V^{\pi_k}((s,x_{0:\infty})) & \leq \E_{\begin{subarray}{l} s' \sim Q_{x_{0:\infty}}(.|s,a) \\\\ x' \sim Q^X_{x_{0:\infty}}(.) \\\\ a \sim \pi_k(\bar{s}) \end{subarray}} \left[ r(s,a,s',x_{0:\infty}) + \gamma \hat{V}_k((s',x',x_{0:T-1})) \right] \nonumber \\
& \qquad \qquad + \frac{2\gamma}{(1-\gamma)^2}\sum_{t=T+1}^{\infty} (M_t + M'_t + N_t) \label{ineq:policy_iteration_1_analogous:proof:03} \\
& \leq \E_{\begin{subarray}{l} s' \sim Q_{x_{0:T},x'_{T+1:\infty}}(.|s,a) \\\\ x' \sim Q^X_{x_{0:T},x'_{T+1:\infty}}(.) \\\\ x'_{T+1:\infty} \sim \mu_2 \\\\ a \sim \pi_k(\bar{s}) \end{subarray}} \left[ r(s,a,s',x_{0:\infty}) + \gamma \hat{V}_k((s',x',x_{0:T-1})) \right] \nonumber \\
& \qquad \qquad + \frac{2}{1-\gamma}\sum_{t=T+1}^{\infty}(M_t + N_t) + \frac{2\gamma}{(1-\gamma)^2}\sum_{t=T+1}^{\infty} (M_t + M'_t + N_t) \label{ineq:policy_iteration_1_analogous:proof:04} \\
& \leq \E_{\begin{subarray}{l} s' \sim Q_{x_{0:T},x'_{T+1:\infty}}(.|s,a) \\\\ x' \sim Q^X_{x_{0:T},x'_{T+1:\infty}}(.) \\\\ x'_{T+1:\infty} \sim \mu_2 \\\\ a \sim \pi_k(\bar{s}) \end{subarray}} \left[ r(s,a,s',x_{0:T},x'_{T+1:\infty}) + \gamma \hat{V}_k((s',x',x_{0:T-1})) \right] \nonumber \\
& \qquad + \sum_{t=T+1}^\infty M'_t + \frac{2}{1-\gamma}\sum_{t=T+1}^{\infty}(M_t + N_t) + \frac{2\gamma}{(1-\gamma)^2}\sum_{t=T+1}^{\infty} (M_t + M'_t + N_t) \label{ineq:policy_iteration_1_analogous:proof:05}
\end{align}

The inequality in \eqref{ineq:policy_iteration_1_analogous:proof:03} follows from \eqref{ineq:policy_iteration_1_analogous:proof:02}, and the inequality in \eqref{ineq:policy_iteration_1_analogous:proof:04} follows from \eqref{eqn:function_different_histories:upper_bound}. The inequality in \eqref{ineq:policy_iteration_1_analogous:proof:05} follows from Assumption \textbf{(A4)}. 

Suppose that $\bar{s} \in B$. Then, 

\begin{align}
V^{\pi_k}((s,x_{0:\infty})) & \leq \E_{\begin{subarray}{l} s' \sim Q_{x_{0:T},x'_{T+1:\infty}}(.|s,a) \\\\ x' \sim Q^X_{x_{0:T},x'_{T+1:\infty}}(.) \\\\ x'_{T+1:\infty} \sim \mu_2 \\\\ a \sim \pi_{k+1}(\bar{s}) \end{subarray}} \left[ r(s,a,s',x_{0:T},x'_{T+1:\infty}) + \gamma \hat{V}_k((s',x',x_{0:T-1})) \right] \nonumber \\
& \qquad \qquad + \frac{2}{(1-\gamma)^2} \sum_{t=T+1}^{\infty}(M_t+N_t) + \frac{1+\gamma^2}{(1-\gamma)^2}\sum_{t=T+1}^\infty M'_t - \delta \label{ineq:policy_iteration_1_analogous:proof:06}\\
& \leq \E_{\begin{subarray}{l} s' \sim Q_{x_{0:T},x'_{T+1:\infty}}(.|s,a) \\\\ x' \sim Q^X_{x_{0:T},x'_{T+1:\infty}}(.) \\\\ x'_{T+1:\infty} \sim \mu_2 \\\\ a \sim \pi_{k+1}(\bar{s}) \end{subarray}} \left[ r(s,a,s',x_{0:T},x'_{T+1:\infty}) + \gamma V^{\pi_k}((s',x',x_{0:\infty})) \right] \nonumber \\
& \qquad \qquad + \frac{2\gamma}{(1-\gamma)^2} \sum_{t=T+1}^{\infty}(M_t+M'_t+N_t) + \frac{2}{(1-\gamma)^2} \sum_{t=T+1}^{\infty}(M_t+N_t) \nonumber \\ 
& \qquad \qquad + \frac{1+\gamma^2}{(1-\gamma)^2}\sum_{t=T+1}^\infty M'_t - \delta \label{ineq:policy_iteration_1_analogous:proof:07}\\
& \leq \E_{\begin{subarray}{l} s' \sim Q_{x_{0:\infty}}(.|s,a) \\\\ x' \sim Q^X_{x_{0:\infty}}(.) \\\\ a \sim \pi_{k+1}(\bar{s}) \end{subarray}} \left[ r(s,a,s',x_{0:T},x'_{T+1:\infty}) + \gamma V^{\pi_k}((s',x',x_{0:\infty})) \right] \nonumber \\
& \qquad \qquad + \frac{2}{1-\gamma} \sum_{t=T+1}^{\infty}(M_t+N_t) + \frac{2+2\gamma}{(1-\gamma)^2} \sum_{t=T+1}^{\infty}(M_t+N_t) \nonumber \\
& \qquad \qquad+ \frac{(1+\gamma)^2}{(1-\gamma)^2}\sum_{t=T+1}^\infty M'_t - \delta \label{ineq:policy_iteration_1_analogous:proof:08}\\
& \leq \E_{\begin{subarray}{l} s' \sim Q_{x_{0:\infty}}(.|s,a) \\\\ x' \sim Q^X_{x_{0:\infty}}(.) \\\\ a \sim \pi_{k+1}(\bar{s}) \end{subarray}} \left[ r(s,a,s',x_{0:\infty}) + \gamma V^{\pi_k}((s',x',x_{0:\infty})) \right] \nonumber \\
& \qquad \qquad + \frac{2}{1-\gamma} \sum_{t=T+1}^{\infty}(M_t+N_t) + \frac{2+2\gamma}{(1-\gamma)^2} \sum_{t=T+1}^{\infty}(M_t+N_t) \nonumber \\
& \qquad \qquad+ \frac{(1+\gamma)^2}{(1-\gamma)^2}\sum_{t=T+1}^\infty M'_t +\sum_{t=T+1}^\infty M'_t - \delta \label{ineq:policy_iteration_1_analogous:proof:09}\\
& = \E_{\pi_{k+1}} \left[ r(s,a,s') + \gamma V^{\pi_k}(\bar{s}') \right] + \frac{4}{(1-\gamma)^2} \sum_{t=T+1}^{\infty}(M_t+N_t) \nonumber \\
& \qquad +\frac{2(1+\gamma^2)}{(1-\gamma)^2} \sum_{t=T+1}^\infty M'_t - \delta \label{ineq:policy_iteration_1_analogous:proof:10}\\
& \leq \E_{\pi_{k+1}}\left[ \sum_{t=0}^{\infty} \gamma^t r(s,a,s') \right] \nonumber \\
& \qquad + \frac{1}{1-\gamma}\left[ \frac{4}{(1-\gamma)^2}\sum_{t=T+1}^{\infty}(M_t+N_t) +\frac{2(1+\gamma^2)}{(1-\gamma)^2} \sum_{t=T+1}^\infty M'_t \right] - \delta \label{ineq:policy_iteration_1_analogous:proof:11}\\
& = V^{\pi_{k+1}}(\bar{s}) + \frac{1}{1-\gamma}\left[ \frac{4}{(1-\gamma)^2}\sum_{t=T+1}^{\infty}(M_t+N_t) +\frac{2(1+\gamma^2)}{(1-\gamma)^2} \sum_{t=T+1}^\infty M'_t \right] -\delta \label{ineq:policy_iteration_1_analogous:proof:12}\\
& = V^{\pi_{k+1}}(\bar{s}), \nonumber
\end{align}

for $\delta = \frac{4}{(1-\gamma)^3}\sum_{t=T+1}^{\infty}(M_t+N_t)+\frac{2(1+\gamma^2)}{(1-\gamma)^3} \sum_{t=T+1}^\infty M'_t$. 

The inequality in \eqref{ineq:policy_iteration_1_analogous:proof:06} follows from the fact that $\bar{s} \in B$. The inequality in \eqref{ineq:policy_iteration_1_analogous:proof:07} follows from \eqref{ineq:policy_iteration_1_analogous:proof:02}, and the inequality in \eqref{ineq:policy_iteration_1_analogous:proof:08} follows from \eqref{eqn:function_different_histories:upper_bound}. For the initial augmented state $\bar{s}\in B$, inequality \eqref{ineq:policy_iteration_1_analogous:proof:10} holds with the term $-\delta$. For all augmented states $\bar{s}\in\bar{\mathcal S}$, the same inequality holds with $\delta=0$, as the approximate greedy improvement does not decrease the approximate one-step value. Therefore, when recursively expanding \eqref{ineq:policy_iteration_1_analogous:proof:10}, we retain the term $-\delta$ at the initial state and use the inequality with $\delta=0$ at all subsequent states. This gives \eqref{ineq:policy_iteration_1_analogous:proof:11}.

Now, suppose $\bar{s} \notin B$, We need to show that this implies (\ref{eqn:nonstationary_small_bellman_error_analogous}), that is, the Bellman error should be very small. It is always true that $\T V \geq V$. In the other direction, for any $\bar{s} = (s, x_{0:\infty}) \in \bar{\S}$, 

Since $\bar{s} \notin B$,

\begin{align}
& \max_{a \in \A} \E_{\begin{subarray}{l} s' \sim Q_{x_{0:T}, x'_{T+1:\infty}}(. | s,a) \\\\ x' \sim Q^X_{x_{0:T}, x'_{T+1:\infty}}(.) \\\\ x'_{T+1:\infty} \sim \mu_2 \end{subarray} } \left[ r(s,a,s',x_{0:T},x'_{T+1:\infty}) + \gamma \hat{V}_k((s', x', x_{0:T-1})) \right] \nonumber \\
& \qquad \qquad \qquad \leq \E_{\begin{subarray}{l} s' \sim Q_{x_{0:T}, x'_{T+1:\infty}}(. | s,a) \\\\ x' \sim Q^X_{x_{0:T}, x'_{T+1:\infty}}(.) \\\\ x'_{T+1:\infty} \sim \mu_2 \\\\ a = \pi_k(\bar{s}) \end{subarray}} \left[ r(s,a,s',x_{0:T},x'_{T+1:\infty}) + \gamma \hat{V}_k((s', x', x_{0:T-1})) \right] + \delta \label{ineq:policy_iteration_1_analogous:proof:13}
\end{align}

We can lower bound bound the left hand side of \eqref{ineq:policy_iteration_1_analogous:proof:13} using \eqref{ineq:policy_iteration_1_analogous:proof:02} and \eqref{eqn:function_different_histories:upper_bound}
\begin{align}
& \max_{a \in \A} \E_{\begin{subarray}{l} s' \sim Q_{x_{0:T}, x'_{T+1:\infty}}(. | s,a) \\\\ x' \sim Q^X_{x_{0:T}, x'_{T+1:\infty}}(.) \\\\ x'_{T+1:\infty} \sim \mu_2  \end{subarray}} \left[ r(s,a,s',x_{0:T},x'_{T+1:\infty}) + \gamma \hat{V}_k((s', x', x_{0:T-1})) \right] \nonumber \\
& \qquad \qquad \qquad \geq \max_{a \in \A} \E_{\begin{subarray}{l} s' \sim Q_{x_{0:T}, x'_{T+1:\infty}}(. | s,a) \\\\ x' \sim Q^X_{x_{0:T}, x'_{T+1:\infty}}(.) \\\\ x'_{T+1:\infty} \sim \mu_2  \end{subarray}} \left[ r(s,a,s',x_{0:T},x'_{T+1:\infty}) + \gamma V^{\pi_k}((s', x', x_{0:\infty})) \right] \nonumber \\
& \qquad \qquad \qquad \qquad \qquad - \frac{2\gamma}{(1-\gamma)^2}\sum_{t=T+1}^{\infty}(M_t+M'_t+N_t) \nonumber \\
& \qquad \qquad \qquad \geq \max_{a \in \A} \E_{\begin{subarray}{l} s' \sim Q_{x_{0:T}, x'_{T+1:\infty}}(. | s,a) \\\\ x' \sim Q^X_{x_{0:T}, x'_{T+1:\infty}}(.) \\\\ x'_{T+1:\infty} \sim \mu_2  \end{subarray}} \left[ r(s,a,s',x_{0:\infty}) + \gamma V^{\pi_k}((s', x', x_{0:\infty})) \right] \nonumber \\
& \qquad \qquad \qquad \qquad \qquad - \frac{2\gamma}{(1-\gamma)^2}\sum_{t=T+1}^{\infty}(M_t+M'_t+N_t) - \sum_{t=T+1}^\infty M'_t \nonumber \\
& \qquad \qquad \qquad \geq \max_{a \in \A} \E_{\substack{s' \sim Q_{x_{0:\infty}}(. | s,a) \\ x' \sim Q^X_{x_{0:\infty}}(.) }} \left[ r(s,a,s',x_{0:\infty}) + \gamma V^{\pi_k}((s', x', x_{0:\infty})) \right] \nonumber \\
& \qquad \qquad \qquad \qquad \qquad  - \frac{2}{(1-\gamma)}\sum_{t=T+1}^{\infty}(M_t+N_t) - \frac{2\gamma}{(1-\gamma)^2}\sum_{t=T+1}^{\infty}(M_t+N_t) \nonumber \\
& \qquad \qquad \qquad \qquad \qquad - \sum_{t=T+1}^\infty \frac{1+\gamma^2}{(1-\gamma)^2} M'_t \label{ineq:policy_iteration_1_analogous:proof:14}
\end{align}

We can upper bound the right hand side of \eqref{ineq:policy_iteration_1_analogous:proof:13} using \eqref{ineq:policy_iteration_1_analogous:proof:02} and \eqref{eqn:function_different_histories:upper_bound}

\begin{align}
& \E_{\begin{subarray}{l}  s' \sim Q_{x_{0:T}, x'_{T+1:\infty}}(. | s,a) \\\\ x' \sim Q^X_{x_{0:T}, x'_{T+1:\infty}}(.) \\\\ x'_{T+1:\infty} \sim \mu_2 \\ a = \pi_k(\bar{s}) \end{subarray}} \left[ r(s,a,s',x_{0:T},x'_{T+1:\infty}) + \gamma \hat{V}_k((s', x', x_{0:T-1})) \right] \nonumber \\
& \qquad \qquad \qquad  \leq \E_{\begin{subarray}{l}  s' \sim Q_{x_{0:T}, x'_{T+1:\infty}}(. | s,a) \\\\ x' \sim Q^X_{x_{0:T}, x'_{T+1:\infty}}(.) \\\\ x'_{T+1:\infty} \sim \mu_2 \\\\ a = \pi_k(\bar{s}) \end{subarray}} \left[ r(s,a,s',x_{0:T},x'_{T+1:\infty}) + \gamma V^{\pi_k}((s', x', x_{0:\infty})) \right] \nonumber \\
& \qquad \qquad \qquad \qquad \qquad 
+ \frac{2\gamma}{(1-\gamma)^2}\sum_{t=T+1}^{\infty}(M_t+M'_t+N_t) \nonumber \\
& \qquad \qquad \qquad  \leq \E_{\begin{subarray}{l}  s' \sim Q_{x_{0:T}, x'_{T+1:\infty}}(. | s,a) \\\\ x' \sim Q^X_{x_{0:T}, x'_{T+1:\infty}}(.) \\\\ x'_{T+1:\infty} \sim \mu_2 \\\\ a = \pi_k(\bar{s}) \end{subarray}} \left[ r(s,a,s',x_{0:\infty}) + \gamma V^{\pi_k}((s', x', x_{0:\infty})) \right] \nonumber \\
& \qquad \qquad \qquad \qquad \qquad 
+ \frac{2\gamma}{(1-\gamma)^2}\sum_{t=T+1}^{\infty}(M_t+M'_t+N_t) + \sum_{t=T+1}^\infty M'_t \nonumber \\
& \qquad \qquad \qquad \leq \E_{\begin{subarray}{l} s' \sim Q_{x_{0:\infty}}(. | s,a) \\\\ x' \sim Q^X_{x_{0:\infty}}(.) \\\\ a = \pi_k(\bar{s}) \end{subarray}} \left[ r(s,a,s',x_{0:\infty}) + \gamma V^{\pi_k}((s', x', x_{0:\infty})) \right] \nonumber \\
& \qquad \qquad \qquad \qquad \qquad  + \frac{2}{(1-\gamma)}\sum_{t=T+1}^{\infty}(M_t+N_t) + \frac{2\gamma}{(1-\gamma)^2}\sum_{t=T+1}^{\infty}(M_t+N_t) \nonumber \\
& \qquad \qquad \qquad \qquad \qquad + \sum_{t=T+1}^\infty \frac{1+\gamma^2}{(1-\gamma)^2} M'_t \label{ineq:policy_iteration_1_analogous:proof:15}
\end{align}

Therefore, from \eqref{ineq:policy_iteration_1_analogous:proof:14} and \eqref{ineq:policy_iteration_1_analogous:proof:15}, and taking 

$\delta = \frac{4}{(1-\gamma)^3}\sum_{t=T+1}^{\infty}(M_t+N_t)+\frac{2(1+\gamma^2)}{(1-\gamma)^3} \sum_{t=T+1}^\infty M'_t$, we get

\begin{align*}
& \frac{-2}{(1-\gamma)^2}\sum_{t=T+1}^{\infty} (M_t+N_t) - \frac{1+\gamma^2}{(1-\gamma)^2} \sum_{t=T+1}^{\infty}M'_t \\
& \qquad \qquad \qquad \qquad \qquad + \max_{a \in \A} \E_{\begin{subarray}{l} s' \sim Q_{x_{0:\infty}}(. | s,a) \\\\ x' \sim Q^X_{x_{0:\infty}}(.) \end{subarray}} \left[ r(s,a,s',x_{0:\infty}) + \gamma V^{\pi_k}((s', x', x_{0:\infty})) \right] \\
& \qquad \leq \E_{\begin{subarray} s' \sim Q_{x_{0:\infty}}(. | s,a) \\\\ x' \sim Q^X_{x_{0:\infty}}(.) \\\\ a = \pi_k(\bar{s}) \end{subarray}} \left[ r(s,a,s',x_{0:\infty}) + \gamma V^{\pi_k}((s', x', x_{0:\infty})) \right] \\
& \qquad \qquad \qquad \qquad \qquad + \frac{2}{(1-\gamma)^2}\sum_{t=T+1}^{\infty} (M_t+N_t) + \frac{1+\gamma^2}{(1-\gamma)^2} \sum_{t=T+1}^{\infty}M'_t + \delta 
\end{align*}
Hence,
\begin{align*}
 &(\T V^{\pi_k})(\bar{s}) - \frac{2}{(1-\gamma)^2}\sum_{t=T+1}^{\infty} (M_t+N_t) - \frac{1+\gamma^2}{(1-\gamma)^2} \sum_{t=T+1}^{\infty}M'_t \nonumber \\ 
&\qquad \qquad\qquad \:\leq\: V^{\pi_k}(\bar{s}) + \frac{2}{(1-\gamma)^2}\sum_{t=T+1}^{\infty} (M_t+N_t) + \frac{1+\gamma^2}{(1-\gamma)^2} \sum_{t=T+1}^{\infty}M'_t + \delta
\end{align*}
\begin{align*}
 (\T V^{\pi_k})(\bar{s}) - V^{\pi_k}(\bar{s}) &\:\leq\: \frac{4}{(1-\gamma)^2}\sum_{t=T+1}^{\infty}(M_t+N_t) + \frac{2(1+\gamma^2)}{(1-\gamma)^2} \sum_{t=T+1}^\infty M'_t \\
 & \qquad + \frac{4}{(1-\gamma)^3}\sum_{t=T+1}^{\infty}(M_t+N_t)+\frac{2(1+\gamma^2)}{(1-\gamma)^3} \sum_{t=T+1}^\infty M'_t \\
\vert (\T V^{\pi_k})(\bar{s}) - V^{\pi_k}(\bar{s})\vert &\:\leq\: \frac{4(2-\gamma)}{(1-\gamma)^3}\sum_{t=T+1}^{\infty}(M_t+N_t) + \frac{2(1+\gamma^2)(2-\gamma)}{(1-\gamma)^3} \sum_{t=T+1}^\infty M'_t
\end{align*}

\section{Sample Complexity}
\label{section:proofs:subsection:sample_complexity}

In this section, we provide the details of the proofs of the sample complexity of policy evaluation as well as the overall policy iteration algorithm, wherein the policy evaluation step is performed using LSTD. The analysis closely follows the results in \citet{Lazaric:2012:FiniteSampleAnalysisOfLSPI}. We state the preliminaries and statements of some of their lemmas and provide details of where our proof differs from theirs.

\subsection{Sample Complexity of Least-Squares Policy Evaluation}

This theorem provides an upper bound on the expected error of the learned value function. Therefore, the proof has two parts: the first is Lemma~\ref{lemma:nonstationary_empirical_error_bound}, which is a bound on the empirical value function, and the second is a generalization to the entire space. The bound on the empirical error is similar to the original in the stationary case; therefore, we omit some steps. The final bound on the expected error, whose proof is provided in Section~\ref{section:proofs:subsection:sample_complexity:subsubsection:generalization}, offers more insights into the trade-off between the tractability of the algorithm and the approximation error due to nonstationarity.

\subsubsection{Empirical Error}
\label{section:proofs:subsection:sample_complexity:subsubsection:empirical_error}

We wish to show that with probability at least $1 - \delta$,
$$
    \left\Vert v - \hat{v} \right\Vert_N \leq \frac{1}{\sqrt{1 - \gamma^2}} \left\Vert v - \widehat{\Pi}v \right\Vert_N + \frac{L}{(1 - \gamma)^2} \sqrt{\frac{d}{\nu_N}} \left( \sqrt{\frac{2 \log(2d / \delta)}{N}} + \frac{1}{N} \right).
$$

Following \citet{Lazaric:2012:FiniteSampleAnalysisOfLSPI},
\begin{equation}
	\left\Vert V - \widehat{V} \right\Vert_N = \left\Vert v - \hat{v} \right\Vert_N \leq \frac{1}{\sqrt{1 - \gamma^2}} \left\Vert v - \widehat{\Pi}v \right\Vert_N + \frac{1}{1 - \gamma} \left\Vert \widehat{\Pi} v - \widehat{\Pi} \widehat{\T} v \right\Vert_N.
\label{eqn:nonstationary_prelim_empirical_error_bound}
\end{equation}
The first term is the approximation error, which is unavoidable owing to the restricted function space. The second term also includes the estimation error due to using the pathwise Bellman operator instead of the actual Bellman operator and is a function of the number of samples used in the training. It can be written as $\left\Vert \widehat{\Pi} \xi \right\Vert_N$, where $\xi = \widehat{\T} v - v$ is the estimation error of the output. Now, for $t < N$,
\begin{align*}
	\xi_t & = r_t + \gamma V(\bar{s}_{t+1}) - V(\bar{s}_t) \\*
		& = r(s_t, \pi(\bar{s}_t), s_{t+1}) + \gamma V(\bar{s}_{t+1}) - \E_{\bar{S}_{t+1}} \left[ r(s_t, \pi(\bar{s}_t, \bar{S}_{t+1}) + \gamma V(\bar{S}_{t+1}) \right] \\
		& = r(s_t, \pi(\bar{s}_t), s_{t+1}) - \E_{S_{t+1} \sim Q_{x_{:t}}(. | s_t, \pi(\bar{s}_t))} \left[ r(s_t, \pi(\bar{s}_t), S_{t+1}) \right] \\
			& \qquad \qquad \qquad \qquad \: + \gamma \left[ V(\bar{s}_{t+1}) - \E_{\bar{S}_{t+1} \sim Q,Q^X_{x_{:t}}(. | s_t, \pi(\bar{s}_t))} V(\bar{S}_{t+1}) \right],
\end{align*}
and for $t = N$,
\begin{align*}
	\xi_t  = r(s_t, \pi(\bar{s}_t), s_{t+1}) - &\E_{S_{t+1} \sim Q_{x_{:t}}(. | s_t, \pi(\bar{s}_t))} \left[ r(s_t, \pi(\bar{s}_t), S_{t+1}) \right] \\ & \qquad\qquad- \gamma \E_{\bar{S}_{t+1} \sim Q,Q^X_{x_{:t}}(. | s_t, \pi(\bar{s}_t))} \left[ V(\bar{S}_{t+1}) \right].
\end{align*}
$\xi_t$ are functions of $\bar{s}_t$, which are samples from the Markov chain induced by policy $\pi$ in the augmented MDP $\bar{M}$. Considering them as random variables, $\left( \xi_t \varphi_t(\bar{s_t}) \right)_{1 \leq t < N}$ is a martingale difference sequence w.r.t $\bar{s}_t$, with each element being zero mean and bounded by $\frac{L}{1 - \gamma}$. So, applying Azuma's inequality to $( \xi_t \varphi_i(\bar{s}))_{1 \leq t < N}$ gives, for each $i \in [d]$,
\begin{align*}
	\P \left( \left| \sum_{t=1}^{N-1} \xi_t \varphi_i(s_t) \right| \geq \epsilon \right) < 2 \exp \left( - \frac{\epsilon^2 (1 - \gamma)^2}{2(N-1) L^2} \right).
\end{align*}
Letting the right hand side be $\frac{\delta}{d}$, with probability at least $1 - \delta$,
\begin{align*}
	\left| \sum_{t=1}^{N-1} \xi_t \varphi_i(s_t) \right| < \frac{L}{1 - \gamma} \sqrt{2(N - 1)\log \left( \frac{2d}{\delta} \right) }, \quad \text{for all} \; i \in [d].
\end{align*}
Since $| \xi_n \varphi_i(\bar{s}_t) | < \frac{L}{1 - \gamma}$ always, we have,
\begin{equation}
\left| \sum_{t=1}^{N-1} \xi_t \varphi_i(s_t) \right| < \frac{L}{1 - \gamma} \left( \sqrt{2(N - 1)\log \left( \frac{2d}{\delta} \right) } + 1 \right),  \; \text{for all} \;  i \in [d] \text{ w.p. } \geq 1 - \delta.
\label{eqn:nonstationary_xi_bound}
\end{equation}
Following \citet{Lazaric:2012:FiniteSampleAnalysisOfLSPI}, this can be used to bound $|| \hat{\Pi} \xi_n ||_N$ as
\begin{equation}
\left\Vert \widehat{\Pi} \widehat{\T} v - \widehat{\Pi} v \right\Vert_N \leq \frac{1}{N} \sqrt{\frac{d}{\nu_n}} \max_{i} \left| \sum_{t=1}^{N} \xi_t \varphi_i(s_t) \right|,
\label{eqn:nonstationary_variance_bound}
\end{equation}
where $\nu_n$ is the smallest positive eigenvalue of the Gram matrix $\Phi^{\transpose} \Phi$.
By substituting~\eqref{eqn:nonstationary_xi_bound} and \eqref{eqn:nonstationary_variance_bound} in \eqref{eqn:nonstationary_prelim_empirical_error_bound} gives the desired result.

\subsubsection{Generalization}
\label{section:proofs:subsection:sample_complexity:subsubsection:generalization}

Intuitively, bounding the expected error of the estimated value function based on the error at only a few samples requires the samples to be close to the stationary distribution, and the function values at the samples to be close to their expected values.

The first requirement is formalized in terms of the rate of mixing of the Markov chain induced by the MDP using a policy.

\begin{definition}
    A Markov chain $\M = \left( X_t \right)_{t \geq 1}$ is said to be exponentially $\beta$-mixing with parameters $\bar{\beta}, b, \kappa$ if its mixing coefficients
    $$
        \beta_i = \sup_{t \geq 1} \left[ \sup_{B \in \sigma\left( X_{t+i}, \dots \right)} \left\vert \P \left(B | X_1, \dots, X_t \right) - \P(B) \right\vert \right]
    $$
    satisfy $\beta_i \leq \bar{\beta} \exp \left( -b i^\kappa \right)$.
\end{definition}

When such an exponentially mixing Markov chain is also ergodic and aperiodic with stationary distribution $\rho$, for any initial distribution $\lambda$, there is a bound on the following total variation: (Definition 21 of \cite{Lazaric:2012:FiniteSampleAnalysisOfLSPI})
$$
    \left\Vert \int_\X \lambda(dx) \P \left(. | x \right) - \rho(.)  \right\Vert_{\mathsf{TV}} \leq \bar{\beta} \exp \left( -b i^\kappa \right).
$$

This helps to bound the difference between the norm of a function in the stationary distribution and the empirical distribution defined using a few samples.

\begin{lemma}[Generalization Lemma for Markov chains, Lemma 24 of \cite{Lazaric:2012:FiniteSampleAnalysisOfLSPI}]
    Let $\tilde{\F}$ be the $d$-dimensional class of linear functions truncated at the threshold $B$. Let $\left( X_t \right)_{t \in [n]}$ be a sequence of samples from an ergodic, aperiodic, exponentially $\beta$-mixing Markov chain with parameters $\bar{\beta}, b, \kappa$, arbitrary initial distribution, and stationary distribution $\rho$. If the first $\tilde{n}$ samples are discarded and only the last $n - \tilde{n}$ samples are used, then with probability at least $1 - \delta$, the empirical and $l_2(\rho)$ norm of any function $\tilde{f} \in \tilde{\F}$ are related as,
    \begin{align*}
        \left\Vert \tilde{f} \right\Vert - 2 \left\Vert \tilde{f} \right\Vert_{X_{1:n}} & \leq \epsilon(\delta),  \\
        \left\Vert \tilde{f} \right\Vert_{X_{1:n}} - 2 \sqrt{2} \left\Vert \tilde{f} \right\Vert & \leq \epsilon(\delta),
    \end{align*}
    $$
        \text{for } \epsilon(\delta) = 12 B \sqrt{\frac{2 \Lambda(n-\tilde{n}, d, \delta)}{n - \tilde{n}} \max \left\{ \frac{\Lambda(n-\tilde{n}, d, \delta)}{1}, 1 \right\}^{\frac{1}{\kappa}}},
    $$
    $$
        \Lambda(n, d, \delta) = 2(d+1)\log n + \log \frac{\epsilon}{\delta} + \log^+ \left( \max \left\{ 16 (6e)^{2(d+1)} , \bar{\beta} \right\} \right),
    $$
    $$
        \mathrm{and} \:\:\:\tilde{n} = \left( \frac{1}{b} \log \left( \frac{2e \bar{\beta} n}{\delta} \right) \right)^{\frac{1}{\kappa}}.
    $$
\label{lemma:nonstationary_markov_chain_generalization}
\end{lemma}

The above lemma is essentially a generalization bound for truncated linear functions operating on samples from an exponentially mixing Markov chain and is used in the proof of Theorem~\ref{thm:nonstationary_sample_complexity_evaluation} to translate the bound on the empirical error to a bound on the expected error in the stationary distribution of the Markov chain induced by the policy under evaluation.

The above bounds hold for every member of the $d$-dimensional class of functions; hence, the dependence of $\epsilon$ on $d$ through $\Lambda(n, d, \delta)$. In contrast, when there is just one truncated linear function under consideration, the above bound holds, but with
$$
    \Lambda(n, \delta) = \log \frac{\epsilon}{\delta} + \log \left( \max \left\{ 6, n \bar{\beta} \right\} \right).
$$

Another issue in generalizing the sample-based bound in Lemma~\ref{lemma:nonstationary_empirical_error_bound} to the entire state space is its dependence on the eigenvalues of the sample Gram matrix. To address this problem, \cite{Lazaric:2012:FiniteSampleAnalysisOfLSPI} derived the following probabilistic lower bound on the smallest eigenvalue of the sample-based Gram matrix as a function of the smallest eigenvalue of the Gram matrix $G = \E_{x \sim \rho} \left[ \phi(x) \phi(x)^\transpose \right]$.

\begin{lemma}[Lemma 4 of \cite{Lazaric:2012:FiniteSampleAnalysisOfLSPI}]
    Let $\omega > 0$ be the smallest eigenvalue of the Gram matrix, defined above. If the data generating process is an exponentially $\beta$-mixing Markov chain with parameters $\bar{\beta}, b, \kappa$, then the smallest eigenvalue $\nu_n$ of the sample-based Gram matrix constructed using $n$ samples satisfies
    $$
        \sqrt{\nu_n} \geq \sqrt{\nu} = \frac{\sqrt{\omega}}{2} - 6L \sqrt{\frac{2 \Lambda(n, d, \delta)}{n} \max\left\{ \frac{\Lambda(n, d, \delta)}{b}, 1 \right\}^{\frac{1}{\kappa}}} > 0,
    $$
    provided the number of samples $n$ satisfies
    $$
        n > \frac{288 L^2 \Lambda(n, d, \delta)}{\omega} \max \left\{ \frac{\Lambda(n, d, \delta)}{b}, 1 \right\}^{\frac{1}{\kappa}},
    $$
    $$
        \text{where } \; \Lambda(n, d, \delta) = 2(d+1) \log n + \log \frac{e}{\delta} + \log^+ \left( \max \left\{ 18 (6e)^{2(d+1)} , \bar{\beta} \right\} \right)
    $$
    comes from a simpler version of Lemma~\ref{lemma:nonstationary_markov_chain_generalization} without the additional initial $\tilde{n}$ samples for the burn-in of the Markov chain.
\label{lemma:smallest_eigenvalue_bound_sample_gram_matrix}
\end{lemma}

The generalization Lemma~\ref{lemma:nonstationary_markov_chain_generalization}, combined with the lower bound on the eigenvalues and our results in Sections~\ref{section:guarantees} and \ref{section:policy_iteration:subsection:analysis}, gives rise to Theorem~\ref{thm:nonstationary_sample_complexity_evaluation} on the sample complexity of policy evaluation.

\subsection{Sample Complexity of Approximate Least-Squares Policy Iteration (Algorithm~\ref{alg:nonstationary_least_squares_policy_iteration})}

Here, we develop a recursive bound on the approximate suboptimality $\Vert V^* - \tilde{V}^{\pi_{k+1}} \Vert$ needed to study the propagation of error and obtain an upper bound on the suboptimality of the policy after $K$ steps of our algorithm.
Denote
$\ell_k = V^* - V^{\pi_k}$,
$e_k = \tilde{V}_k - V^{\pi_k}$, 
$g_k = V^{\pi_{k+1}} - V^{\pi_k}$, 
$b_k = \tilde{V}_k - T^{\pi_k} \widetilde{V}_k$, and 
$h_k = T^{\bar{\pi}_{k+1}} \widetilde{V}_k - T^{\pi_{k+1}} \widetilde{V}_k$.


\begin{lemma}

    \begin{align*}
        \ell_{k+1} & \leq \gamma P^{\pi^*}l_k + \gamma P^{{\pi}_{k+1}} \left[ \left[ I - \gamma \left(I - \gamma P^{{\pi}_{k+1}}\right)^{-1}\left(P^{{\pi}_k} - P^{{\pi}_{k+1}}\right)\right]e_k \right. \\*
        & \qquad \qquad  + \left. \left[ I - \gamma P^{{\pi}_{k+1}} \right]^{-1}h_k \right]
        - \gamma P^{\pi^*}e_k + h_k \\*
        & \leq \gamma P^{\pi^*}l_k + \gamma \left[ P^{{\pi}_{k+1}} \left( I - \gamma P^{{\pi}_{k+1}} \right)^{-1} - P^{\pi^*} \left( I - \gamma P^{{\pi}_k} \right)^{-1} \right]b_k \\*
        & \qquad \qquad + \left[ \gamma P^{{\pi}_{k+1}} \left( I - \gamma P^{{\pi}_{k+1}} \right)^{-1} + I \right]h_k
    \end{align*}
\label{lemma:munos_recursion}
\end{lemma}
\begin{proof}
We adapt and modify Lemmas 2-4 of \citet{Munos:2003:ErrorBoundsforApproximatePI} to account for the fact that the policy under consideration by our algorithm is not the true greedy policy, but an approximately greedy policy that considers only a subset of the extended state, which is the current state and a finite history of events. Following Lemma 2 of \citet{Munos:2003:ErrorBoundsforApproximatePI}, we obtain
\begin{align}
\ell_{k+1} &= V^*-V^{{\pi}_{k+1}} \nonumber \\
        &= (T^{\pi^*}V^* - T^{\pi^*}V^{{\pi}_k}) + (T^{\pi^*}V^{{\pi}_k} - T^{\pi^*}\tilde{V}_k) + (T^{\pi^*}\tilde{V}_k - T^{\bar{\pi}_{k+1}}\tilde{V}_k) \nonumber \\ 
        & \quad + (T^{\bar{\pi}_{k+1}}\tilde{V}_k - T^{{\pi}_{k+1}}\tilde{V}_k) + (T^{{\pi}_{k+1}}\tilde{V}_k - T^{{\pi}_{k+1}}V^{{\pi}_k}) + (T^{{\pi}_{k+1}}V^{{\pi}_k} - T^{{\pi}_{k+1}}V^{{\pi}_{k+1}}) \nonumber \\
        &\leq \gamma  P^{\pi^*}l_k - \gamma P^{\pi^*}e_k + 0 + h_k + \gamma P^{{\pi}_{k+1}}e_k - \gamma P^{{\pi}_{k+1}}g_k.
        \label{eqn:l_k:upper_bound}
\end{align}
Similarly, following Lemma 3 of \citet{Munos:2003:ErrorBoundsforApproximatePI},
\begin{align*}
    g_k &= V^{{\pi}_{k+1}} - V^{{\pi}_k} \\
        &= (T^{{\pi}_{k+1}}V^{{\pi}_{k+1}} - T^{{\pi}_{k+1}}V^{{\pi}_k}) + (T^{{\pi}_{k+1}}V^{{\pi}_k} - T^{{\pi}_{k+1}}\tilde{V}_k) + (T^{{\pi}_{k+1}}\tilde{V}_k - T^{\bar{\pi}_{k+1}}\tilde{V}_k) \\
        &\quad + (T^{\bar{\pi}_{k+1}}\tilde{V}_k - T^{{\pi}_k}\tilde{V}_k) + (T^{{\pi}_k}\tilde{V}_k - T^{{\pi}_k}V^{{\pi}_k}) \\
        &\geq \gamma P^{{\pi}_{k+1}}g_k - \gamma P^{{\pi}_{k+1}} - h_k + \gamma P^{{\pi}_k}e_k \\
        &\geq -\left[I - \gamma P^{{\pi}_{k+1}}\right]^{-1}\left[\gamma \left(P^{{\pi}_{k+1}}-P^{{\pi}_k}\right)e_k + h_k \right].
\end{align*}
Finally, following Lemma 4 of \citet{Munos:2003:ErrorBoundsforApproximatePI},
\begin{align*}
    e_k-g_k \leq \left[ I - \gamma \left( I - \gamma P^{{\pi}_{k+1}}\right)^{-1} \left( P^{{\pi}_k} - P^{{\pi}_{k+1}} \right) \right]e_k + \left[ I - \gamma P^{{\pi}_{k+1}} \right]^{-1}h_k.
\end{align*}
Substituting the above in \eqref{eqn:l_k:upper_bound} yields the desired result.

\end{proof}

\begin{lemma}
\[
|h_k(\bar{s})| \leq \frac{2R_{\mathrm{max}}}{1-\gamma}\sum_{t=T+1}^{\infty}(M_t + N_t)
\]
for all states $\bar{s} \in \bar{S}$, the augmented state space.
\label{lemma:bound_h_k}
\end{lemma}

\begin{proof}
For $\bar{s}^{(T)} = (s,x_{0:T})$, the approximately greedy policy and true greedy policy are defined as below.
\begin{equation*}
    \pi_{k+1} \left( \bar{s}^{(T)} \right) = \argmax_{a \in A} \; \E_{\begin{subarray}{l}{s' \sim Q_{x_{0:T}, (0)}(. | s,a)} \\\\ {x' \sim Q^X_{x_{0:T}, (0)}(.|s)}\end{subarray}} \bigg[ r(s, a, s')  + \gamma \widetilde{V}_k \left( \left( s', x', x_{0:T-1} \right) \right) \bigg]
\end{equation*}

\begin{equation*}
    \bar{\pi}_{k+1} \left( \bar{s} \right) = \argmax_{a \in A} \; \E_{\begin{subarray}{l}{s' \sim Q_{x_{0:\infty}}(. | s,a)} \\\\ {x' \sim Q^X_{x_{0:\infty}}(.|s)} \end{subarray}} \bigg[ r(s, a, s')  + \gamma \widetilde{V}_k \left( \left( s', x', x_{0:T-1} \right) \right) \bigg]
\end{equation*}

\noindent
Define corresponding Bellman operators

\begin{equation*}
     (T^{\pi_{k+1}} V_k)(\bar{s}) = \E_{
     \begin{subarray}{l}{s' \sim Q_{x_{0:\infty}}(. | s,\pi_{k+1}(\bar{s}^{(T)}))} \\ \\ {x' \sim Q^X_{x_{0:\infty}}(.|s)}\end{subarray}} \bigg[ r(s, \pi_{k+1}(\bar{s}^{(T)}), s')  + \gamma \widetilde{V}_k \left( \left( s', x', x_{0:T-1} \right) \right) \bigg]
\end{equation*}

\begin{equation*}
     (T^{\pi_{k+1}} V_k)(\bar{s}) = \E_{\begin{subarray}{l} s' \sim Q_{x_{0:\infty}}(. | s,\pi_{k+1}(\bar{s}^{(T)})) \\\\ x' \sim Q^X_{x_{0:\infty}}(.|s) \end{subarray}} \bigg[ r(s, \pi_{k+1}(\bar{s}^{(T)}), s')  + \gamma \widetilde{V}_k \left( \left( s', x', x_{0:T-1} \right) \right) \bigg]
\end{equation*}

\begin{equation*}
     (T^{\bar{\pi}_{k+1}} V_k)(\bar{s}) = \E_{\begin{subarray}{l} s' \sim Q_{x_{0:\infty}}(. | s,\bar{\pi}_{k+1}(\bar{s}) \\\\ x' \sim Q^X_{x_{0:\infty}}(.|s) \end{subarray}} \bigg[ r(s, \bar{\pi}_{k+1}(\bar{s}), s')  + \gamma \widetilde{V}_k \left( \left( s', x', x_{0:T-1} \right) \right) \bigg]
\end{equation*}

Define the approximate state-action value function $\widetilde{V}_{\mathrm{action}}$ function as below (state-action function is also referred to as $Q$ function, to avoid conflict of notation, we denote this by $V_{\mathrm{action}}$)
\begin{equation*}
     \widetilde{V}_{\mathrm{action}(\bar{s},a)} = \E_{\begin{subarray}{l} s' \sim Q_{x_{0:\infty}}(. | s,a) \\\\ x' \sim Q^X_{x_{0:\infty}}(.|s) \end{subarray}} \bigg[ r(s, a, s')  + \gamma \widetilde{V}_k \left( \left( s', x', x_{0:T-1} \right) \right) \bigg]
\end{equation*}

Using the definition of $\widetilde{V}_{\mathrm{action}}$, we can write policies $\pi_{k+1}$, $\bar{\pi}_{k+1}$ and their corresponding Bellman operators as below

\[
\pi_{k+1}(\bar{s}^{(T)}) = \arg\max_a \widetilde{V}_{\mathrm{action}}((\bar{s}^{(T)}, (0)_{T+1:\infty}), a)
\]
\[
\bar{\pi}_{k+1}(\bar{s}) = \arg\max_a \widetilde{V}_{\mathrm{action}}(\bar{s}, a)
\]
\[
(T^{\pi_{k+1}}V_k)(\bar{s}) = \widetilde{V}_{\mathrm{action}}(\bar{s}, \pi_{k+1}(\bar{s}^{(T)})) = \widetilde{V}_{\mathrm{action}}(\bar{s}, \bar{\pi}_{k+1}((\bar{s}^{(T)}, (0)_{T+1:\infty})))
\]
\[
(T^{\bar{\pi}_{k+1}}V_k)(\bar{s}) = \widetilde{V}_{\mathrm{action}}(\bar{s}, \bar{\pi}_{k+1}(\bar{s})) = \max_a \widetilde{V}_{\mathrm{action}}(\bar{s}, a)
\]

\noindent
Using the above definitions, we can write $h_k(\bar{s})$ for all $\bar{s} \in \bar{S}$ as

\[
h_k(\bar{s}) = \widetilde{V}_{\mathrm{action}}(\bar{s}, \bar{\pi}_{k+1}(\bar{s})) - \widetilde{V}_{\mathrm{action}}(\bar{s}, \bar{\pi}_{k+1}((\bar{s}^{(T)}, (0)_{T+1:\infty})))
\]

\begin{align}
    h_k(\bar{s}) &= \widetilde{V}_{\mathrm{action}}(\bar{s}, \bar{\pi}_{k+1}(\bar{s})) - \widetilde{V}_{\mathrm{action}}\left(\bar{s}, \bar{\pi}_{k+1}((\bar{s}^{(T)}, (0)_{T+1:\infty}))\right) \nonumber \\
    h_k(\bar{s}) &= \widetilde{V}_{\mathrm{action}}(\bar{s}, \bar{\pi}_{k+1}(\bar{s})) - \widetilde{V}_{\mathrm{action}}\left((\bar{s}^{(T)}, (0)_{T+1:\infty}), \bar{\pi}_{k+1}((\bar{s}^{(T)}, (0)_{T+1:\infty}))\right) \nonumber \\
    &\quad + \widetilde{V}_{\mathrm{action}}\left((\bar{s}^{(T)}, (0)_{T+1:\infty}), \bar{\pi}_{k+1}((\bar{s}^{(T)}, (0)_{T+1:\infty}))\right) - \widetilde{V}_{\mathrm{action}}\left(\bar{s}, \bar{\pi}_{k+1}((\bar{s}^{(T)}, (0)_{T+1:\infty}))\right)
    \label{eqn:h_k:upper_bound_intermediate}
\end{align}

\noindent
Using \eqref{eqn:function_different_histories:upper_bound}, we can bound 

\begin{align}
    \widetilde{V}_{\mathrm{action}}\left((\bar{s}^{(T)}, (0)_{T+1:\infty}), \bar{\pi}_{k+1}((\bar{s}^{(T)}, (0)_{T+1:\infty}))\right) - \widetilde{V}_{\mathrm{action}}\left(\bar{s}, \bar{\pi}_{k+1}((\bar{s}^{(T)}, (0)_{T+1:\infty}))\right) \nonumber \\
    \leq \|\widetilde{V}_{\mathrm{action}}\|_\infty \sum_{t=T+1}^{\infty}(M_t + N_t) \leq \frac{R_{\mathrm{max}}}{1-\gamma} \sum_{t=T+1}^{\infty}(M_t + N_t) 
    \label{eqn:h_k:bound:first_term}
\end{align}

\noindent
Bounding the first term of \eqref{eqn:h_k:upper_bound_intermediate}, 
\begin{align*}
    &\widetilde{V}_{\mathrm{action}}(\bar{s}, \bar{\pi}_{k+1}(\bar{s})) - \widetilde{V}_{\mathrm{action}}\left((\bar{s}^{(T)}, (0)_{T+1:\infty}), \bar{\pi}_{k+1}((\bar{s}^{(T)}, (0)_{T+1:\infty}))\right) \\
    &= \widetilde{V}_{\mathrm{action}}(\bar{s}, \bar{\pi}_{k+1}(\bar{s})) - \widetilde{V}_{\mathrm{action}}\left((\bar{s}^{(T)}, (0)_{T+1:\infty}), \bar{\pi}_{k+1}(\bar{s})\right) \\
    & \quad +\widetilde{V}_{\mathrm{action}}\left((\bar{s}^{(T)}, (0)_{T+1:\infty}), \bar{\pi}_{k+1}(\bar{s})\right) - \widetilde{V}_{\mathrm{action}}\left((\bar{s}^{(T)}, (0)_{T+1:\infty}), \bar{\pi}_{k+1}((\bar{s}^{(T)}, (0)_{T+1:\infty}))\right)
\end{align*}

\noindent
Using \eqref{eqn:function_different_histories:upper_bound} and by the definition of $\widetilde{V}_{\mathrm{action}}$ and $\bar{\pi}_{k+1}$,

\begin{align}
    \widetilde{V}_{\mathrm{action}}(\bar{s}, \bar{\pi}_{k+1}(\bar{s})) &- \widetilde{V}_{\mathrm{action}}\left((\bar{s}^{(T)}, (0)_{T+1:\infty}), \bar{\pi}_{k+1}((\bar{s}^{(T)}, (0)_{T+1:\infty}))\right) \nonumber \\
    &\qquad  \leq \|\widetilde{V}_{\mathrm{action}}\|_\infty \sum_{t=T+1}^{\infty}(M_t + N_t) \nonumber  \\ 
    &\qquad \leq \frac{R_{\mathrm{max}}}{1-\gamma} \sum_{t=T+1}^{\infty}(M_t + N_t) 
    \label{eqn:h_k:bound:second_term}
\end{align}

\noindent
By substituting \eqref{eqn:h_k:bound:first_term} and \eqref{eqn:h_k:bound:second_term} in \eqref{eqn:h_k:upper_bound_intermediate}, we get

\[
h_k(\bar{s}) \leq \frac{2R_{\mathrm{max}}}{1-\gamma}\sum_{t=T+1}^{\infty}(M_t + N_t)\:.
\]

\noindent
Similarly, 
\[-h_k(\bar{s}) \leq \frac{2R_{\mathrm{max}}}{1-\gamma}\sum_{t=T+1}^{\infty}(M_t + N_t)\,.
\]

\end{proof}

\section{Experimental Details}

\subsection{Policy Deployment}
In the policy deployment experiments, a three-layer fully connected network was utilized for the actor network, with each layer consisting of 300 nodes. The number of input and output nodes depends on the dimensions of the observation and action, respectively. ReLU activation functions were applied to the hidden layers, while a Tanh activation function was used at the output layer. A similar neural network architecture was implemented for the critic, which included
an additional action input. The output layer of the critic network was designed without an activation function. 

The Adam optimizer was applied to both the actor and the critic, with learning rates set at 1e-4 and 1e-3, respectively. A batch size of 256 was used, and the discount factor was set to 0.95. The critic employed a mean squared error loss function.

In the naive-lookahead experiments, a three-layer fully connected network was used as the dynamics model. This network takes the current augmented state and current action as inputs to predict the subsequent augmented state. The hidden layers employ ReLU activation functions, whereas the output layer was implemented without any activation function.

In the intelligent-lookahead experiments, a three-layer fully connected network was also used for the dynamics model. This network receives the current augmented state and current action to predict the subsequent state. A parametric model, which learns the external process parameters through maximum likelihood estimation (MLE), was employed to forecast future events.
\subsection{Policy Improvement}

\subsubsection{CarRacing Environment}

In the CarRacing environment experiments, we employed the Proximal Policy Optimization (PPO) algorithm, utilizing distinct Convolutional Neural Network (CNN) architectures for the policy and value networks. A three-layer CNN, each layer comprising 32 filters, was implemented to extract features from the raw pixel data. Subsequently, a group normalization layer was applied, followed by an Exponential Linear Unit (ELU) activation function. The event history was processed through a two-layer fully connected network. The features derived from the CNN (1152 dimensions) were concatenated with the event history features (1152 dimensions) and input into a two-layer fully connected network. A softmax activation function was employed for the policy network.

For both the actor and critic networks, the Adam optimizer was utilized with a learning rate of 5e-4. The batch size was configured to 128, and the discount factor was set to 0.99.

\subsubsection{Graph Navigation Environment}
In the graph navigation environment derived from town03 of the CARLA simulator, the PPO algorithm is employed with a shared feature extractor for both the actor and critic components. A shared 32-dimensional embedding is employed for the graph nodes. The event history is encoded using a single multilayer perceptron (MLP) layer, followed by an attention mechanism that focuses on the state-conditioned features from the encoded event history. The state embeddings were concatenated with the attention-guided history features and processed through a three-layer fully connected network, resulting in a 64-dimensional feature vector. For the policy network, this 64-dimensional feature vector was input into an MLP layer with five output nodes, corresponding to the maximum number of decisions available at any junction in the graph. A masked softmax function was applied to exclude illegal actions. For the value network, an MLP layer transforms the 64-dimensional feature vector into a single scalar value.

Both the actor and the critic use the Adam optimizer, starting with an initial learning rate of 7e-4 that decays linearly to 7e-5. The batch size is configured to 256, and the discount factor is set at 0.99.


\vskip 0.2in
\bibliography{references}

\end{document}